\pdfoutput=1
\documentclass[a4paper,12pt]{article}

\usepackage[verbose=true,a4paper]{geometry}
\usepackage[table,x11names,dvipsnames]{xcolor}  

\usepackage{scrextend}
\usepackage{lineno}
\usepackage{alias}
\usepackage{tikz}
\usepackage{colortbl}    %
\usepackage{diagbox}
\usepackage[vlined,ruled]{algorithm2e}
\usepackage[numbers,sort&compress,comma]{natbib}
\usepackage[utf8]{inputenc}
\usepackage{amsmath,amsfonts,amssymb,amsthm}
\usepackage{thm-restate}  %
\usepackage{multirow}
\usepackage{tabularx}
\usepackage{lscape}
\usepackage{booktabs}
\usepackage{mathrsfs}
\usepackage{graphicx}
\usepackage{graphics}
\usepackage{enumitem}
\usepackage{pgfplots}
\pgfplotsset{compat=1.18}
\usepackage{pgfplots}
\usepackage{makecell}
\usepackage{authblk}
\usepackage{eqparbox,hhline}
\usepackage{caption}

\usepackage{hyperref}
\usepackage{cleveref}

\newcolumntype{L}[1]{>{\raggedright\arraybackslash}p{#1}}
\newcolumntype{C}[1]{>{\centering\arraybackslash}p{#1}}
\newcolumntype{R}[1]{>{\raggedleft\arraybackslash}p{#1}}

\hypersetup{
    colorlinks=true,         
    linktocpage=false,       
    linkcolor=blue,          
    citecolor=blue,          
    filecolor=blue,          
    urlcolor=blue,           
    pdfborder={0 0 0},       
    bookmarksopen=true,      
    bookmarksnumbered=true,  
}

\definecolor{darkblue}{rgb}{0.3, 0.3, 1}

\title{Algebraic Machine Learning: Learning as computing an algebraic decomposition of a task}

\author[1,2]{Fernando Martin-Maroto \thanks{\href{mailto:fmmaroto@gmail.com}{fmmaroto@gmail.com}}}
\author[2]{Nabil Abderrahaman}
\author[1]{David Méndez}
\author[1,2]{Gonzalo G. de Polavieja  \thanks{\href{mailto:gonzalo.depolavieja@gmail.com}{gonzalo.depolavieja@gmail.com}}}

\affil[1]{Champalimaud Research, Champalimaud Centre for the Unknown, Lisbon, Portugal} 
\affil[2]{Algebraic AI}

\begin{document}

\maketitle 
\date{}

\begin{abstract}

 Statistics and Optimization are foundational to modern Machine Learning. Here, we propose an alternative foundation based on Abstract Algebra, with mathematics that facilitates the analysis of learning. In this approach, the goal of the task and the data are encoded as axioms of an algebra, and a model is obtained where only these axioms and their logical consequences hold. Although this is not a generalizing model, we show that selecting specific subsets of its breakdown into algebraic “atoms” obtained via subdirect decomposition gives a model that generalizes. We validate this new learning principle on standard datasets such as MNIST, FashionMNIST, CIFAR-10, and medical images, achieving performance comparable to optimized multilayer perceptrons. Beyond data-driven tasks, the new learning principle extends to formal problems, such as finding Hamiltonian cycles from their specifications and without relying on search. This algebraic foundation offers a fresh perspective on machine intelligence, featuring direct learning from training data without the need for validation dataset, scaling through model additivity, and asymptotic convergence to the underlying rule in the data. 
\end{abstract}
\begin{center}
\href{https://github.com/Algebraic-AI/Open-AML-Engine}{https://github.com/Algebraic-AI/Open-AML-Engine}
\end{center}
\newpage
 \newpage

\section*{Introduction}

Algebraic methods are widely used in Machine Learning \cite{algebra_in_ml_1, algebra_in_ml_2, algebra_in_ml_3}; however, the learning mechanism is based primarily on Statistics and Optimization \cite{bengio2017deep}. We propose Algebraic Machine Learning (AML) as an approach that uses Abstract Algebra as the foundation for learning itself, rather than in a supporting role. An advantage of taking an algebraic approach lies in its mathematical transparency and conceptual simplicity, offering new opportunities to analyze and understand learning.

AML differs from other Machine Learning methods. One difference is that its mathematics are closer to those of symbolic systems, yet it can learn from high-dimensional data like a connectionist system. This makes AML depart from AI's historical divide between symbolic methods \cite{symbolic_feigenbaum1973, symbolic_michie1984, symbolic_newell1959, symbolic_sowa2000} and learning methods \cite{connectionist_rumelhart1986, deep_learning_goodfellow2016, deep_learning_lecun2015}. It is also different from approaches in neurosymbolic AI that either combine a symbolic and a learning system to work together \cite{neurosymbolic_hybrid_1, neurosymbolic_hybrid_2} or the role of learning and the symbolic part are both done using gradients \cite{neurosymbolic_gradients_1, neurosymbolic_gradients_2}.

Another property of AML is that it generalizes directly from training data. Unlike statistical learning \cite{bengio2017deep}, no validation data is needed to determine hyperparameters or to stop training before overfitting. AML can also learn to solve formal problems, such as finding a Hamiltonian cycle or resolving Sudokus from the problem specification, without using training data or search. AML was introduced in a preliminary arxiv report \cite{Maroto}, followed by three reports with an analysis of its mathematics \cite{SecondPaperArX, ThirdPaperArX, InfX}. 

\begin{figure}[hbt!]
    \centering
	\includegraphics[width=1\linewidth]{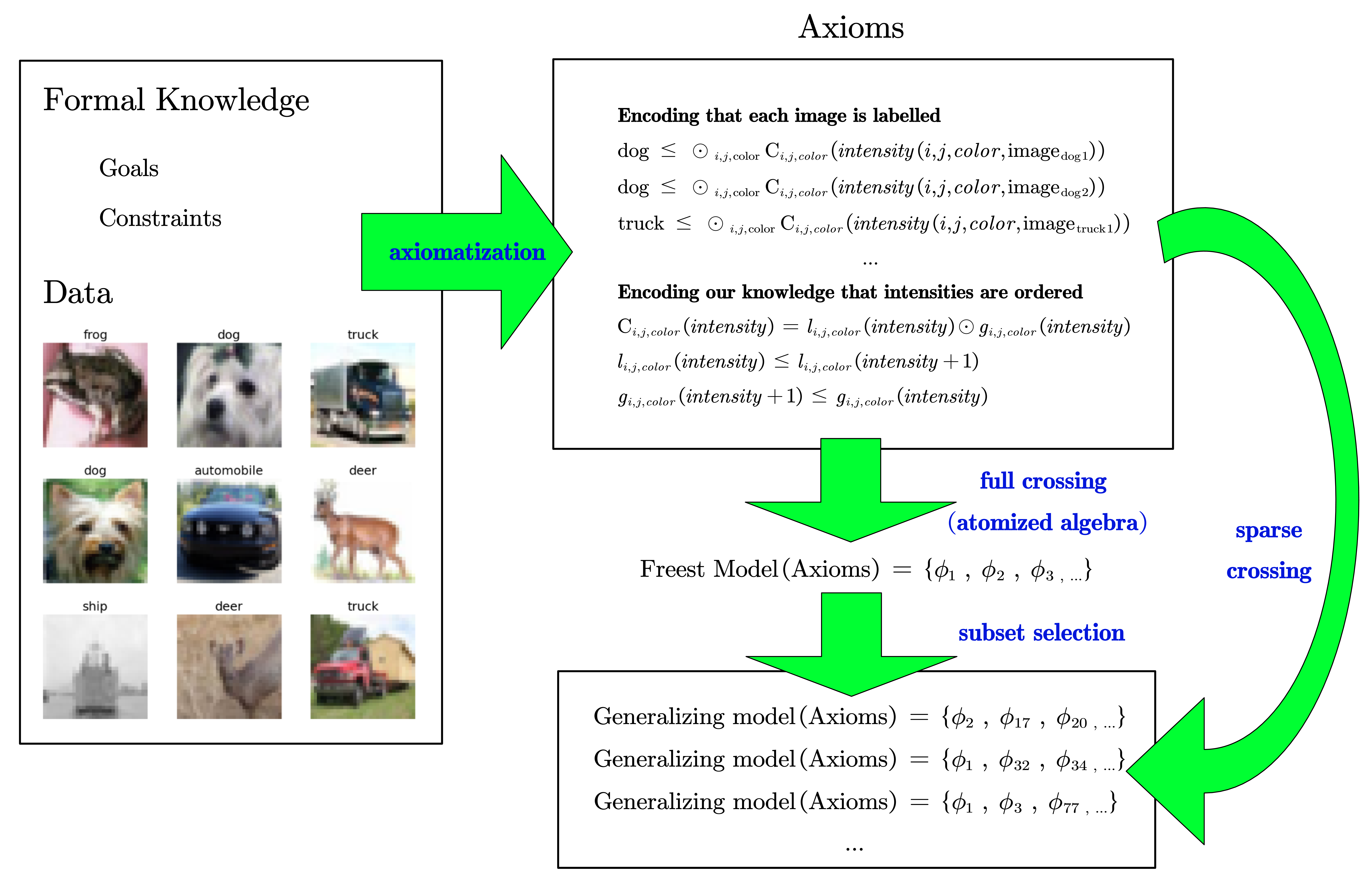}
	\caption{\textbf{Schematic representation of the Algebraic Machine Learning pipeline}. The process begins with axiomatization, where the problem, defined by data, goals, and prior knowledge, is encoded as a set of axioms. Then, we apply the Full Crossing procedure to obtain a specific model of the axioms, the freest model, a model in which the only true statements are the axioms and their logical consequences. Furthermore, the model is given explicitly as a subdirect product, expressed as basic atomic components (the atoms). Generalizing models are obtained by selecting certain subsets of atoms that collectively satisfy the axioms. In practical implementations, computing all atoms of the freest model is unnecessary; instead, a sparse variant of the Full Crossing procedure is used to directly calculate generalizing subsets of atoms.}
    \label{figure:visual_summary}
\end{figure}

\section*{Results}

\textbf{Figure \ref{figure:visual_summary}} provides a schematic representation of our approach. We start with axiomatization, where the problem, defined by data, goals, and prior knowledge, is encoded as a set of algebraic axioms. Then we use a procedure we call Full Crossing to obtain a model of the axioms. The specific model we obtain has two characteristics. First, it is the freest model, meaning the model in which only the axioms and their logical consequences are true. Second, it is expressed as a subdirect decomposition $\{ \phi_{1}, \phi_{2}, ...\}$, a decomposition known in Abstract Algebra \cite{Birkhoff} that we propose to find the fundamental building blocks, or atoms, of a problem. Of these atoms, specific subsets are each a generalizing model. In practice, we use Sparse Crossing, a version of Full Crossing that directly obtains the generalization subsets from the axioms. In this paper, we describe each of these steps, demonstrate how they produce generalization properties, and present results for standard datasets.

\subsection*{Encoding a task as axioms of an algebraic structure}

An \textit{algebraic structure} is a set $S$ with one or more operations that satisfy some axioms \cite{Burris, denecke2018universal,pinter2010book}. Specifically, we use a \textit{semilattice} algebra, which has a single binary operation $\odot$ that is commutative, associative and idempotent (i.e.\ $a \odot a = a$) \cite{Burris}. The semilattice provides a simple yet expressive enough framework that can effectively represent a broad range of tasks. 

The set $S$ contains certain special elements that we call \textit{constants}. These constants, $C$, are the primitives that we use to describe the specific task and data. For instance, in an image classification problem, a constant might represent a pixel in a particular color, while in a board game, a constant might represent a specific position or piece. 

In addition to these constants, $S$ includes all possible \textit{terms}, which are sets of constants formed using the operation $\odot$. For example, given the constants $\{c_1, c_2, \dots, c_n\}$, a possible term is $T = c_2 \odot c_8 \odot c_9$, where the component constants of the term $T$ are $\{c_2, c_8, c_9\}$.

To encode a machine learning task in the algebra, we introduce additional axioms. Each of these axioms asserts a relationship between two terms in the following way: a term, say $T_R$, has a property characterized by another term, say $T_L$, when $T_R \odot T_L = T_R$. This expression is saying that $T_L$ is already contained in or implied by $T_R$. To make this clear, we express $T_R \odot T_L = T_R$ with the more compact notation

\begin{equation}
T_L \leq T_R,
\label{eq:duple}
\end{equation}
We refer to this expression as a \textit{duple} because it can be represented as an ordered pair of terms, $ r \equiv (T_L, T_R) $. A task is thus expressed as a set of positive duples $T_{i} \leq T_{j}$ and negative duples $T_{i} \not\leq T_{j}$. 

As an example, consider the task of expressing that some binary sequences of length $4$ share some property. We can start by assigning a constant $p$ to the property. For the sequence we could use $2$ constants for each position, one for digit $1$ and another for digit $0$, giving a total of $8$ constants. For example, the constant $c_{31}$ could represent that the third position in the sequence is $1$. To express that the sequence $0100$ belongs to class $p$, we write the duple $T_L \leq T_R$, where $T_L = p$ and $T_R = c_{10} \odot c_{21} \odot c_{30} \odot c_{40}$. The task could then be encoded as a set of such duples, one for each sequence in the class.

The example illustrates a simple case of task encoding using a semilattice. This encoding technique is known as semantic embedding. It was introduced by mathematical logicians as encodings of algebraic structures within other algebraic structures, such as describing a group within a graph. For example, semantic embeddings have been extensively used in the study of undecidability \cite{Burris}. We have studied different types of semilattice embeddings with examples in \cite{ThirdPaperArX}.

\subsection*{Atomized models of the task}

Once the task is expressed as a set of duples, each of the form $T_{i} \leq T_{j}$ or $T_{i} \not\leq T_{j}$, the next step is to build a \textit{model}. A model is a specific semilattice structure in which these duples hold true.

Instead of building a semilattice, we compute an \textit{atomized semilattice} model \cite{SecondPaperArX}. An atomized semilattice has an idempotent operation $\odot$ and a binary, reflexive, and transitive order relation $<$. In semilattices, the idempotent operation $\odot$ defines an order relation $\leq$ while in atomized semilattices it is the other way around: the order relation $<$ defines the idempotent operator $\odot$ (see \textbf{Supplementary Section \ref{suppSection:atomizedStimlattices}, Theorem \ref{axiomatizationProperties}}).

An atomized semilattice has two sorts of elements: the regular elements (the terms) and the \textit{atoms}, which gives two disjoint sets, $S$ and $A$. We use Latin letters for regular elements, and Greek letters for atoms. Every atomized semilattice is a semilattice with respect to the regular elements, the set $S$, and a partial order with respect to all the elements, $S \cup A$. The idempotent operation $\odot$ acts only on elements of $S$ while the order relation $<$ acts on both, regular elements and atoms. 

An atomized semilattice satisfies an extended set of axioms that go beyond the commutative, associative and idempotent properties of a semilattice. The extended set of axioms describe the relationship between regular elements, atoms and constants (\textbf{Supplementary Section \ref{suppSection:atomizedStimlattices}, Definition \ref{definition:finiteAtomizedSemilattice}} and \cite{SecondPaperArX}). Here we mention some of the axioms and some of their consequences more directly related to how we build a model. One axiom is that for each atom $\phi$ there is at least one constant $c$ in its upper segment, that is, $\phi < c$. Also, each regular element $T$ has at least one atom $\phi$ in its lower segment, that is, $\phi < T$. However, no regular element is in the lower segment of an atom. 

One consequence of the axioms is that a duple, say $T_L \leq T_R$, is satisfied in the model if the atoms in the lower segment of $T_L$ are a subset of the atoms in the lower segment of $T_R$ (\textbf{Supplementary Section \ref{suppSection:atomizedStimlattices}, Theorem \ref{atomicSegmentFromTermTheorem} (vi)}):
\begin{equation}
\label{eq:atoms_in_a_duple}
       T_{L} \leq T_{R} \Leftrightarrow \{\phi| \phi< T_{L}\} \subseteq \{\phi| \phi<T_{R}\}.
\end{equation}

To make a practical use of \textbf{Equation \ref{eq:atoms_in_a_duple}}, we still need to know how to compute the lower segment of a term. For this we use that another consequence of the axioms is that the lower segment of a term  $T=c_{1} \odot c_{2}\odot...\odot c_{n}$ is the union of the lower segments of its component constants (\textbf{Supplementary Section \ref{suppSection:atomizedStimlattices}, Theorem \ref{atomicSegmentFromTermTheorem}(v)}):
\begin{equation}
\label{eq:atoms_in_a_term}
       \{\phi| \phi< T \} = \{\phi| \phi< c_{1} \} \cup \{\phi| \phi< c_{2} \} \cup ...\cup  \{\phi| \phi< c_{n} \}.
\end{equation}

To check if a duple $T_{L} \leq T_{R}$ holds in an atomized semilattice model, we must then verify that the atoms present in the model satisfy \textbf{Equation \ref{eq:atoms_in_a_duple}}, for which we need the atoms in the lower segments of $T_L$ and $T_R$ that can be obtained using \textbf{Equation \ref{eq:atoms_in_a_term}}.
\newline

Atomized semilattices have the following properties: 
\begin{itemize}
    \item An atom $\phi$ is fully characterized by the constants in its upper segment, i.e.\ those that satisfy  $\phi < c$ (\textbf{Supplementary Section \ref{suppSection:atomizedStimlattices}, Theorem \ref{atomicSegmentFromTermTheorem} (iv)}). This suggests a natural notation for atoms, e.g.\ $\phi[c_{3}, c_{4}]$ representing an atom with $c_{3}$ and $c_{4}$ in its upper segment and no other constants.
    
    \item An atomized semilattice model can be constructed from its atoms alone, so a model can be fully described as a set of atoms, each atom equal to a set of constants. A model $M$ can then be represented as:
        \begin{equation}
        \label{eq:example_model}
        M=\{ \phi[c_{1}, c_{2}, c_{3}], \phi[c_{2}, c_{5}], \phi[c_{1}, c_{6}],
        \phi[c_{3}], \phi[c_{3}, c_{4}],  \phi[c_{2}, c_{3}, c_{5}] \}.
        \end{equation}
    
    \item Since atoms are sets of constants, they have a universal meaning not associated to a particular atomized semilattice model. For example, according to \textbf{Equation \ref{eq:atoms_in_a_duple}}, an atom $\phi$ in a model $M$ that satisfies $\phi < T_1$ and $\phi \not< T_2$ causes $T_1 \not\leq T_2$ in the model $M$. Then, any model that has $\phi$ present will also satisfy $T_1 \not\leq T_2$ (\textbf{Supplementary Section \ref{suppSection:atomizedStimlattices}, Theorem \ref{atomIndependentFromTheRestTheorem}}). 
    
    \item  If the set of constants in the upper segment of an atom, for example  $\phi[c_{2}, c_{3}, c_{5}]$ above, can be written as the union of the constants in the upper segments of other different atoms of a model, e.g. $\phi[c_{3}]$ and $\phi[c_{2}, c_{5}]$, then the atom $\phi[c_{2}, c_{3}, c_{5}]$ is called ``redundant''. Redundant atoms can be eliminated from the model $M$ without altering which duples the model obeys (\textbf{Supplementary Section \ref{suppSection:atomizedStimlattices}, Theorem \ref{redundantAtom}}). Eliminating the redundant atoms in the model in \textbf{Equation \ref{eq:example_model}}, we have
        \begin{equation}
        \label{eq:example_model_nr}
        M=\{ \phi[c_{1}, c_{2}, c_{3}], \phi[c_{2}, c_{5}], \phi[c_{1}, c_{6}],
        \phi[c_{3}], \phi[c_{3}, c_{4}] \}.
        \end{equation}

    \item Non-redundant atoms of a model act as generators of the set of all atoms of a model
    (\textbf{Supplementary Section \ref{suppSection:atomizedStimlattices}, Theorem \ref{compositionTheorem}}).
    
    \item If the terms in the axioms are all concatenations of constants from the set $C$, any semilattice model of the axioms can be found as an atomized semilattice over $C$ (\textbf{Supplementary Section \ref{suppSection:atomizedStimlattices}, Theorem \ref{atomizationExistsTheorem}}).
    
    \item Each atom, redundant or non-redundant, of an atomized semilattice maps to a subdirectly irreducible component \cite{Burris} of the semilattice it atomizes. An atomized model is thus identifying the irreducible algebraic components of the task's model \cite{SecondPaperArX, InfX}.
\end{itemize}

\subsection*{Freest atomized model}

The \textit{freest model} of the task is the one for which the axioms of the task and its logical consequences are the only true statements. The logical consequences of the axioms are the positive and negative duples that are true in every model of the axioms. Any other model of the axioms satisfies a greater number of positive duples than the freest model and we say that it is less free than the freest model. 

\textit{Full Crossing} is a procedure to compute, step by step, the freest model of a set of axioms 
(\textbf{Supplementary Section \ref{suppSection:atomizedStimlattices}, Theorem \ref{fullCrossingIsFreestTheorem}}). It works in the following way. Let $X$ be the set of positive task duples already satisfied by a model $M$. We want to make positive a task duple that, according to $M$, is negative, $T_L \not\leq T_R$. Full Crossing operates over the model $M$ and produces the freest model of the task duples $X \cup \{ (T_L \leq T_R) \}$. For the task duple $T_L \leq T_R$ to be true, the atoms in the lower segment of $T_L$ must also be in the lower segment of $T_R$ (\textbf{Equation \ref{eq:atoms_in_a_duple}}). Let $R$ denote the set of atoms in the lower segment of $T_R$ and $n = \vert R \vert$. Let the discriminant $D$ be the set of atoms that are in the lower segment of $T_L$ but not in the lower segment $T_R$. Full Crossing replaces each atom in the discriminant, $\phi \in D$, by $n$ atoms, each given by a set of constants that is the union of the constants in the upper segment of $\phi$ and the constants in the upper segment of one atom in $R$.

To illustrate the Full Crossing procedure, consider the model $M$ given in \textbf{Equation \ref{eq:example_model_nr}} and suppose that we want to enforce the duple $T_L \leq T_R$ in $M$, where $T_L = c_1 \odot c_2$ and $T_R = c_3 \odot c_4$. In this case, $R = \{ \phi[c_1, c_2, c_3], \phi[c_3], \phi[c_3, c_4] \}$ and the discriminant is $D = \{ \phi[c_2, c_5], \phi[c_1, c_6] \}$. Each atom in $D$ is then substituted by $n=3$ new atoms. This process can be visualized in \textbf{Table \ref{table:full-crossing}}, where the atoms in $D$ are arranged in a column on the left and the atoms in $R$ as a row at the top. The new atoms are shown on a gray background in the table. Each new atom has an upper segment that is the union of the upper constant segment of the atom in $D$ in the same row and the upper constant segment of the atom in $R$ in the same column.

\begin{table}[hbt!]
\begin{center}

\begin{tabular}{|c|c|c|c|}
  \hline
  \backslashbox{$\,D\,$}{$R$} & $\phi[c_{1}, c_{2}, c_{3}]$ & $\phi[c_{3}]$ & $\phi[c_{3},c_{4}]$ \\ 
  \hline
  $\phi[c_{2},c_{5}]$ &  \cellcolor{gray!25} $\phi[c_{1},c_{2},c_{3},c_{5}]$ &  \cellcolor{gray!25} $\phi[c_{2},c_{3},c_{5}]$ &  \cellcolor{gray!25} $\phi[c_{2},c_{3},c_{4},c_{5}]$ \\
  \hline
  $\phi[c_{1},c_{6}]$ &  \cellcolor{gray!25} $\phi[c_{1},c_{2},c_{3},c_{6}]$ &  \cellcolor{gray!25} $\phi[c_{1},c_{3},c_{6}]$ &  \cellcolor{gray!25} $\phi[c_{1},c_{3},c_{4},c_{6}]$ \\
  \hline
\end{tabular}
\end{center}
 \caption{\textbf{Example of a full-crossing table}. To enforce the duple $T_L \leq T_R$ in the model given in \textbf{Equation \ref{eq:example_model_nr}}, we can build the following table: place on the top row the atoms in $R$, that is, the atoms in the lower segment of $T_R$, and on the left column the atoms in the discriminant $D$, i.e.\ those that are in the lower segment of $T_L$ and not in the lower segment of $T_R$. The procedure replaces each atom in the discriminant by the atoms in its row. Notice that each atom in the grayed area is the union of the atom in the top row and the atom in the left column. }
\label{table:full-crossing}
\end{table}

If we replace in $M$ the atoms in the discriminant by the atoms with gray background in the crossing \textbf{Table \ref{table:full-crossing}}, we obtain a model $N$ atomized as:
\begin{equation}
        \label{eq:example_model_afetrXs}
        N =\{ \phi[c_{1}, c_{2}, c_{3}], \phi[c_{1},c_{2},c_{3},c_{5}], 
        \phi[c_{2},c_{3},c_{5}],
        \phi[c_{1},c_{2},c_{3},c_{6}],
        \phi[c_{1},c_{3}, c_{6}],
        \phi[c_{3}], \phi[c_{3}, c_{4}] \},
        \end{equation}
which obeys $T_L = c_1 \odot c_2 \leq c_3 \odot c_4 = T_R$, and where $\phi[c_{2},c_{3},c_{4}, c_{5}]$ and $\phi[c_{1},c_{3},c_{4},c_{6}]$ are redundant and have been removed from $N$.

To compute the freest model of the task's axioms we can compute the Full Crossing procedure for all duples in the task, in any order (\textbf{Supplementary Section \ref{suppSection:atomizedStimlattices}, Theorem \ref{fullCrossingIsCommutative}}). To start with the sequence of crossings, we need an initial model that satisfies no positive duple besides those that are true on any semilattice. This freest semilattice can be atomized with as many atoms as constants, each atom with a single constant in its upper segment 
(\textbf{Supplementary Section \ref{suppSection:atomizedStimlattices}, Theorem \ref{freestModelTheorem}})
, e.g. $\{ \phi[c_{1}], \phi[c_{2}],..., \phi[c_m]\}$ where $m = \vert C \vert$. 

\subsection*{Freest atomized model of the task's axioms}

To build our intuition about the freest model of a task, consider the task of characterizing with a property $p$ the following set of $3{,}375$ black and white $4 \times 4$ images. The first column of each image is black, while the other three columns have pixels that are either black or white but without an entire black column. \textbf{Figure \ref{figure:full_crossing_example}a} displays $16$ of the images that meet this criterion.

\begin{figure}[hbt!]
    \centering
	\includegraphics[width=1\linewidth]{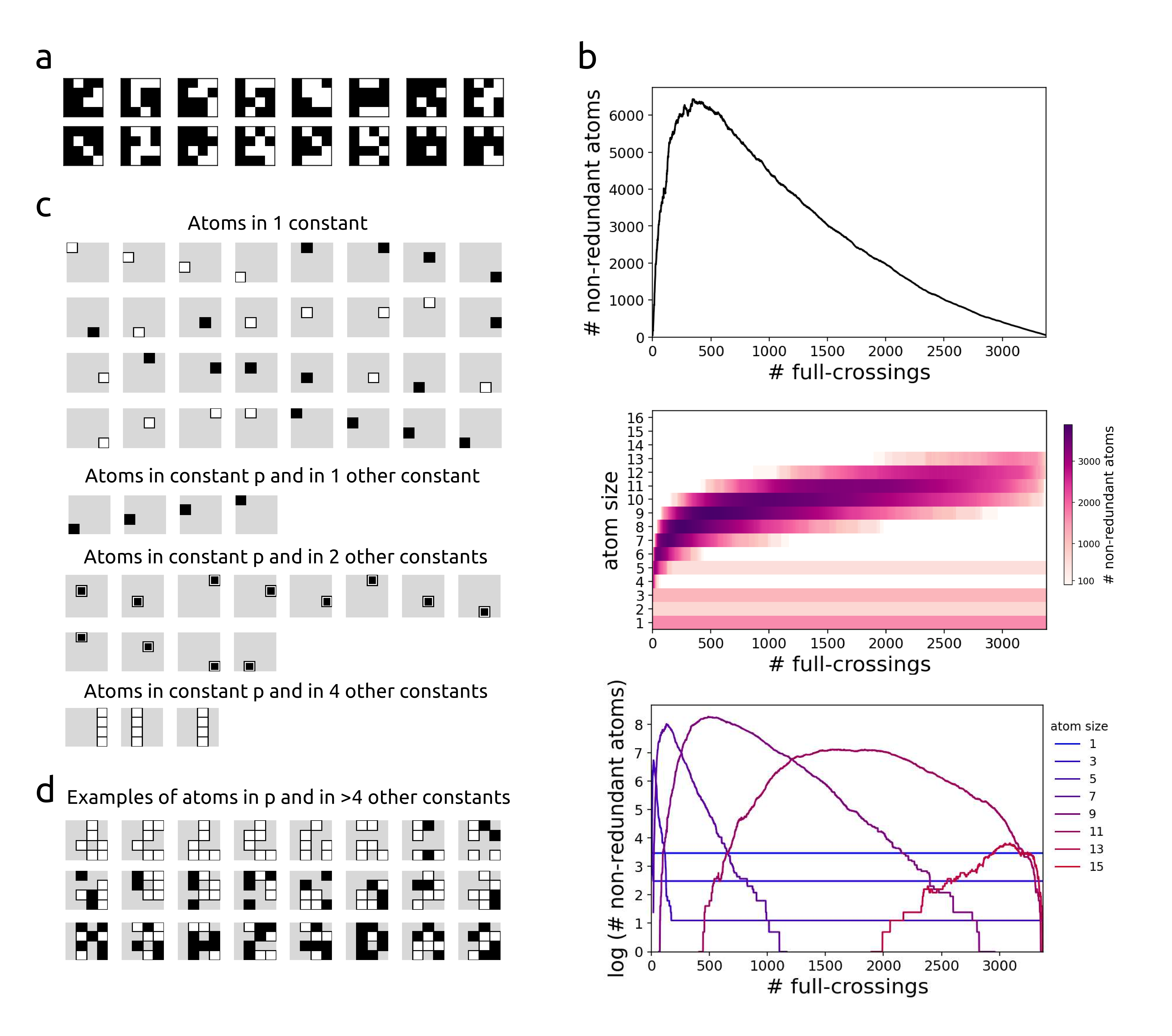}
	\caption{\textbf{Freest models using images with the only first column in black}.
	\textbf{(a)} $16$ of the possible $3{,}375$ images with only the first column in black. Training examples are of the form $p < T_i$ with $T_i$ the term representing an image. \textbf{(b)} Top: Number of non-redundant atoms of the model obtained after a number of full-crossings. Middle: Number of non-redundant atoms of a given atom size for the model obtained after a number of full-crossings. Bottom: Same as Middle bur represented by several curves, each for a different atom size. \textbf{(c)} Atoms of the final model. \textbf{(d)} Example of large atoms that are part of the models at intermediate number of full-crossings.}
\label{figure:full_crossing_example}
\end{figure}

For this problem, we can use $32$ constants for the $16$ pixels in black or in white, and one constant for the property $p$, a total of 33 constants. Our initial model is the freest semilattice atomized by $\{ \phi[c_{1}], \phi[c_{2}],..., \phi[c_ {32}], \phi[c_p]\}$. Starting from this model and using Full Crossing, we can enforce, one by one, $3{,}375$ duples, each of the form $p \leq T_i$, with $T_i$ a term of $16$ component constants representing the image. As the Full Crossing procedure progresses, the number of non-redundant atoms initially increases to approximately $6{,}000$ and then decreases to $51$ (\textbf{Figure \ref{figure:full_crossing_example}b}, top). When atoms are grouped by size (number of constants in its upper segment), we see that the number of non-redundant atoms with $1$, $2$, $3$, and $5$ constants quickly stabilizes to $32$, $4$, $12$, and $3$ atoms, respectively (\textbf{Figure \ref{figure:full_crossing_example}b}, middle). Larger atoms appear early on, increase in number, and then get removed from the model with more full-crossings (\textbf{Figure \ref{figure:full_crossing_example}b}, middle and bottom).

Let us look at the final model, \textbf{Figure \ref{figure:full_crossing_example}c}. It has a total of $51$ atoms. $32$ of these atoms are each in one of the $32$ constants representing a pixel in a color. These $32$ atoms were already in the initial model so they existed before any task duple was full-crossed. There are also $4$ atoms in two constants: constant $p$ and one of the four constants representing a black pixel in the first column of the image. These atoms capture that all images contain a black vertical bar in the first column. There are also $12$ atoms, one for each position in the last three columns, with $3$ constants: constant $p$ and the black and white constants of the same pixel. These atoms capture the fact that each pixel in the last three columns can be either black or white. There are $3$ atoms in $5$ constants: constant $p$ and the $4$ white constants of one of the three last columns of the image. These atoms capture that each of the last three columns is never completely black. 

It is also instructive to look at an intermediate model early on the crossing sequence, say after $200$full-crossings. This model already contains all the atoms of the final model, \textbf{Figure \ref{figure:full_crossing_example}c}. It also has larger atoms (some examples in \textbf{Figure \ref{figure:full_crossing_example}d}), which will all eventually be removed in later crossings.

\subsection*{Generalizing models} 

In the previous section, we considered the task of assigning a property $p$ to the set of images with the hidden rule that every image had the first column entirely black and the other columns with at least one white pixel, \textbf{Figure \ref{figure:full_crossing_example}}. The final freest model revealed this rule explicitly in its non-redundant atoms, \textbf{Figure \ref{figure:full_crossing_example}c}.

This result is general, as we can see in the following. Let $P$ be the set of duples that define the hidden rules of the task. Let $Q$ be the set of all duples that are the logical consequence of $P$ (the duples that are valid in all possible models of the task duples), with $Q$ excluding $P$.
We can prove that the non-redundant atoms of the freest model of $Q$ are the same as the non-redundant atoms of the freest model of $P$ (\textbf{Supplementary Section \ref{suppSection:discoveryOfRulesInData}, Theorem \ref{causalTheorem}}). The theorem then says that if we provide enough task duples, i.e.\ a large enough subset of $Q$, the freest model of the task duples becomes equivalent to the model of the rule duples $P$. 

Although this is true in the limit where all the consequences $Q$ are known, non-redundant atoms of the final model, or an approximation to them, must be created much earlier. In our example of the black bar, the final model required $3{,}375$ crossings, but its non-redundant atoms, 
\textbf{Figure \ref{figure:full_crossing_example}c}, are already present before $200$ crossings. Extracting those non-redundant atoms at $200$ crossings would give us a perfect generalizing model. In this section, we argue why this generalization, the early convergence to the rules of the task in some subset of the atoms, is a general phenomenon. We start studying it algebraically and then by using the expectation of the probability of false positive and false negative in a test dataset. 

First, we need to understand how atoms evolve as the positive task duples $r_1, r_2, \dots, r_n$ are enforced, where usually $n$ is much smaller that the number of consequences of the underlying rule in the data, $n << \vert Q \vert$, with $\vert Q \vert$ usually a very large number. Starting with the freest semilattice model as initial model, $N_0$, which does not yet satisfy the first task duple $r_1$, the Full Crossing procedure can be applied to enforce $r_1$, producing the model $N_1$. This process is applied to each duple, creating a chain of models $N_1, N_2, \dots, N_n$. For each atom $\phi$ in the final model $N_n$, there is an \textit{inward chain}  of atoms $\lambda_0, \lambda_1, \dots, \lambda_n$, with $\lambda_i \in N_i$ for $i \in \{0,...,n\}$, and $\lambda_{n} = \phi$ (\textbf{Supplementary Section \ref{suppSection:atomizedStimlattices}, Theorem \ref{inwarsOutwardSequenceTheorem}}). If the atom $\lambda_{i - 1}$ is not in the discriminant of $r_i$ then $\lambda_i = \lambda_{i - 1}$ while if it is, $\lambda_i$ has more constants in its upper segment than $\lambda_{i - 1}$ and we say the atom ``grows'' or becomes ``wider''. {\bf Figure \ref{figure:chains}} depicts the evolution of atoms from model $N_0$ to model $N_3$ formed after three crossing operations. 

\begin{figure}[hbt!]
    \centering
	\includegraphics[width=1\linewidth]{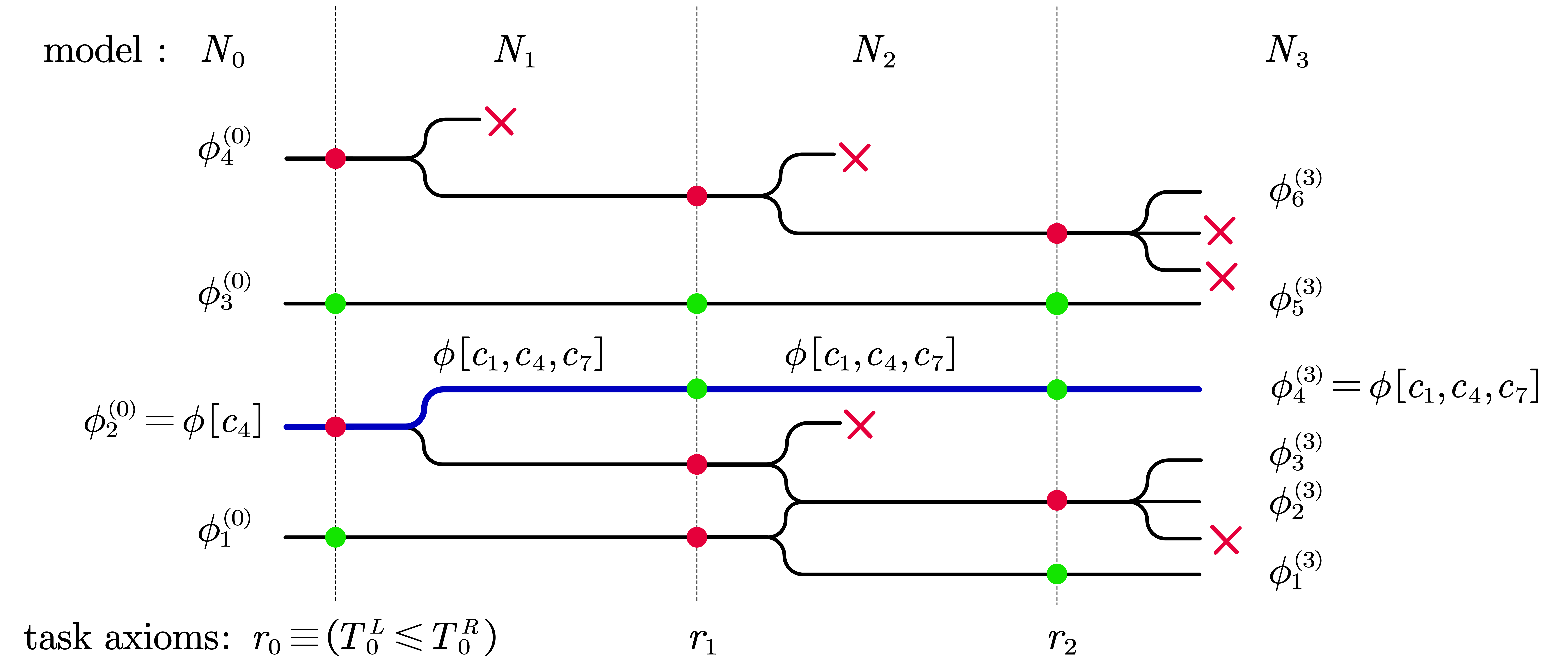}
	\caption{\textbf{Evolution of atoms during learning}. Starting with an initial model $N_0$, the crossing of duples $r_0, r_1$ and $r_2$ produces a sequence of four models $N_0,N_1,N_2, N_3$. An atom $\phi$ in the final model $N_3$ can be tracked to an atom in each of the models $N_2, N_1$ and $N_0$ forming at least one ``inward chain'' of four atoms $\lambda_i \in N_i$ and $\lambda_3 = \phi$. Along the chain, the atoms either grow, i.e.\ the number of constants in the upper segment of $\lambda_{i}$ is larger than the number of constants in $\lambda_{i - 1}$ (red nodes), or stays the same $\lambda_{i} = \lambda_{i - 1}$ (green nodes). Some atoms, marked with a red cross, are redundant and can be discarded. The blue line indicates an inward chain from a final atom to an initial atom. In this chain, there is one atom growth, $g(\phi_{4}^{(3)})=1$, and the final atom has been successful twice since the last growth, $h(\phi_{4}^{(3)})=2$.}
\label{figure:chains}   
\end{figure}

There are some quantities that help us characterize how atoms change during training. Given an inward chain for an atom in the final model, $\phi \in N_n$, let $g(\phi)$ be the number of times in which we find $\lambda_i \not= \lambda_{i - 1}$, i.e.\ the number of times the atoms in its chain grow. Let $k(\phi) \in \{0,...,n\}$, be the index $k$ of the first model in the sequence $N_1, N_2, \dots, N_n$ such that the atom $\phi$ is in model $N_k$, and let the ``success'' of atom $\phi$ be the number of consecutive crossings in which $\phi$ has remained unchanged, from its creation until the end of the crossing sequence,  $h(\phi) = n - k(\phi)$.

Using these quantities, we can express how each atom matures during training. Since the set of constants in the upper segment of an atom cannot be larger than the total number of constants, $\vert C \vert$, there is a finite number of times an atom can grow. As a result, after the $n$ crossing operations, even when $n << \vert Q \vert$, an atom $\phi$ present in the model may have grown to is final size and matured. A mature atom causes $0$ false negatives, but if the atom is not yet mature, at least we know that $\phi$ has grown $g(\phi)$ times and it has been consistent with the training duples $h(\phi)$ times since the last growth. These two quantities are what we need to compute the Probability of a False Negative ($\operatorname{PFN}$) in the test set, that is, the probability that the atom $\phi$ causes a test duple that should be positive to be negative in the model $N_n$. The expected $\operatorname{PFN}$, making the standard assumption that training and test distributions are the same, is
(\textbf{Supplementary Section \ref{suppSection:statistical_model}}):

\begin{equation}
\label{eq:fnr_of_atom}
\overline{\operatorname{PFN}}(\phi) = \min \left( \frac{1}{h(\phi) + 2},  \frac{g(\phi) + 1}{n + 2} \right).
\end{equation}

At the beginning of the training, $\frac{g(\phi) + 1}{n + 1}$ dominates due to the low success $h(\phi)$. After training with more positive training examples, $\frac{1}{h(\phi) + 1}$ becomes dominant as the atoms mature, producing lower (or even zero) probability of false negative. As an example, for the MNIST dataset of hand-written digits \cite{Lecun1998}, the number of training examples is $n = 50{,}000$, and most atoms have ten constants in its upper segment (\textbf{Figure \ref{figure:MNIST}c}), so they grow ten times during training, $g(\phi) \approx 10$. Ten growth events in $50{,}000$ examples imply that an average atom is successful $h(\phi) = 50{,}000 / 10 = 5{,}000$ times, giving a low individual $\overline{\operatorname{PFN}}$ of $0.0002$.

So far, we have characterized how a single atom matures during training, and now we are interested in subsets of atoms. Suppose that we extract a subset of $Z$ atoms of the freest model $N_n$. Each atom $\phi_i$ of this subset, with $i = 1,2,..,Z$, has undergone $g(\phi_i)$ stages of growth along its inward chain, and since it was created, it has been successful (i.e.\ consistent with the positive task duples) $h(\phi_i)$ times. The Probability of a False Negative ($\operatorname{PFN}$) in the test set is the probability that one or more of the $Z$ atoms causes a test duple that must be positive to be negative in the model $N_n$. After $n$ positive training examples, the expected test $\operatorname{PFN}$ can be approximated as (\textbf{Supplementary Section \ref{suppSection:statistical_model}}):
\begin{equation}
\label{eq:fnr_of_subsets}
\overline{\operatorname{PFN}}(\phi_1,\dots,\phi_Z) \approx \sum_{i=1}^{Z} \frac{1}{h(\phi_i) }.
\end{equation}
From this expression, it follows that the test $\operatorname{PFN}$ is reduced by lowering the number of atoms in $Z$ and by using atoms with a high success $h(\phi)$. 

The test Probability of False Positive ($\operatorname{PFP}$) is the probability that a test duple that must be negative is assigned positive in $N_n$. To have a false positive, every atom in the subset should fail to discriminate the duple, so the larger $Z$ is the less likely is to have a false positive. If we assume the probability of causing a false positive of individual atoms independent of each other, the collective $\operatorname{PFP}$ is given by the product of the individual $\operatorname{PFP}$s of each of the $Z$ atoms: 
\begin{equation}
\label{eq:fpr_of_subsets}
\overline{\operatorname{PFP}}(\phi_1,\dots,\phi_Z) = \prod_{i=1}^{Z} \overline{\operatorname{PFP}}(\phi_i).
\end{equation}

Since the negative duples of the training dataset play no role in the calculation of the freest model (every training duple $r_1, r_2, \dots, r_n$ is positive), the $\operatorname{PFP}$ of individual atoms can be obtained empirically using the negative examples of the training dataset as long as the training and test distributions are the same. The more effective an atom is at discriminating duples of the training set, the lower its probability of false positive.

In the formula above, we assumed that the individual $\overline{\operatorname{PFP}}(\phi_i)$ are independent of each other. If there are correlations, the lower the correlations between these individual probabilities are, the smaller the expected $\overline{\operatorname{PFP}}(\phi_1,\dots,\phi_Z)$ of the subset. Therefore, to obtain a good generalizing model, the atoms should be selected to be discriminative and with low mutual correlation.  

\textbf{Equations \ref{eq:fnr_of_subsets}} and \textbf{\ref{eq:fpr_of_subsets}} provide a way to extract a generalizing model from the freest atomized model. To minimize the test $\operatorname{PFN}$, the number of atoms selected should be as few as possible and highly successful during training (with high $h(\phi_i)$ values, which depend upon the positive duples of the training set). To minimize the test $\operatorname{PFP}$, the atoms in the subset should be selected to be effective at discriminating negative duples (with low $\operatorname{PFP}(\phi_i)$), have low mutual correlation, and a sufficient number to render every negative duple in the training set negative.

If we apply this method to the example of \textbf{Figure \ref{figure:full_crossing_example}}, we can isolate some of the atoms of the rule given in \textbf{Figure \ref{figure:full_crossing_example}c} before 200 crossings. For this purpose, we can use a training set of negative duples corresponding to counterexample images that do not adhere to the hidden rule. The method then extracts the $4$ atoms that are in the lower segment of $p$ and in the lower segment of another constant, as well as the $3$ atoms that are in the lower segment of $p$ and in the lower segments of $4$ white pixel constants. The method does not obtain the atoms in the lower segment of $p$ and in the black and white constants of the same pixel location. These atoms encode that every positive example contains either the black or the white pixel constant at each location of the last three columns of the image. Since the counterexamples used are also images, these atoms are not discriminative and are therefore not obtained using this method. In general, the method finds atoms that correspond to the rules satisfied by the positive examples but not by the negative examples of the training set. In this case, the subset of atoms extracted is a generalization model with zero error.

\subsection*{Practical computation of generalizing subsets with Sparse Crossing}

The freest model of a set of task duples is usually too large to calculate in practice. Since we are interested in its generalizing subsets, we devised a method to directly obtain, from the axioms, generalizing subsets of the freest model through a sparse version of the Full Crossing procedure.

The Sparse Crossing algorithm operates as follows: Every subset of atoms of the freest model satisfies all the positive task tuples. Regarding negative task tuples, the presence of a single atom in a model is sufficient for the model to satisfy a negative duple; indeed, the condition for a duple to be positive in a model is given by \textbf{Equation \ref{eq:atoms_in_a_duple}}. Consequently, there always exist subsets of atoms from the freest model that satisfy all positive and negative duples with cardinality less than or equal to the number of negative task duples. To identify a small subset of atoms that satisfies all the negative duples, we enforce the positive duples sequentially in a series of crossing steps. Instead of retaining all atoms in the full-crossing table, we selectively choose the atoms needed to discriminate the negative duples and discard the rest, as illustrated in \textbf{Table \ref{table:sparse-crossing}}. However, simply selecting atoms that satisfy the negative duples at a given crossing step does not work, as these atoms may not generate a discriminating subset after subsequent crossing steps. To address this issue, atoms are selected based on an invariance condition: the preservation of a quantity we call the trace. This condition allows us to discard atoms while ensuring that every negative task duple will be satisfied after the crossing of all positive task tuples (see \textbf{Supplementary Section \ref{suppSection:SparseCrossingInDepth}}).

\begin{table}[hbt!]
\begin{center}
\begin{tabular}{|c|c|c|c|}
  \hline
  \backslashbox{$\,D\,$}{$R$}  & $\phi[c_{1}, c_{2}, c_{3}]$ & $\phi[c_{3}]$ & $\phi[c_{3},c_{4}]$ \\ 
  \hline
  $\phi[c_{2},c_{5}]$ &  \cellcolor{gray!25} 
  $        $ &  \cellcolor{gray!25} 
  $\phi[c_{2},c_{3},c_{5}]$ &  \cellcolor{gray!25} 
  $  $ \\
  \hline
  $\phi[c_{1},c_{6}]$ &  \cellcolor{gray!25} 
  $           $ &  \cellcolor{gray!25} 
  $\phi[c_{1},c_{3},c_{6}]$ &  \cellcolor{gray!25} $\phi[c_{1},c_{3},c_{4},c_{6}]$ \\
  \hline
\end{tabular}
\end{center}
\caption{\textbf{Example of a sparse-crossing table}. A subset of the atoms of the full-crossing \textbf{Table \ref{table:full-crossing}} that suffice to preserve the trace of all the terms mentioned in the axioms.}
\label{table:sparse-crossing}
\end{table}

With Sparse Crossing, positive and negative task tuples are processed in batches selected among the task tuples with replacement. The initial model of a batch is the output model of the previous batch. Additionally, Sparse Crossing allows the atoms produced in all previous batches, not just the immediately preceding one, to influence the process of discarding atoms by means of the \textit{pinning terms} (\textbf{Supplementary Section \ref{suppSection:atomizedStimlattices}, Definition \ref{definition:pinning_tem}}). The pinning terms provide an effect similar to augmenting the set of negative axioms and accelerate the discovery of atoms of the freest model that are building blocks of other atoms (i.e., atoms whose set of constants in their upper segment is a subset of that of various other atoms (see \textbf{Supplementary Section \ref{SparseCrossing:smarter}} and \textbf{Theorem \ref{SparseCrossing:indicatorsTheorem}}). Since the non-redundant atoms are the building blocks of all the atoms, the presence of pinning terms increases the likelihood of discovering non-redundant atoms. Moreover, because every duple discriminated by an atom is also discriminated by at least one non-redundant atom, the non-redundant atoms of the model are often among the most effective at satisfying the negative tuples, which further increases their likelihood of discovery.

The result of applying Sparse Crossing to a batch of positive and negative duples is a subset of atoms of the freest model that satisfies the positive and negative duples of the batch, that has small cardinality (smaller than the number of negative duples), and with all its atoms very successful for the positive task duples. Small subsets of atoms that collectively discriminate every negative duple tend to be highly discriminative while having low correlation with each other. Sparse Crossing is thus obtaining subsets with all the characteristics needed for generalization.

Sparse-crossing is a stochastic algorithm, so it is possible to compute several different models of a given set of positive and negative axioms. Since the union of models (as set union of atoms) is also a model of the axioms, it is possible to use (embarrassingly) parallel computation to calculate larger models. For the complete details of the Sparse-Crossing, including various theorems and pseudocode, see \textbf{Supplementary Section \ref{suppSection:SparseCrossingInDepth}}.

\subsection*{Learning from data}

Black and white images can be classified using the same embedding strategy we applied to the toy example in \textbf{Figure \ref{figure:full_crossing_example}}. At each pixel location, one constant represents the pixel in black and another represents it in white. Each image is then encoded in a term resulting from the idempotent summation of its pixel constants. The handwritten digit recognition dataset (MNIST) \cite{Lecun1998} is ideal for testing this embedding as there is variability in how digits are written, the training set contains some mislabeled images \cite{mislabels}, and the images were originally black and white. In this case, we have a total of $2 \times 28 \times 28$ constants for the pixels and constants $digit_i$, with $i=0,1,...,9$, for the $10$ classes. The grayscale values in these images resulted from centering the digits, so we binarized them back by thresholding pixel values. We applied Sparse-Crossing to the $50{,}000$ MNIST training examples, each encoded as a task duple $digit_i \leq image_k$. Additionally, we have a set of $450{,}000$ negative task duples, each of the form $digit_{j \not= i} \not\leq image_k$.

A test image is classified as digit $i$ when the atoms in the lower segment of constant $digit_i$ are a subset of the atoms in the lower segment of the term representing the test image, as in \textbf{Equation \ref{eq:atoms_in_a_duple}}. After training, about $70\%$ of the test images have a digit assigned in this way. This is because the training set is not large enough to obtain a model that gives assignations for every example of the test set. However, we can give ``best guess'' assignations for each test example. One simple method is to classify a test image as belonging to the class that more closely obeys the subset condition \textbf{Equation \ref{eq:atoms_in_a_duple}}. We use the word ``misses'' to refer to the atoms in the lower segment of the left-hand side of a duple, in this case $digit_i$, that are not in the lower segment of the right-hand side, in this case the term that represents the test image. A test image can then be classified as the digit with the fewest misses. 

\textbf{Figure \ref{figure:MNIST}a} shows the frequency of the number of misses for queries of whether test images corresponds to digit $7$, both for test examples of digit $7$ (\textbf{Figure \ref{figure:MNIST}a}, green) with mean $11$ and for the other digits (\textbf{Figure \ref{figure:MNIST}a}, red), with mean $1{,}192$. The two distributions have very small overlap, explaining why the simple method of selecting the class with fewer misses gives a good separation between positive and negative test examples. The resulting test classification accuracy is $97.63\%$ (\textbf{Table \ref{table:performance_comparison}}), ``AML fewest misses'' column). When trained only with the first $1{,}000$ examples of the training set, the test accuracy is $90.24\%$. Importantly, Since the algebra grows organically as it learns, it is not necessary to specify an architecture, so no validation dataset is used to select architecture and other hyperparameters. Also, we do not need to use a validation dataset to stop training, as both our theoretical analysis and the empirical results show no overfitting (\textbf{Figure \ref{figure:MNIST}b}).

As an alternative to the ``fewest misses'' method, we also used logistic regression as a very simple way to include statistical information. The input to the logistic regression is the output of AML, given in the following way. The atoms that are in the lower segment of the image term are given a value of $+1$ and the atoms that are not are given a value of $-1$. For each image, the input to the logistic regression is then a sequence of $+1$ and $-1$ values. Each element of the sequence connects with a linear weight to each of 10 $\operatorname{softmax}$ outputs. This method then decides which class corresponds to an input using a single linear hyperplane per class. We trained the linear weights using only the training dataset, Adam optimizer \cite{adam} and cross-entropy loss \cite{deep_learning_goodfellow2016}, and obtained a test accuracy of $98.43\%$ for the $50{,}000$ training examples and $91.56\%$ for $1{,}000$ training examples (see column ``AML log.\ reg.'' in \textbf{Table \ref{table:performance_comparison}}). 

The embedding strategy used is generally applicable to classification problems as it is not limited to images. We therefore compared our results with Multilayer Perceptrons (MLPs), which are also free of image-specific biases. MLPs have multiple hyperparameters that require optimization. We ran 360 MLPs with different hyperparameter configurations and two to four hidden layers (see \textbf{Methods}). A validation dataset of $10{,}000$ examples was used to stop training before overfitting and to select the best of the 360 models, which achieved a test accuracy of $98.46\%$ (column ``MLP best'' in \textbf{Table \ref{table:performance_comparison}}). The best MLP trained only with the first $1{,}000$ examples of the training set reached $88.70\%$ test accuracy.

\begin{figure}[hbt!]
    \centering
        \includegraphics[width=1\linewidth]{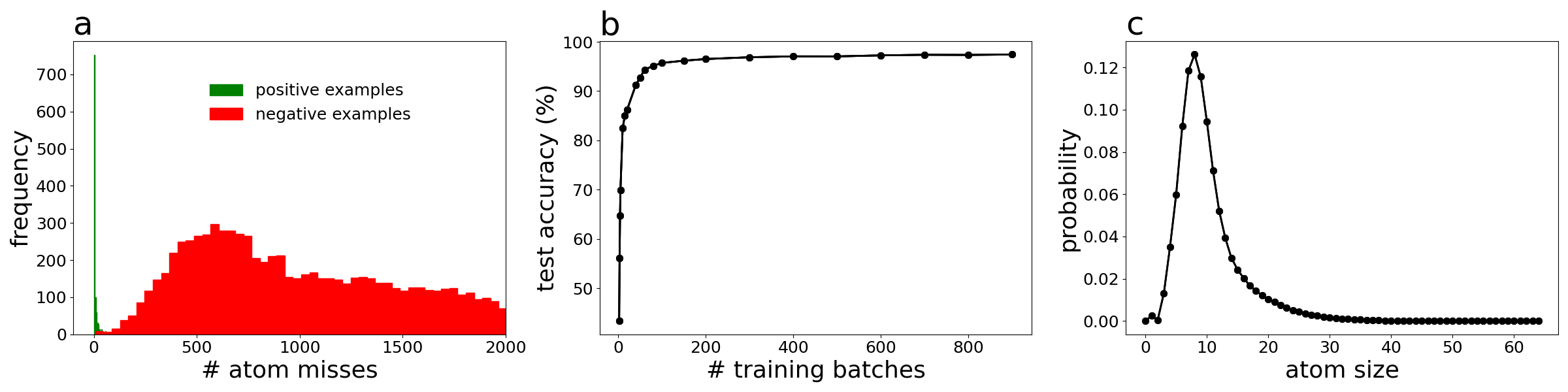}
	\caption{\textbf{Sparse Crossing of hand-written digits (MNIST dataset)} \textbf{(a)} Frequency of the number of misses for a query of whether a test image is a $7$ for test examples of digit $7$ (green) and for the other digits (red). \textbf{(b)} Test accuracy increases during training. \textbf{(c)} Distribution of atom sizes, with atom size the number of constants in the upper segment of the atom).}
\label{figure:MNIST}
\end{figure}

\begin{table}[htbp!]
    \centering    \small
    \renewcommand{\arraystretch}{1.2}  %
     \resizebox{0.9\textwidth}{!}{
    \begin{tabular}{|L{0.23\textwidth}|C{0.12\textwidth}|C{0.12\textwidth}|C{0.20\textwidth}|C{0.18\textwidth}|}
        \hline
        \rowcolor[HTML]{E6E6E6} 
        \textbf{Dataset} & 
        \makecell{\textbf{AML} \\ \scriptsize{fewest misses}} & 
        \makecell{\textbf{AML} \\ \scriptsize{log.\ reg.}} & 
        \makecell{\textbf{MLP} \\ \scriptsize{best}} & 
        \makecell{\textbf{MLP} \\ \scriptsize{mean $\pm$ std.}} \\ 
        \hline
        
        MNIST \newline \scriptsize{$28 \times 28$, 10, 50000/10000/10000}  
        & \makecell{$97.63\%$} & \makecell{$98.43\%$} 
        & \makecell{$98.46\%$ \\ \scriptsize{(2048, 1024, 128)}} 
        & \makecell{$98.03\% \pm 0.21\%$} \\ 
        \hline

        MNIST \newline \scriptsize{$28 \times 28$, 10, 1000/10000/10000}  
        & \makecell{$90.24\%$} & \makecell{$91.56\%$} 
        & \makecell{$88.70\%$ \\ \scriptsize{(4096, 256, 128)}} 
        & \makecell{$87.73\% \pm 1.54\%$} \\ 
        \hline

        fashionMNIST \newline \scriptsize{$28 \times 28$, 10, 50000/10000/10000}  
        & \makecell{$87.27\%$} & \makecell{$89.47\%$} 
        & \makecell{$89.52\%$ \\ \scriptsize{(4096, 256)}} 
        & \makecell{$88.40\% \pm 0.45\%$} \\ 
        \hline

        fashionMNIST \newline \scriptsize{$28 \times 28$, 10, 1000/10000/10000}  
        & \makecell{$79.62\%$} & \makecell{$81.89\%$} 
        & \makecell{$80.73\%$ \\ \scriptsize{(2048, 256)}} 
        & \makecell{$79.52\% \pm 1.57\%$} \\ 
        \hline

        CIFAR-10 \newline \scriptsize{$32 \times 32$, 10, 50000/5000/5000}  
        & \makecell{$48.56\%$} & \makecell{$53.60\%$} 
        & \makecell{$54.58\%$ \\ \scriptsize{(4096, 256)}} 
        & \makecell{$54.12\% \pm 0.79\%$} \\ 
        \hline

        CIFAR-10 \newline \scriptsize{$32 \times 32$, 10, 1000/5000/5000}  
        & \makecell{$36.33\%$} & \makecell{$38.49\%$} 
        & \makecell{$36.23\%$ \\ \scriptsize{(4096, 256, 512)}} 
        & \makecell{$35.58\% \pm 0.63\%$} \\ 
        \hline
        
        dermaMNIST \newline \scriptsize{$28 \times 28$, 7, 7007/1003/2005}  
        & \makecell{$73.47\%$} & \makecell{$74.21\%$} 
        & \makecell{$65.94\%$ \\ \scriptsize{(4096, 2048, 128, 256)}} 
        & \makecell{$55.97\% \pm 3.28\%$} \\ 
        \hline
        
        pneumoniaMNIST \newline \scriptsize{$28 \times 28$, 2, 4708/524/624}  
        & \makecell{$84.62\%$} & \makecell{$85.90\%$} 
        & \makecell{$87.66\%$ \\ \scriptsize{(2048, 256, 512)}} 
        & \makecell{$87.63\% \pm 1.19\%$} \\ 
        \hline
        
        pneumoniaMNIST \newline \scriptsize{$64 \times 64$, 2, 4708/524/624}  
        & \makecell{$84.13\%$} & \makecell{$84.93\%$} 
        & \makecell{$87.50\%$ \\ \scriptsize{(512, 256, 256)}} 
        & \makecell{$87.35\% \pm 1.29\%$} \\ 
        \hline
        
        organCMNIST \newline \scriptsize{$28 \times 28$, 11, 12975/2392/8216}  
        & \makecell{$81.28\%$} & \makecell{$86.75\%$} 
        & \makecell{$76.44\%$ \\ \scriptsize{(4096, 2048, 128)}} 
        & \makecell{$75.04\% \pm 1.33\%$} \\ 
        \hline
        
        bloodMNIST \newline \scriptsize{$28 \times 28$, 8, 11959/1712/3421}  
        & \makecell{$85.26\%$} & \makecell{$90.93\%$} 
        & \makecell{$85.30\%$ \\ \scriptsize{(4096, 1024, 256)}} 
        & \makecell{$84.87\% \pm 1.56\%$} \\ 
        \hline
        
        bloodMNIST \newline \scriptsize{$64 \times 64$, 8, 11959/1712/3421}  
        & \makecell{$87.55\%$} & \makecell{$92.90\%$} 
        & \makecell{$89.16\%$ \\ \scriptsize{(2048, 256, 256)}} 
        & \makecell{$86.76\% \pm 1.30\% $} \\ 
        \hline
    \end{tabular}}
    \captionsetup{width=.9\textwidth}
         \caption{\textbf{Test accuracy of AML and MLP models on various image datasets}. Dataset details include image dimensions, class count, and training/validation/test sample sizes. For each dataset, a single algebraic model was computed using only training data and evaluated via fewest misses method and also via logistic regression on the AML output. MLP results show test accuracy of the best-performing configuration on validation data (neurons per layer shown) from 360 configurations and mean ± std of test accuracies across all configurations. MLP configurations were obtained using grid search over learning rates and architectures with up to $4$ hidden layers of varying neuron counts (see \textbf{Methods}).
 }
    \label{table:performance_comparison} %
\end{table}

We also evaluated models obtained with Sparse Crossing in several medical datasets (MEDMNIST, \cite{yang2023medmnist}), as well as in fashionMNIST \cite{xiao17} and CIFAR-10 \cite{Kri09}. These datasets have grayscale images, and CIFAR, bloodMNIST and dermaMNIST also in color. In order to embed the color and grayscale values of the images, instead of using two constants, one for black and another for white, we use two sets of constants with as many constants as grayscale intensities. These sets are structured as intensity-ordered chains, one ascending and the other descending. For the pixel located at position $i, j$ in the image matrix and color channel $k$, we define the chains
\begin{equation}
\begin{aligned}
l_{i,j, k}(intensity) &\leq l_{i,j,k}(intensity + 1) \\
g_{i,j, k}(intensity + 1) &\leq g_{i,j,k}(intensity),
\end{aligned}
\end{equation}
where $l_{i,j,k}(intensity)$ and $g_{i,j,k}(intensity)$ are constants. An individual intensity value is then embedded as an idempotent summation of two constants: 
\begin{equation}
\begin{aligned}
l_{i,j, k}(intensity) \odot  g_{i,j, k}(intensity),  
\end{aligned}
\end{equation}
and an image is represented by a term equal to the idempotent summation along all pixel locations and color channels:
\begin{equation}
\begin{aligned}
term(\textnormal{image}) = \odot_{i,j, k}\, \big(\,l_{i,j, k}(intensity(i,j, k, \textnormal{image})) \odot  g_{i,j, k}(intensity(i,j, k, \textnormal{image}))\,\big).  
\end{aligned}
\end{equation}
For images with three color channels, each pixel is encoded as the idempotent summation of six constantans, three in ascending chains and three in descending chains. This embedding uses $2 \times resolution$ constants and $2 \times (resolution - 1)$ positive duples for each pixel and color channel. The original intensity resolution of $256$ gray levels per channel was retained for some of the datasets while others were downsized to $20$ gray levels per channel to reduce computational load (\textbf{Methods}). 

\textbf{Table \ref{table:performance_comparison}} shows that the test accuracy of a single algebraic model using logistic regression on top is comparable to the best performing MLP. Note that to obtain the AML model we use only training data, whereas for MLPs we also use validation data to select the hyperparameters of the best-performing model out of 360 configurations and for early stopping of training to prevent overfitting (\textbf{Methods}). 

\section*{Learning without data}

So far we have seen that the algebraic embedding approach and the subdirect decomposition of its models into atoms can be used to learn from data. This method extends beyond data-driven learning to axiom sets that describe a problem without containing any data. For example, it is possible to train an algebra to learn how to solve Sudoku puzzles or to form complete Sudoku boards starting from an empty grid (see \cite{SecondPaperArX} and \textbf{Methods}). In this case, learning occurs without providing any examples, with the axioms describing the constraints of a correct Sudoku board and the goal of the game.

As with data-driven tasks, learning for these problems consists of discovering discriminative atoms of the freest model of the axioms. Using Sparse Crossing, this process of discovery typically occurs gradually, after processing multiple batches, each containing the complete set of axioms. 
To better understand why Sparse Crossing also works in these problems, consider the following result. We proved in \cite{ThirdPaperArX} that for a type of embedding of a problem we call ``explicit embedding”, each solution of the problem has a model atomized by a subset of non-redundant atoms of the freest model of the axioms. For example, consider an explicit embedding for the Hamiltonian cycle problem. For this embedding, each solution model, i.e.\ each Hamiltonian cycle, is atomized by a subset of the non-redundant atoms of the freest model. Since most atoms are redundant, i.e.\ are unions of non-redundant atoms, this property severely restricts the size of the atom space that contains the solutions. As Sparse Crossing is designed to be effective at finding non-redundant atoms (see \textbf{Supplementary Section \ref{SparseCrossing:smarter}}), this may explain its effectiveness in solving these problems.

\begin{center}
\begin{tabular}{|c|c|c|c|c|}
\hline
Graph & First & Median & All & SLH Transforms \\
\hline
G1 & 9 & 773 & 3651 & 13356 \\
\hline
G2 & 12 & 124 & 471 & 5078 \\
\hline
G3 & 808   &  6798 & 46419  & 172316 \\
\hline
G4 & 11 & 1008 & 3379 & 266 \\
\hline
G5 & 1818 & 10492 & 64013 & 81571 \\
\hline
G6 & 28 & 437 & 887 & 370 \\
\hline
G7 & 3560 & 28202 & 292521 & 412275 \\
\hline
G8 & 207 & 1838 & 5823 & 666801 \\
\hline
G9 &  8282 & 130717  &   & 472180 \\
\hline
G10 & 434 & 768 & 2907 & 285 \\
\hline
\end{tabular}
\captionof{table}{Results of applying Sparse Crossing to graphs 1 to 10 of the FHCP challenge set \cite{FHCP1001}. Each graph was independently run $10$ times. Each run consists of a series of sparse-crossing batches. We make a single attempt to find a cycle after each sparse-crossing batch. Column ``First'': number of attempts needed to obtain the first Hamiltonian cycle in any of the runs. Column ``Median'': number of attempts needed so $5$ out of the $10$ runs find a first Hamiltonian cycle. Column ``All'': number of attempts needed so the $10$ runs find a Hamiltonian cycle. Column ``SLH Transforms'': number of graph transformations needed to find a first path following the Snakes and Ladders Heuristic algorithm \cite{SLH}, a state of the art algorithm for Hamiltonian cycles. }
\label{tab:tableG1G10}
\end{center}

\begin{figure}[h!]
\centering
\includegraphics[width=0.8\textwidth]{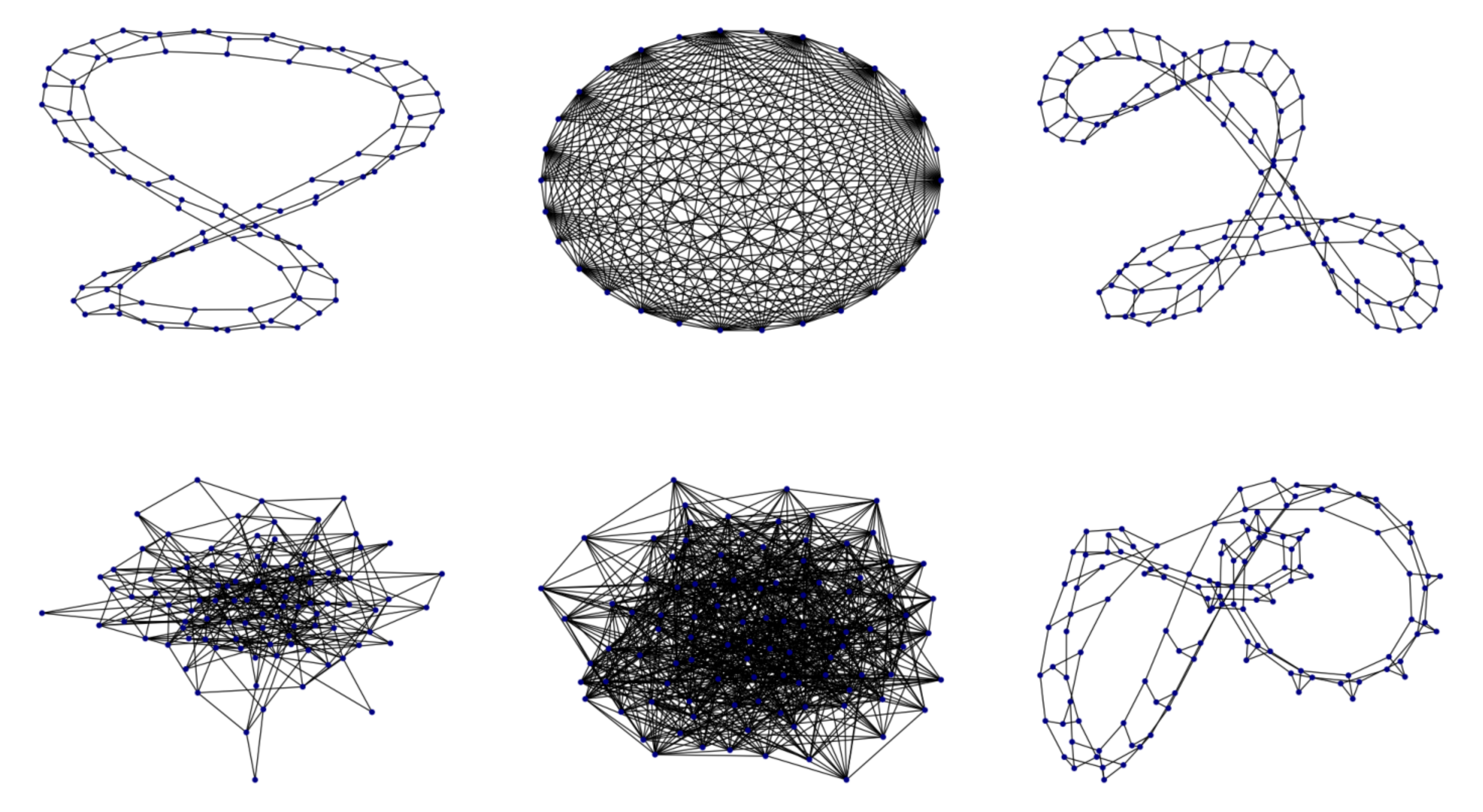}
\caption{Hamiltonian cycles obtained using AML for different graphs. Sparse Crossing obtains Hamiltonian cycles in randomly generated graphs of variable edge density (first two graphs of the bottom row), modified Flower Snarks (SNm\_124, bottom row, right), graph G7 of the FHCP set \cite{FHCP1001} (top row, left), Sheehan graphs of various sizes (top row, center) and generalized Peterson graphs (GPN\_122, top row, right).}
\label{fig:graphImages}
\end{figure}

We illustrate how AML can deal with formal problems in the case of Hamiltonian cycles. We need to explain as axioms that we want a closed loop path that visits each node of a graph exactly once. In \textbf{Methods}, we give a complete description of these axioms, and here we discuss a few of them. Assume we have a graph with $v$ nodes and $e$ edges. To find Hamiltonian cycles we can use an embedding with the following constants: a constant $V_i$ for each graph node, a constant $E_k$ for each edge, a constant $P$ to refer to the path we want to compute and a constant $W$ to encode constraints and allow for training with results obtained during the process of computing Hamiltonian cycles if we desire. In addition, we use constants: $nE_k$ for the absence of edge $k$, as many auxiliary constants $Z_k$ as graph edges and as many context constants $g_k$ and $h_k$ as graph edges. It is also possible to specify that we want a connected path, for which we use as many $id_i$ constants as graph nodes. This gives a total of $2v + 5e + 2$ constants where $v$ is the number of nodes and $e$ is the number of edges in the graph.

For example, we express for the topology of the graph with the set of positive axioms:  \[
V_{r(k)} \odot V_{s(k)} \leq E_k,
\]
where $r(k)$ and $s(k)$ are the indexes of the two nodes of edge $E_k$. To describe a path we use the following axioms. For each node $i$ and for each couple of edges $E_y$ and $E_z$,  we use: \[
P \odot ( \odot_{x; \, x \not\in \{y,z\}} nE_x) = E_y \odot E_z  \odot P
\]
where the idempotent summation $\odot_{x; \, x \not\in \{y,z\}} nE_x$ runs along the indexes of every edge of the node $i$, except edges $y$ and $z$. This axiom specifies that in the context of a path $P$, having two edges present, $E_y$ and $E_z$, that share the same node, is equivalent to having every other edge of the node absent, which follows from the fact that there cannot be more than two edges of $P$ incident to the same node.

There are other sets of positive and negative axioms needed, given a total of $2v + 1$ negative axioms and approximately $2 (e^2 /v) + 7e + 2v + 2$ positive axioms, as described in \textbf{Methods}.

Once a model $M$ of the embedding axioms is produced, we interpret that path $P$ has edge $E_k$ if and only if $(E_k \leq P)$ is valid in the model. In this way, it is possible to determine if a model contains a solution or not. In the experiments reported in \textbf{Table \ref{tab:tableG1G10}} the model produced after each sparse-crossing batch was interpreted in this manner, thereby resulting in an ``attempt'' per batch.

Optionally, we can add to the axioms information we find while computing Hamiltonian cycles. If a path is produced that cannot be completed with additional edges, we can add to the axioms:
\[
W \leq  \odot_u E_u.
\]
where the idempotent summation $\odot_u E_u$ sums along the edges of the unwanted path, and then we specify with another axiom that the path must not be like these unwanted paths:
\[
W \not \leq P.
\]
It should be understood that the constraints defined in our embedding are soft, in the sense that they are more an invitation than a hard constraint. For example, there are ``bad'' models of the embedding axioms for which $P$ contains every node but does not have enough edges to justify their presence. However, experimental results consistently show that with some training, ``good''  models are produced, and in fact, they are produced early even for hard graphs (see \textbf{Table \ref{tab:tableG1G10}}). The fact that good models are found, despite the potential existence of many more bad models, suggests that good models provide a simpler standard interpretation of the constraints compared to bad models. This simpler interpretation makes good models more likely to be discovered by Sparse Crossing. For example, for Sheehan graphs of any size (e.g. SH\_66 of the FHCP challenge set \cite{FHCP45}) our embedding always produces the only existing cycle in the first attempt, suggesting that non-standard interpretations of the embedding do not exist for Sheehan graphs. In fact, if the Hamiltonian cycle solution is discarded by adding $W \leq  \odot_u E_u$ to the embedding, where the summation runs along the edges of the cycle, the embedding becomes inconsistent. This makes sense, as Sheehan graphs have only one Hamiltonian cycle, and shows that there are no other interpretation of the constraints in this case.  

Sparse Crossing could find Hamiltonian cycles, using this embedding, in a wide range of random and hard graphs (see \textbf{Figure \ref{fig:graphImages}} and  \textbf{Table \ref{tab:tableG1G10}}). Although this method can find paths in very few attempts, each attempt is time consuming (every batch takes about 0.5 to 5 seconds depending on the graph in a regular desktop computer), making this method much slower than state of the art algorithms such as the Snakes and Ladders Heuristic algorithm \cite{SLH}. However, note that our axioms simply describe the graph and the goal of the task and do not encode any method to find the solution.

\section*{Discussion}

We have introduced Algebraic Machine Learning (AML) as a novel approach to automated learning that uses an algebraic decomposition as the basis for learning and generalization. It works by encoding tasks into axioms of an algebra and constructing atomized models of these axioms. Learning results from the cumulative discovery of certain atoms of the freest model. This process occurs gradually, using discovered atoms to find more and better atoms. Certain subsets of atoms from the freest model serve as generalizing models. We demonstrated the versatility of this method across problems of very different nature, using image classification and obtaining Hamiltonian cycles as examples.

We find that AML, without incorporating image-specific inductive biases, can classify images with accuracy comparable to the best multilayer perceptrons identified through grid hyperparameter search using a validation dataset. We also demonstrate that the same method finds Hamiltonian cycles in few attempts compared to state-of-the-art heuristics and in graphs known to be some of the hardest for the task. 

An advantage of AML is that the models grow autonomously, thereby eliminating the need to predefine an architecture. The inherent absence of overfitting, combined with the minimal set of hyperparameters (see \textbf{Methods}), renders the use of a validation dataset unnecessary. Another potential advantage of AML stems from the additivity of the atomized representation, which can be used to construct larger models from the union of the atom sets of independently computed models. 

We demonstrate that if the data can be explained by rules that can be expressed in the form of axioms in a semilattice, the algebraic model of the data shares all the discriminative atoms (those useful for generalization) with the freest model of the rules. Furthermore, based on simple probabilistic considerations and the fact that atoms cannot grow without limit, we expect to observe atoms of the freest model of the rules emerging from the embedding of small amounts of data. This ability that AML has to find the underlying rules in the data suggests a potential for model transparency and explainability.

AML provides a different basis for learning that does not use optimization or search and differs considerably from all other known methods and, particularly, from Statistical Learning approaches. This novel perspective could help enhance our understanding of learning and intelligence and potentially offer lessons applicable to improve other methods. For example, the role played by the freest model, understood as the model of what can be proven from the axioms, and the conceptualization of learning as a form of weakened deduction, offer unique insights that could be applicable to other methods.

Hybrid methods combining the algebraic approach and statistical learning show significant potential. For image datasets, the most effective approach combines logistic regression with the algebraic model, suggesting that data is separable into between algebraic and statistical components. Supporting evidence includes the lack of improvement when using validation data or replacing logistic regression with a multi-layer network. Furthermore, optimal performance occurs when the algebraic model achieves zero training error, possibly because this prevents the statistical layer from compensating for patterns that should be better captured algebraically.

In this work, we use atomized semilattices due to their simplicity and sufficient expressive power. However, we hypothesize that the underlying learning method relies primarily on the subdirect decomposition rather than on the particularities of the semilattice algebra. We expect that AML can be implemented with other algebras.

\subsection*{Code availability}

We have made available an open-source Python/C hybrid implementation of Sparse Crossing: \href{https://github.com/Algebraic-AI/Open-AML-Engine}{https://github.com/Algebraic-AI/Open-AML-Engine}. The dual-language approach allows for seamless instrumentation, enabling researchers to explore and easily modify the algorithm in Python while maintaining the performance advantages of C. A decorator ``@tryfast'' in every computationally intensive function provides a way to choose between running the function in Python or in C, facilitating code instrumentation and modification. The code can also compute the Full Crossing algorithm. The repository includes example embeddings for various tasks, including Hamiltonian cycle finding, Sudoku, and MNIST handwritten digit classification.

\section*{Acknowledgments} 

We are grateful for the support from Champalimaud Foundation (Lisbon, Portugal), from Portuguese national funding through FCT in the context of the project UIDB/04443/2020, and from the European Commission provided through projects H2020 ICT48 \emph{Humane AI; Toward AI Systems That Augment and Empower Humans by Understanding Us, our Society and the World Around Us} (grant $\# 820437$) and the H2020 ICT48 project \emph{ALMA: Human Centric Algebraic Machine Learning} (grant $\# 952091$). 

\newpage
\section*{Methods}

\numberwithin{equation}{section}

\subsection*{AML models} \label{Methods:AML_models}

\textbf{Images}. The smallest datasets from MEDMNIST \cite{yang2023medmnist} were kept in their original 256-level grayscale depth. For larger medical images, FashionMNIST, CIFAR-10, to speed up computations, the grayscale intensity resolution was reduced from the original 256-level depth to 20 equidistantly distributed levels.

\textbf{Training in Sparse Crossing}. All the datasets were processed following the same protocol. Batch size starts with $500$ images and increases linearly until reaching $2/3$ of the training set in batch $500$. Sparse-Crossing gives a model $M_{i}$ per batch $i$, which we call master model, and ``union models'' that take into account previous batches (see {\textbf{Supplementary Section \ref{SparseCrossing:iterative}}}). Training stops when the ``union model'' has $0$ error in the training set. A single AML model was obtained for each dataset.

\textbf{Hyperparameters in Sparse Crossing}. Sparse Crossing has $4$ hyperparameters. These hyperparameters were set manually and are fixed, i.e., they are not optimized for each individual dataset. The manual setting was carried out based on experience gathered from many synthetic datasets and in MNIST. All other datasets used in this study had no influence on the manual setting of the hyperparameters.\newline
1. {\textit{Simplification threshold $\gamma$}}: during the process of sparse-crossing the positive axioms, if the number of atoms of the master model (see \textbf{Supplementary Section \ref{SparseCrossing:iterative}}) grows from a size $N$ to a size larger than $\gamma N$, a call to a simplification routine triggers. The simplification consists of discarding atoms with the constraint of keeping the traces of all the constants invariant (see \textbf{Algorithm \ref{SparseCrossing:simplifyFromConstants}}). The simplification parameter has an impact on computation time and it may or may not have an impact on the quality of the models produced. The value $\gamma = 1.5$ was used for all the image datasets, while for Sudoku and Hamiltonian cycles the value $\gamma = 1.1$ was set. \newline
2. {\textit{Batch size}}: The batch size has an impact on computation time and model test accuracy. For image datasets, we used a policy of making the batch size grow linearly as training progresses, see \textbf{Training in Sparse Crossing}. For Sudoku or Hamiltonian Cycles, all the positive and negative duples are presented at each batch.\newline
3. {\textit{Union model fractioning parameter $\kappa$}}: the atoms of the dual are either associated to negative duples or to pinning terms. Let $D$ be the set of atoms of the dual, $D_N$ the set of atoms associated to pinning terms and $D_R$ the set of atoms associated to negative duples, so $\vert D  \vert = \vert D_N \vert  + \vert D_R \vert $. The fractioning parameter selects, at random and at each batch, a subset of $S \subseteq D_N$ such that $\vert D_R  \vert \geq \kappa (\vert D_R  \vert + \vert S \vert)$. In other words, $\kappa$ is the minimal proportion of atoms associated to negative duples that we want in the dual. Since the number of pinning terms increases with training, if this fractioning does not take place, the proportion of atoms associated to duples decreases. We found that ensuring a proportion $\kappa$ of atoms associated to duples helps increase atom variability, i.e.\ fractioning helps explore a larger volume of the atom space. For image datasets we used  $\kappa = 0.1$ while for Hamiltoinian cycles we observed that larger values, like  $\kappa = 0.5$, gave better results. We found this parameter to have a significant impact in model performance, particularly for smaller training sets. \newline
4. \textit{Model reduction parameter $\delta$}: Since the accuracy remains approximately constant for a wide range of atomization sizes (see \textbf{Supplementary Figure \ref{figure:mnistpfppfn}}), it is possible to reduce the size of the union model $N$. To extract a good generalizing model from the union model, a subset of its atoms with size $\delta \vert N \vert$ is extracted using the method described in \textbf{Subset selection}. Size reduction with parameter $\delta = 0.1$ was used before the logistic regression and the fewest misses evaluations for all image datasets. For Sudoku or Hamiltonian cycle problems, no reduction was applied, as each solution is extracted from the master model and not from the union model.

\textbf{Subset selection}. Out of the Sparse Crossing procedure we obtain a set of atoms, from which we extract the following subset. Good generalizing models need subsets of atoms that are individually discriminative, collectively discriminating the entire training set and with low correlation. To build a subset $S$ with these characteristics, we first randomly sort atoms. Starting with $S$ empty and reading the atoms in order, an atom $\phi$ is added to $S$ only if there is a negative duple of the training set discriminated by $\phi$ and by no other atom of $S$. This results in a subset of atoms that discriminates the entire training set, of cardinality smaller than the number of negative duples of the training set. We add various subsets of atoms selected in this manner until reaching a model of a size equal to $1/10$ of the initial model obtained from Sparse Crossing. The atoms that are not associated to labels (those which upper segment contain no label constants) are removed from the model, as they play no role in associating labels to term images. This protocol results in good generalizing models ten times smaller than the initial model. 

\textbf{Logistic regression on top of AML}. If we are interested in adding statistical information to AML, a simple way is to use the AML model as input to logistic regression in the following way. The atoms that are in the lower segment of the image term are given a value of $+1$ and the atoms that are not are given a value of $-1$. For each image, the input to the logistic regression is then a sequence of $+1$ and $-1$ values. Each element of the sequence connects with a linear weight to each of $N$ $\operatorname{softmax}$ outputs, one per class. Only the training dataset was used to find optimal parameters, with Adam optimizer \cite{adam} and cross-entropy loss \cite{deep_learning_goodfellow2016}.

\subsection*{Multi-layer perceptrons}

To build MLP models, we use the validation dataset to optimize architecture parameters and avoid overfitting by early stopping of training. We evaluated a family of two, three and four hidden layer multilayer perceptrons with ReLU activations. More concretely, we perform a grid search over the number of neurons in the first hidden layer (512, 2048 or 4096 hidden units) and the second hidden layer (256, 1024 or 2048 hidden units), using Ray Tune, \cite{liaw2018}, with the goal of minimizing validation loss. The third layer, when it exists, is allowed to have 128, 256 or 512 hidden units, and the fourth layer, when it exists, can have 128 or 256 hidden units. For the third and fourth layers, a random sample is performed for the sizes. We perform $90$ runs using two layers, $180$ runs using three, and an additional $90$ runs with four layers, for a total of $360$ train runs. Training runs for $240$ iterations or until the validation loss does not improve for 10 iterations. In each run, we uniformly sample the learning rate ($5\cdot10^{-4}$, $10^{-4}$, $5\cdot10^{-5}$ or $10^{-5}$) and the L2-regularization coefficient ($10^{-3}$, $5\cdot10^{-4}$, $10^{-4}$ or $0$). We use the ADAM optimizer to minimize cross-entropy loss. 

\subsection*{Semantic Embeddings} 

A detailed analysis of the concept of semantic embeddings as an axiomatic extension of the theory of semilattices can be found in \cite{SecondPaperArX}.

\subsubsection*{Embedding for Sudoku} \label{Methods:sudoku_embedding}

The embedding for Sudoku is presented in \cite{SecondPaperArX}, with a comprehensive study of its properties and the resulting atomized models. Additionally, within the open-source engine at \href{https://github.com/Algebraic-AI/Open-AML-Engine}{https://github.com/Algebraic-AI/Open-AML-Engine}, exemplary files ``example02\_Sudoku.py'' and ``embedding\_Sudoku.py'' are also provided. 

\subsubsection*{Embedding for Hamiltonian cycles} \label{Methods:hamiltonian_embedding}

Consider the following sets of constantans:

\begin{itemize}
\item $V_i$: A constant for each graph node
\item $E_k$: A constant for each edge
\item $P$: A constant to refer to the path we want to compute
\item $W$: A constant to encode constraints and allow for training
\item $nE_k$: A constant for the absence of edge $k$
\item $Z_k$: Auxiliary constants, as many as graph edges
\item $id_i$: a path ``id'' constant associated to node $i$ 
\item $g_k$ and $h_k$: Context constants, as many as graph edges

\end{itemize}

This gives a total of $2v + 5e + 2$ constants, where $v$ is the number of nodes and $e$ is the number of edges in the graph.

We start by embedding the topology of the graph.
Let $r(k)$ and $s(k)$ be the index of the two nodes of edge $E_k$. The edges are undirected so it does not matter which of the two nodes is $r(k)$ or $s(k)$. For each (undirected) edge $k$ joining nodes $V_{r(k)}$ and $V_{s(k)}$ we define a positive duple: \[
V_{r(k)} \odot V_{s(k)} \leq E_k.
\]

The embedding constant $Z_k$ represents either an edge or its absence and is defined with: \[
Z_k \leq E_k \odot nE_k,
\]
Think about $Z_k$ as a kind of weak variable that we wish to be equal to either $E_k$ or to $nE_k$ but that can take any value in between.
The path $P$ we want to find passes through every node and it is formed with edges, so we add:   \[
\odot_i V_i \leq P,
\]  \[
\odot_k Z_k = P.
\]
 
For the constant $W$, which we will use to learn the ``wrong paths'', we start with the following duples; for each edge $k$ it is a wrong path one that simultaneously has the constant of the edge and the constant for the absence of the edge:  \[
W \leq E_k \odot nE_k \odot P.
\]
Since we want our path not to be a wrong path we also impose the additional negative axiom:  \[
W \not\leq P.
\]

Then we describe the concept of path with the help of the constants $nE_k$. For each node $i$ and for each couple of edges $E_y$ and $E_z$ We use: 
\[
P \odot (\odot_{x; \, x \not\in \{y,z\}} nE_x) = E_y \odot E_z  \odot P
\]
where the idempotent summation $\odot_{x; \, x \not\in \{y,z\}} nE_x$ runs along the indexes of every edge of node $i$, except for $y$ and $z$. There are a variable number of these positive duples depending upon the graph, on the order of $ 2v \left(\frac{e}{v}\right)^2 $. 
 
We need some negative duples in the embedding (usually, the fewer the better). It is enough with one negative duple for each node $i$ establishing that the presence of node $i$ only depends upon the presence of edges incident to node $i$ and it is independent of everything else: \[
V_i \not\leq  (\odot_{j;\,j \not= i} (V_j \odot id_j))  \odot (\odot_k nE_k)  \odot  (\odot_{t;\, i \not\in \{r(t), s(t)\}} E_t)
\]
where the idempotent summation $\odot_{j;\,j \not= i} (V_j \odot id_j)$ sums along all the nodes except $i$, the idempotent summation $\odot_k nE_k$ sums along every edge of the graph, and $\odot_{t;\, i \not\in \{r(t), s(t)\}} E_t$ sums along every edge not incident to node $i$.   

To specify we want a connected path, we use the path identity constants $id_i$. These constants become equal for the nodes in the same path. The identities of adjacent nodes become equal in the presence of a connecting edge:
 \[
E_k \odot id_{r(k)} = id_{s(k)} \odot E_k.
\]
To require that the path connects every two nodes $i$ and $j$ we add: 
\[
P \odot id_{i} = id_{j} \odot P,
\]
which gives $v^2$ duples that are equivalent to just $2v$ duples. To convey the meaning of the node identity we add the following set of negative duples:
 \[
id_i \not\leq  (\odot_{j;\,j \not= i} (V_j \odot id_j))  \odot (\odot_k nE_k)  \odot  (\odot_{t;\, i \not\in \{r(t), s(t)\}} E_t),
\]
which has the same right-hand side than the negative duples above.

Using context constants (see \cite{ThirdPaperArX}) ensures the solution models are spawn by non-redundant atoms and increase the probability of finding a solution; for each edge $k$ we add a context in which  $Z_k$ is equal to $E_k$ and another context in which $nZ_k$ is equal to $nE_k$: \[
g_k \odot Z_k  =  E_k \odot g_k,
\] \[
h_k \odot Z_k  =  nE_k \odot h_k.
\]

The embedding theory has a total of $2v + 1$ negative duples and around $2 (e^2 /v) + 7e + 2v + 2$ positive duples.

Additionally it is possible to discard paths extracted from the attempts made; if a path is produced that cannot be completed with additional edges, add to the axioms:\[
W \leq  \odot_u E_u.
\]
where the idempotent summation $\odot_u E_u$ sums along the edges of the unwanted path.

Once a model $M$ of the embedding axioms is produced (we used the ``master'' model, see \textbf{Supplementary Section \ref{suppSection:SparseCrossingInDepth}}), we interpret that path $P$ has edge $E_k$ if and only if $E_k \leq P$ is valid in the model $M$.

%\bibliographystyle{unsrt}
%\bibliography{algebraic}

%
\newpage

The Supplementary Information is divided in four sections: \textbf{Supplementary Section \ref{suppSection:atomizedStimlattices}} reviews the main results of atomized semilattices, included for completeness, and presents a few new results necessary to support this paper. \textbf{Supplementary Section \ref{suppSection:discoveryOfRulesInData}} is devoted to the discovery of underlying rules in data from an algebraic perspective. \textbf{Supplementary Section \ref{section:genSubsetsFreestModel}} presents a probabilistic analysis of the expected false positive and negative ratios. \textbf{Supplementary Section \ref{suppSection:SparseCrossingInDepth}} offers an in-depth analysis of the Sparse Crossing algorithm, including pseudocode. 

\section{Atomized Semilattices} \label{suppSection:atomizedStimlattices}

In this Supplementary Section we provide a review of the background on atomized semilattices taken from \cite{SecondPaperArX}, as well as a few new results (\textbf{Supplementary Section \ref{appendix:additional_ressults}}) needed to support the main text and other Supplementary Sections. For an in-depth analysis of finite and infinite atomized semilattices, see \cite{SecondPaperArX, InfX}.

\subsection{Definitions} \label{subsection:definitions}

\begin{definition} 
\label{definition:theory_of_semilattices}
A \textbf{semilattice} is an algebra with a single binary function $\odot$, that satisfies the commutative, associative and idempotent properties:
$\forall x \forall y \,\,[(x \odot y) =(y \odot x)],$ 
$ \forall x \forall y\forall z [x \odot (y \odot z) =(x \odot y) \odot z]$ and  
$\forall x \,\,[(x \odot x) =x]$ 
\end{definition}

\begin{definition}
\label{definition:component_constants}
The \textbf{component constants} of a single constant $c$ is defined as the constant itself, ${\bf{C}}(c)  = \{c\}$, and the component constants of a term $t = c_1 \odot  c_2 \odot... \odot c_n$ as the set ${\bf{C}}(t)  = \{c_1,c_2,...,c_n\}$.
\end{definition}

\begin{definition}
\label{definition:duple}
\textbf{Positive and negative duples in a model}. We use the word \emph{duple} to refer to an ordered pair of terms, $r \equiv (a, b)$. This notation is silent about whether it is valid or not in a model. We say that a duple is positive in a model $M$, and we denote it by $r^{+}$, if the duple is valid in the model, $M \models (a \leq b)$. Similarly, we say a duple is negative, and write $r^{-}$, if $M \not\models (a \leq b)$ or, equivalently, $M\models (a \not\leq b)$. For a set of positive duples we use the notation $R^{+}$, with $R^{+}=\{r_1^+, r_2^+, ..., r_N^+\}$ and for a set of negative duples of positive duples we use the notation $R^{-}$, with $R^{-}=\{r_1^-, r_2^-, ..., r_N^-\}$.
\end{definition}

\begin{definition}
\label{definition:theory_of_a_model}
\textbf{Theory of a model}. For any model $M$, its theory, denoted $Th(M)$, is the set of sentences (duples) satisfied by $M$.  The subscript ``0" in $Th_{0}(M)$ specializes sentences without quantifiers, i.e.\ atomic and negative atomic sentences. We use $Th_{0}(M^{+})$ and $Th_{0}(M^{-})$ when we want to refer to all the positive or negative duples satisfied by $M$, respectively.
\end{definition}

\begin{definition}
\label{definition:freer_or_as_free_model}
A semilattice $M_{1}$ is \textbf{freer than or as free} as the semilattice $M_{2}$ if for every duple $r$ for which $M_{2} \models r^{-}$ we also have $M_{1} \models r^{-}$. Equivalently, $Th_0^{-}(M_{2})  \subseteq Th_0^{-}(M_{1})$. 
\end{definition}

\begin{definition}
\label{definition:freest_model}
The \textbf{freest model} over the constants $C$ of the set of positive duples $R^{+}$, 
$F_{C}(R^{+})$, is the model such that if $F_{C}(R^{+}) \models r^{+}$ for any duple $r$ then all other models of $R^{+}$ also satisfy $r^{+}$. 
\end{definition}

\begin{definition} 
\label{definition:finiteAtomizedSemilattice}
An \textbf{atomized semilattice} over a set of constants $C$ is a structure $M$ with elements of two sorts, the regular elements in Latin letters  $\{a, b, c,...\}$ and the atoms in Greek letters $\{\phi, \psi,...\}$, with an idempotent, commutative and associative binary operator $\odot$ defined for regular elements and a partial order relation $<$ (i.e.\ a binary, reflexive, antisymmetric and transitive relation) that is defined for both regular elements and atoms, such that the regular elements are either constants or idempotent summations of constants, and $M$ satisfies the axioms of the operations and the additional:
\[
\forall \phi \exists c : (c \in C) \wedge (\phi < c), \tag{AS1} 
\]\[
\forall \phi \forall a \,(a \not\leq \phi),  \tag{AS2} 
\]\[
\forall a \forall b \, (  a \leq b \,  \Leftrightarrow  \neg \exists \phi : ( (\phi < a)  \wedge   (\phi \not< b))),   \tag{AS3} 
\]\[
\forall \phi  \forall a \forall b  \,  (\phi < a \odot b \, \Leftrightarrow  (\phi < a) \vee (\phi < b)),   \tag{AS4} 
\]\[
\forall c \in C \,\, ((\phi < c) \Leftrightarrow (\psi < c)) \,  \Rightarrow (\phi = \psi),   \tag{AS5} 
\]\[
\forall a \exists \phi : (\phi < a).   \tag{AS6} 
\]
\end{definition} 

\begin{definition} \label{definition:atomInAsAsetDefinition}  
An \textbf{atom} $\phi$ of an atomized semilattice $M$ over $C$ is determined by its upper constant segment, the set $U^c(\phi) \subseteq C$, as follows: $\phi <_M c$ if and only if $(\phi \in M) \wedge (c \in U^c(\phi))$. 
\end{definition} 

\begin{definition}
\label{definition:lower_atomic_segment}
The \textbf{lower atomic segment} of a regular element $x$ of a model $M$ atomized by a set of atoms $A$ is  $L^{a}_{M}(x) = \{\phi: (\phi < x) \wedge (\phi \in A)\}$.
\end{definition}

\begin{definition}
\label{definition:upper_constant_segment}
The \textbf{upper constant segment} of an atom $\phi$, in any model, is the constants in which the atom $\phi$ is in, ${\bf U}^{c}(\phi) = \{c: (\phi < c) \wedge (c \in C)\}$. Atoms are ``universally" defined independently of a particular model as a consequence of Theorem \ref{atomicSegmentFromTermTheorem} (iv) and Theorem \ref{atomicSegmentFromTermTheorem} (iv).
\end{definition}

\begin{definition}
\label{definition:atom_wider_than_another_atom} We say that an atom $\phi$ is \textbf{wider than} an atom $\eta$ if ${U^{c}}(\eta) \subsetneq {U^{c}}(\phi)$, with the upper constant segment $U^{c}$ as in Definition \ref{definition:upper_constant_segment}.
\end{definition}

\begin{definition}
\label{definition:redundant_atom}
\textbf{Redundant atom}. An atom $\phi$ is \emph{redundant} in model $M$ if and only if for each constant $c$ such that $\phi < c$ there is at least one atom $\eta < c$ in $M$ with $\phi$ larger than $\eta$, with the notion of larger as in Definition \ref{definition:atom_wider_than_another_atom}.
\end{definition}

\begin{definition}
\label{definition:discriminant}
The \textbf{discriminant} of terms $a$ and $b$ in the atomized model $M$, written as $dis_{M}(a,b)$, is the set of atoms in $a$ that are not in $b$. Using the definition of lower atomic segment in Definition \ref{definition:lower_atomic_segment}, we can write it as $dis_{M}(a,b)={L^{c}_{M}}(a) \setminus {L^{c}_{M}}(b)$.
\end{definition}

\begin{definition}
\label{definition:atom_union_of_atoms}
The \textbf{union of atoms} $\phi$ and $\psi$ is an atom represented as $\phi \bigtriangledown \psi$ with upper constant segment is $U^{c}(\phi \bigtriangledown \psi) = U^{c}(\phi) \cup U^{c}(\psi)$, with $U^{c}$ the upper constant segment in Definition \ref{definition:upper_constant_segment}.
\end{definition}

\begin{definition}
\label{definition:full_crossing}
\textbf{Full-crossing}. Let $M$ be our model and let $r \equiv (r_{L},r_{R})$ be a duple that is not valid in the model, $M \not \models r^{+} \equiv(r_{L} < r_{R})$. Let $H$ be the discriminant of the terms $r_{L}$ and $r_{R}$, $H=dis_{M}(r_{L},r_{R})$, with the discriminant in Definition \ref{definition:discriminant}. Let $B$ be the set of atoms of $M$ that are in $r_{R}$, $B=L^{a}_{M}(r_{R})$ with $L^{a}_{M}$ the lower atomic segment in Definition \ref{definition:lower_atomic_segment}. The \textit{full-crossing} of duple $r$ into model $M$ gives a new model $\square_{r} M$ (read as ``the full-crossing of $r$ into $M$") atomized by:

\begin{equation*}
    \square_{r} M=(M \setminus H) \cup (H \bigtriangledown B),
\end{equation*}
where $\,\square_{r}$ is the Full Crossing operator and $H \bigtriangledown B$ the set of atoms resulting from all pairwise unions of each atom of $H$ and each atom of $B$, 
\begin{equation*}H \bigtriangledown B\equiv \{\lambda \bigtriangledown \rho:(\lambda \in H) \wedge (\rho \in B)\},
\end{equation*}
with the atom union of atoms, $\lambda \bigtriangledown \rho$, in Definition \ref{definition:atom_union_of_atoms}.  
\end{definition}

\begin{definition}
\label{definition:causal_set}
\textbf{Causal set}. Let $P$ and $Q$ be two sets, each of positive and negative duples. $P$ is a causal set of $Q$ when every duple of $Q$ is a logical consequence of the duples in $P$.
\end{definition}

\begin{definition}
\label{definition:pinning_tem}
The \textbf{pinning term} of atom $\phi$ in a semilattice over the constants $C$ is the term with component constants $C-{U^{c}}(\phi)$.
\end{definition}

\subsection{Review of basic results} \label{supplementary:review}

For completeness, we include here a few results with their proofs extracted from \cite{SecondPaperArX}, so this text is self-contained.

\begin{theorem} \label{atomicSegmentFromTermTheorem}
Let $\,t,s \in F_C(\emptyset)$ be two terms that represent two regular elements $\nu_M(t)$ and $\nu_M(s)$ of an atomized model $M$ over a finite set of constants $C$. Let $\phi$ be an atom, $c$ a constant in $C$ and let $a$ be a regular element of $M$:
\begin{enumerate}[label=\roman*), leftmargin=3cm]
\item $\forall  t \forall c (c \in {\bf{C}}(t)  \, \Rightarrow \,  \nu_M(c) \leq \nu_M(t))$,
\item $\phi < \nu_M(t) \, \Leftrightarrow  \exists  c: ((c \in {\bf{C}}(t)) \wedge  (\phi < \nu_M(c)))$,
\item $(\phi < a) \, \Leftrightarrow  \exists  c : ((c \in C) \wedge (\phi < \nu_M(c) \leq a))$,   
\item $L^a_M(\nu_M(t))  = \{ \phi \in M: {\bf{C}}(t) \cap U^c(\phi) \not= \emptyset \},$
\item $L^a_M(\nu_M(s) \odot \nu_M(t)) \, = \,  L^a_M(\nu_M(t)) \cup L^a_M(\nu_M(s))$,
\item $\nu_M(t) \leq \nu_M(s) \, \Leftrightarrow \,  L^a_M(\nu_M(t))  \subseteq L^a_M(\nu_M(s)).$
\end{enumerate}
\end{theorem}
\begin{proof}
(i) From $t = t \odot c$ and the natural homomorphism $\nu_M(t) = \nu_M(t \odot c) = \nu_M(t) \odot \nu_M(c)$ we get $\nu_M(c) \leq \nu_M(t)$.   \\
(ii) Right to left, $\phi < \nu_M(c) \leq \nu_M(t)$ follows from (i) and, from and the transitivity of the order relation, $\phi < \nu_M(t)$. Left to right can be proven from the fourth axiom of atomized models $\phi < a \odot b \, \Rightarrow  (\phi < a) \vee (\phi < b)$ applied to the component constants ${\bf{C}}(t)$ of $t$. The number of component constants of $t$ is at least 1 and at most $\vert C \vert$ so it is a finite number and we need to apply this axiom a finite number of times to get $\phi < \nu_M(c)$ for some component constant of $t$. This proves (ii), and (iii) is a consequence of (ii) that follows with just choosing any term $t$ of $F_C(\emptyset)$ to represent element $a = \nu_M(t)$.  \\
(iv) Consider an atom $\phi \in L^a_M(\nu_M(t))$ then $\phi < \nu_M(t)$ and from proposition (ii) there is a component constant $c \in {\bf{C}}(t)$ such that $\phi < \nu_M(c)$ which means that $c \in {\bf{C}}(t) \cap U^c(\phi)$. Conversely, if $c \in {\bf{C}}(t) \cap U^c(\phi)$ then $\phi < \nu_M(c) \leq \nu_M(t)$ and $\phi \in L^a_M(\nu_M(t))$.  \\
(v) Since $\nu_M$ is a homomorphism $\nu_M(s) \odot \nu_M(t) = \nu_M(s \odot t)$ and, using proposition (iv) $L^a_M(\nu_M(s \odot t))  = \{ \phi \in M: {\bf{C}}(s \odot t) \cap U^c(\phi) \not= \emptyset \} = \{ \phi \in M: ({\bf{C}}(s) \cup {\bf{C}}(t)) \cap U^c(\phi) \not= \emptyset \} = L^a_M(\nu_M(t)) \cup L^a_M(\nu_M(s))$. Note that this proposition is an alternative way to write axiom AS4.   \\
(vi) It is straightforward from proposition (v) and AS3. 
\end{proof}	
\bigskip

\begin{theorem} \label{axiomatizationProperties} Assume the axiom AS4 of the atomized semilattice and the antisymmetry of the order relation. \\
i) AS3 implies AS3b, with AS3 an axiom of the atomized semilattice and AS3b the partial order of the semilattice: \[
\forall a \forall b \, (  a \leq b \,  \Leftrightarrow \, a \odot b = b).   \tag{AS3b} \]
ii) Assume $\,\forall a \forall b  (\forall \phi ((\phi < a) \Leftrightarrow (\phi < b)) \, \Rightarrow (a = b))$. Then $AS3b \Rightarrow AS3.$  \\
iii) AS3 implies $\,\forall a \forall b  (\forall \phi ((\phi < a) \Leftrightarrow (\phi < b)) \, \Rightarrow (a = b))$.
\end{theorem} 
\begin{proof}
(i) $AS3 \Rightarrow AS3b$: Assume $a \leq b$. AS3 implies $\neg \exists \phi:((\phi < a)  \wedge  (\phi \not< b))$ and then $(\phi < a) \vee (\phi < b) \Leftrightarrow  (\phi < b)$. From here, using AS4, follows $(\phi < a \odot b) \Leftrightarrow (\phi < b)$, and we can use AS3 (right to left this time) to get $(b \leq a \odot b)  \wedge (a \odot b \leq b)$, and from the antisymmetry of the order relation, i.e.\ $(x = y) \Leftrightarrow (x \leq y) \wedge (y \leq x)$, we obtain $a = a \odot b$. \\
Assume $a \odot b = b$. Then $(\phi < a \odot b) \Leftrightarrow (\phi < b)$, and using AS4: $(\phi < a) \vee (\phi < b) \Leftrightarrow  (\phi < b)$ which implies $\neg \exists \phi : ((\phi < a)  \wedge   (\phi \not< b))$ from which, using AS3, we get $a \leq b$.  \\
(ii) $AS3b \Rightarrow AS3$: Assume $a \leq b$. AS3b implies $a \odot b = b$. From AS4 we get that $(\phi < a) \vee (\phi < b) \Leftrightarrow  (\phi < a \odot b = b)$, so $\phi < a$ implies $\phi < b$ and then $\neg \exists \phi : ((\phi < a)  \wedge   (\phi \not< b))$. \\
Assume now that $\neg \exists \phi : ((\phi < a)  \wedge   (\phi \not< b))$. Then $(\phi < a) \vee (\phi < b) \Leftrightarrow  (\phi < b)$ and AS4 leads to $(\phi < a \odot b) \Leftrightarrow  (\phi < b)$. However, we cannot go from here to $a \leq b$ unless we add an axiom such as: $\forall \phi ((\phi < a) \Leftrightarrow (\phi < b)) \, \Rightarrow (a = b)$.  \\
(iii) From $\forall \phi ((\phi < a) \Leftrightarrow (\phi < b))$ follows, using AS3, that $(a \leq b) \wedge (b \leq a)$, and by the antisymmetry of the order relation we get $a = b$.
\end{proof}
\bigskip

\begin{theorem} \label{atomIndependentFromTheRestTheorem}
Let $t,s \in F_C(\emptyset)$ be two terms. If an atom $\phi$ of a model $M$ discriminates a duple $(t ,s)$ in $M$ then it discriminates $(t ,s)$ in any model that contains $\phi$.
\end{theorem} 
\begin{proof}
Consider two atomized models $M$ and $N$ that contain $\phi$. If $\phi$ discriminates a duple $(t ,s)$ in $M$ from part (ii) of Theorem \ref{homomorphismTheorem} follows that $\phi$ discriminates $(t,s)$ in $F_C(\emptyset)$ and then using part (ii) again $\phi$ also discriminates $(t,s)$ in $N$.
\end{proof}	
\bigskip

\begin{theorem} \label{homomorphismTheorem}
Let $t,s \in F_C(\emptyset)$ be two terms and $M$ an atomized model. Let $\phi$ be an atom and $\nu_M : F_C(\emptyset) \rightarrow M$ the natural homomorphism.
\begin{enumerate}[label=\roman*), leftmargin=3cm]
\item $(F_C(\emptyset) \models (t \leq  s)) \, \Rightarrow \,  M \models (\nu_M(t) \leq \nu_M(s))$
\item$(\phi \in M ) \wedge (F_C(\emptyset) \models (\phi < s)) \, \Leftrightarrow \,  M \models (\phi < \nu_M(s)) $
\end{enumerate}
\end{theorem} 
\begin{proof}
i) This proposition is a well-known fact that follows from the fact that $\nu_M$ is a homomorphism. Proposition (i) is provided here for comparison with proposition (ii). \\ 
ii) Note that we use here the same atom $\phi$ in the contexts of two atomized models, $F_C(\emptyset)$ and $M$. Left to right: using Theorem \ref{atomicSegmentFromTermTheorem} (iii), $F_C(\emptyset) \models (\phi < s)$ implies that there is some constant $c$ such that $F_C(\emptyset) \models (\phi < c \leq s)$ and then, from the natural homomorphism, $M \models (\nu_M(c) \leq \nu_M(s))$.  From Definition \ref{definition:atomInAsAsetDefinition}, $F_C(\emptyset) \models (\phi < c)$ requires $c \in U^c(\phi)$ and then, because we assume $\phi \in M$, the same definition allows us to write $M \models (\phi < \nu_M(c))$.  Using the transitive property of the order relation $M \models (\phi < \nu_M(s))$.  \\ 
Right to left is essentially the same proof as left to right except for the fact that we do not need to require $\phi \in F_C(\emptyset)$ as this is always true for any atom.
\end{proof}	
\bigskip

\begin{theorem} \label{redundantAtom}
An atom $\phi$ of an atomized model $M$ can be eliminated without altering the model if and only if $\phi$ is redundant.
\end{theorem}
\begin{proof}
Let $Th^{+}(M)$ be the set of all positive duples satisfied by the regular elements of $M$ (all elements are regular except the atoms). Since positive duples do not become negative when atoms are eliminated, taking out an atom $\phi$ from a model $M$ produces a model $N$ of $Th^{+}(M)$. Therefore, when removing an atom we only need to worry about negative duples that may become positive. 
To prove that a redundant atom can be eliminated let $a$ and $b$ be a pair of regular elements and $(a \not\leq b)$ a negative duple satisfied by $M$ and discriminated by an atom $\phi < c \leq a$ where $c$ is some constant. If $\phi$ is redundant there is an atom $\eta < c$ in $M$ such that $\phi$ is larger than $\eta$. Suppose $\eta < b$. There is a constant $e$ such that $\eta < e \leq b$. Because $\phi$ is larger than $\eta$ then $\phi < e \leq b$ should also hold contradicting the assumption that $\phi \in {dis_M}(a, b)$. We have proved that $N \models (\eta \not< b)$. Therefore, any negative duple of $M$ is also satisfied by $N$. If $N$ models the same positive and negative duples than $M$ then the subalgebras of $M$ and $N$ spawned by the regular elements are isomorphic. \\
Conversely, assume atom $\phi$ can be eliminated without altering $M$. The pinning term $T_{\phi}$ is the idempotent summation of all the constants in the set $C \setminus U^{c}(\phi)$.  For each constant $c$ such that $\phi < c$ we have $\phi \in {dis_M}(c, T_{\phi})$. That $\phi$ discriminates duple $(c, T_{\phi})$ follows from Theorem \ref{atomIndependentFromTheRestTheorem}, the fact that $c$ and $T_{\phi}$ are terms and $F_C(\emptyset) \models (\phi \not< T_{\phi})$. If $\phi$ can be eliminated without altering $M$, for each constant $c$ in $U^{c}(\phi)$ there should be some atom $\eta_c < c$ with $\eta_c \not= \phi$ in $M$ discriminating the duple $(c, T_{\phi})$ which implies $T_{\phi} < T_{\eta_c}$ or, equivalently, that $\phi$ is larger than $\eta_c$. Hence, for each constant such $\phi < c$ there is an $\eta_c \in M$ such that $\phi$ is larger than $\eta_c$ which proves $\phi$ is redundant.
\end{proof}	
\bigskip

\begin{theorem} \label{uniquedupleTheorem}
Let $M$ be a model and $\phi$ a non-redundant atom of $M$ such that there is at least one constant in $C$ that is not in the upper constant segment of $\phi$. There is at least one pinning duple that is discriminated by $\phi$ and only by $\phi$. 
\end{theorem}
\begin{proof}
As in the proof of Theorem \ref{redundantAtom}, for each constant $c \in U^{c}(\phi)$ we have $\phi \in {dis_M}(c, T_{\phi})$ where $T_{\phi}$ is the pinning term of $\phi$. This is a consequence of Theorem \ref{atomIndependentFromTheRestTheorem} and $F_C(\emptyset) \models (\phi \not< T_{\phi})$. If, for a pinning duple $(c, T_{\phi})$ there is another atom $\varphi$ of $M$ that discriminates $(c, T_{\phi})$ then $\phi$ is larger than $\varphi$ and, if the same is true for every pinning duple, then $\phi$ is redundant with $M$, which is against our assumptions. Therefore, there should be at least one constant $c$ such that $(c, T_{\phi})$ is discriminated only by $\phi$.
\end{proof}	
\bigskip

\begin{theorem} \label{unionProperties}
Let $x$ be a regular element of a model M and let $\phi$, $\psi$ and $\eta$ be atoms of M. The union of atoms has the properties:
\begin{enumerate}[label=\roman*), leftmargin=3cm]
\item $\phi \bigtriangledown \phi = \phi$, 
\item $\bigtriangledown \text{ is commutative and associative}$, 
\item $\phi < x  \Rightarrow  (\phi \bigtriangledown \psi < x)$, 
\item $(\phi \bigtriangledown \psi < t)  \Leftrightarrow (\phi < x) \vee  (\psi < x)$,
\item $(\phi \bigtriangledown \psi < x) \wedge (\phi \not< x)  \Rightarrow (\psi < x)$.
\item $\phi \text{ is larger or equal to } \eta \text{ if and only if }  \phi = \phi \bigtriangledown \eta.$
\end{enumerate}
\end{theorem}
\begin{proof}
An atom is determined by the constants in its upper segment, therefore atom $\phi \bigtriangledown \psi$ is fully defined by $U^c(\phi \bigtriangledown \psi) = U^c(\phi) \cup U^c(\psi)$ and then (i) and (ii) follow from the idempotence, commutativity and associativity of the union of sets.  Choose a term $t \in F_C(\emptyset)$ such $x = \nu(t)$ with $\nu$ the natural homomorphism of $F_C(\emptyset)$ onto $M$. From Theorem \ref{atomicSegmentFromTermTheorem} we know
\[
 L^a_M(\nu(t))  = \{ \phi \in M: {\bf{C}}(t) \cap U^c(\phi) \not= \emptyset \},
\] 
which shows how to calculate the lower atomic segment of an element represented with any term by using the component constants of the term. We can use this duple to prove the other statements. $\phi < x$ implies that exists $c \in {\bf{C}}(t)$ such that $x = \nu(t)$ and $\phi < c$, hence, $c \in U^c(\phi \bigtriangledown \psi)\,$ so $\,\phi \bigtriangledown \psi  < c \leq x$ and (iii) follows. (iv) right to left says that there is a constant $c \in {\bf{C}}(t)$ and a term $t$ such that $x = \nu(t)$ and $(\phi < c) \vee (\psi < c)$ which implies $\phi \bigtriangledown \psi  < c \leq x$. To prove (iv) left to right, write the right side as $\exists t \exists c (x = \nu(t) \, \wedge \, c \in {\bf{C}}(t))$ such that $\phi \bigtriangledown \psi < c$, which implies $c \in U^c(\phi) \cup U^c(\psi)$ and then $(\phi < c) \vee (\psi < c)$ that, together with $c \leq x$, yields $(\phi < x) \vee (\psi < x)$. (v) can be proved in the same way then the others and is left to the reader.
(vi) $\phi$ is larger than, or equal to, atom $\eta$ if and only if for every constant $c$, $(\eta < c) \Rightarrow (\phi < c)$ or, in other words, ${U^{c}}(\eta) \subseteq {U^{c}}(\phi)$. It follows that $U^c(\phi \bigtriangledown \eta) = U^c(\phi) \cup U^c(\eta) = U^c(\phi)$. Hence, $\phi$ is larger than or equal to atom $\eta$ if and only if $\phi = \phi \bigtriangledown \eta$.  
\end{proof}
\bigskip

\begin{theorem} \label{redundantIsUnionOfAtomsTheorem}
Let $M$ be an atomized model over a finite set $C$ of constants. Let $\phi$ be an atom that may or may not be in the atomization of $M$. \\
i) $\phi$ is redundant with $M$ if and only if it is a union $\phi = \bigtriangledown_{i} \eta_{i}$ of atoms of $M$ such that $\forall i (\phi \not= \eta_{i})$. \\
ii) $\phi$ is redundant with $M$ if and only if it is a union of two or more non-redundant atoms of $M$.
\end{theorem}
\begin{proof}
If $\phi$ is redundant, for each constant such that $\phi < c$ there is an atom $\eta_i$ of $M$ such that $\phi$ is larger than $\eta_i$ and $\eta_i < c$. Theorem \ref{unionProperties} assures us that if $\phi$ is larger than $\eta_i$ then $\phi = \phi \bigtriangledown \eta_i$ and, since for each constant $c \in {U^{c}}(\phi)$ there is some $\eta_i$ such that $\eta_i < c$ then ${U^{c}}(\phi) = \cup_{i}{U^{c}}(\eta_{i})$. It follows $\phi = \bigtriangledown_{i} \eta_{i}$. Conversely if $\phi$ is a union of atoms $\phi = \bigtriangledown_{i} \eta_{i}$ then for each constant such $\phi < c$ there is some atom $\eta_{i}$ that contains $c$ in its upper constant segment with $\phi$ larger than $\eta_{i}$, hence, $\phi$ is redundant. This proves (i). \\
Since $C$ is finite then $\vert {U^{c}}(\eta_i) \vert < \vert {U^{c}}(\phi) \vert$. If any of the atoms $\eta_i$ is redundant with $M$ it can be further expressed as unions of atoms of $M$ with ever smaller upper constant segments until reaching non-redundant atoms of $M$. Because $\bigtriangledown$ is associative there is at least one decomposition of $\phi$ as a union of non-redundant atoms of $M$. 
\end{proof}
\bigskip

\begin{theorem} \label{uniqueAtomization}
i) Two atomizations of the same model have the same non-redundant atoms. \\
ii) Any model has a unique atomization without redundant atoms. 
\end{theorem}
\begin{proof}
Let $A$ and $B$ be two atomizations of a model M without redundant atoms. Choose an atom $\phi$ of $B$ and consider the model $A + \{\phi\}$ spawned by $\phi$ and the atoms of $A$.  It is clear that $A + \{\phi\}$ spawns the same model as $A$ otherwise there is a positive duple of $A$ discriminated only by $\phi$ and, hence, negative in $B$ contradicting that $A$ and $B$ are atomizations of the same model. From Theorem \ref{redundantAtom} either $\phi$ is an atom of $A$ or $\phi$ is redundant with $A$. Assume $\phi$ is redundant with $A$. There is a set $E_{\phi}$ of atoms of $A$ such that $\phi$ is a union of the atoms in $E_{\phi}$ (see Theorem \ref{redundantIsUnionOfAtomsTheorem}). Choose an atom $\eta$ in $E_{\phi}$ and consider the model $B + \{\eta\}$. The same reasoning applies so we should get that either $\eta$ is in $B$ or is redundant with atoms of $B$. We can substitute $\eta$ in $E_{\phi}$ with the atoms that make $\eta$ redundant in $B$ to form a set $E'_{\phi}$. In this way we can replace every atom of $E_{\phi}$ with atoms of $B$ such that $\phi$ is a union of the atoms in $E'_{\phi}$ which implies that $\phi$ is redundant in $B$, against our assumptions. Therefore, $\phi$ cannot be redundant with $A$ and then $\phi$ should be an atom of $A$, which proves proposition (i) and also proposition (ii) because $A$ and $B$ should be identical.
 \end{proof}	
\bigskip

\begin{theorem} \label{unionWithFreer}
Let $A$ and $B$ be two atomized models, with $A$ freer or as free as $B$.\\
 i) The model $A + B$ spawned by the atoms of $A$ and the atoms of $B$ is the same as the model spawned by $A$ alone. \\
 ii) The atoms of $B$ are in $A$ or are redundant with the atoms of $A$.  \\
 iii) $A$ is freer than $B$ if and only if the atoms of $B$ are atoms of $A$ or unions of atoms of $A$.
\end{theorem}
\begin{proof} (i) Since $A$ is freer or as free as $B$ all the negative duples of $B$ are also negative in $A$. This means that a duple discriminated by an atom of $B$ is also discriminated by some atom of $A$. In addition, each positive duple of $A$ is also a positive duple of $B$ and also a positive duple of $A + B$. A duple of $A + B$ is  positive if and only if it is positive in $A$, and is negative if and only if it is negative in $A$. Therefore, the models $A + B$ and $A$ are equal. Part (ii) follows directly from Theorem \ref{uniqueAtomization} and the fact that $A + B$ spawns the same model as $A$. \\
(iii) Assume $A$ is freer than $B$. Consider the model $A + B$ spawned by the atoms of $A$ and the atoms of $B$. Proposition (i) tells us that if $A$ is freer than $B$ the model $A + B$ is equal to $A$. Theorem \ref{redundantAtom} assures us that each atom $\phi$ of $B$ is an atom of $A$ or is redundant with the atoms of $A$. Either $\phi$ is an atom of $A$ or for each constant $c$ such that $\phi < c$ there is at least one atom $\eta < c$ in $A$ such that $\phi$ is larger than $\eta$, i.e.\ the set  ${U^{c}}(\phi)$ contains the set ${U^{c}}(\eta)$. In fact, as there is an $\eta$ for each constant $c \in {U^{c}}(\phi)$, if the set $\{\eta_{i}:i=1,...,n\}$ makes $\phi$ redundant in $N$, then ${U^{c}}(\phi) = \cup_{i}{U^{c}}(\eta_{i})$ and then $\phi$ is the union $\phi = \bigtriangledown_{i} \eta_{i}$.  \\
Assume now that the atoms of $B$ are atoms of $A$ or unions of atoms of $A$. All the atoms of $B$ are redundant with $A$ and then, Theorem \ref{redundantAtom} assures that any duple $r$ discriminated by an atom of $B$ is discriminated by at least one atom of $A$ so, $B \models r^{-}$ implies $A \models r^{-}$ and $A$ is freer or equal to $B$.
\end{proof}	
\bigskip

\begin{theorem} \label{fullCrossingIsFreestTheorem}
Let $M$ be an atomized model with or without redundant atoms and $r$ a duple so $M \models r^{-}$. The full-crossing of $r$ in $M$ is the freest model $F_C(Th^{+}(M) \cup\, r^{+})$.  
\end{theorem}
\begin{proof}
Let $H \subseteq M$ be the discriminant of $r=(r_L,r_R)$, i.e.\ the set of atoms $\varphi$ such that $\varphi < r_L \,\wedge\, \varphi \not< r_R$. Let $B \subseteq M$ be the set of atoms of $r_R$, i.e.\ the atoms $\varphi$ such that $\varphi < r_R$. The full-crossing of $(r_L,r_R)$ in model $M$ is the model $K = (M \setminus H) \cup (H\bigtriangledown B)$ where $H\bigtriangledown B \equiv \{\lambda \bigtriangledown \rho : (\lambda \in H) \wedge (\rho \in B) \}$ is the set of all pairwise unions of an atom of $H$ and an atom of $B$. 

Using Theorem \ref{unionProperties} (iii) it follows that the atoms $\lambda \bigtriangledown \rho \in H\bigtriangledown B \subseteq K$ introduced by the full-crossing operation satisfy $\lambda \bigtriangledown \rho < r_R$ because $\rho < r_R$. Since the atoms in the discriminant $\lambda \in H = dis_M(r)$ are no longer present in $K$ and the atoms introduced by the full-crossing are all in the lower segment of $r_R$ then $K \models r^{+}$.

It is immediate from the definition of the order relationship, $<$, in atomized models (the axiom $\forall a \forall b \, (  a \leq b \,  \Leftrightarrow  \neg \exists \phi ( (\phi < a)  \wedge   (\phi \not< b))$) that the elimination of atoms from a model cannot cause any positive duple to become negative. Hence, the atoms eliminated by the full-crossing operation cannot switch positive duples into negative. We have to prove that the atoms introduced by the full-crossing do not switch positive duples into negative duples either.  Assume $M \models s^{+}$ for some duple $s=(s_L,s_R)$.  Suppose that $s_L$ acquires one of the new atoms $\lambda \bigtriangledown \rho$ in its lower segment. From Theorem \ref{unionProperties} (iv) follows that either $\lambda < s_L$ or  $\rho< s_L$ holds in $M$. Because $s_L \leq s_R$ in $M$ then $M \models (\lambda < s_R) \vee (\rho < s_R$). By Theorem \ref{unionProperties} (iii) we get $K \models (\lambda \bigtriangledown \rho < s_R)$. Therefore, the atoms of the form $\lambda \bigtriangledown \rho$ cannot switch a positive duple $s^{+}$ into $s^{-}$.  

So far we know that $K \models r^{+}$ and that $(M \models s^{+}) \Rightarrow (K \models s^{+})$ so we can write: $K \models Th^{+}(M) \cup r^{+}$. To prove that $K$ is the freest model of $Th^{+}(M) \cup r^{+}$ we have to show that the full-crossing does not switch negative duples into positive either unless they are logical consequences of $Th^{+}(M) \cup r^{+}$.

Consider a duple $s=(s_L,s_R)$ such that $K \models s^{+}$ and $M \models s^{-}$. The crossing of $r$ has switched $s$ from negative to positive. For this to occur a necessary condition is that the discriminant of $s$ should disappear from $K$ as a result of the crossing, in other words, $dis_M(s) \subseteq H = dis_M(r)$, so the atoms of $dis_M(s)$ should all be atoms of $r_L$. This implies  $s_L \leq s_R \odot r_L$ holds in $M$. 

Since $M \models s^{-}$ there is at least one $\lambda^{*} \in dis_M(s)$. From Theorem \ref{unionProperties} (iii), all the atoms of the form $\{  \lambda^{*} \bigtriangledown \rho :  \rho \in B \}$ are in $s_L$. If $K \models s^{+}$ then all the atoms $\lambda^{*} \bigtriangledown \rho$ should also be atoms of $s_R$ in model $K$. Using Theorem \ref{unionProperties} (v), $\lambda^{*} \not< s_R$ and $\lambda^{*} \bigtriangledown \rho < s_R$ implies $\rho < s_R$ for each $\rho \in B$. Since $B$ is the entire lower atomic segment of $r_R$ in $M$ it immediately follows that $M \models (r_R \leq s_R)$.  Putting both conditions together:
\[
M \models  ((s_L \leq s_R \odot r_L) \,\wedge\, (r_R \leq s_R)).
\]
Since
\[
(s_L \leq s_R \odot r_L) \,\wedge\, (r_R \leq s_R)  \,\wedge\, (r_L \leq  r_R) \Rightarrow s_L \leq  s_R
\]
$s^{+}$ is a logical consequence of the theory $Th^{+}(M) \cup r^{+}$ so any model of $Th^{+}(M) \cup r^{+}$ must satisfy $s^{+}$.  We have proved that any duple $s^{+}$ satisfied by $K$ is satisfied by any model of $Th^{+}(M) \cup r^{+}$ and, because $K \models Th^{+}(M) \cup r^{+}$ then $K = F_C(Th^{+}(M) \cup r^{+})$.    
\end{proof}
\bigskip

\begin{theorem} \label{fullCrossingIsCommutative}
The full-crossing operation is commutative up to redundant atoms, i.e.\ the resulting model of full-crossing each duple in a set of duples $R$ is independent of the order chosen. 
\end{theorem}
\begin{proof}
Theorem \ref{fullCrossingIsFreestTheorem} shows that the result of full-crossing the duples in a set $R$ in a model $M$ is the freest model $F_C(Th^{+}(M) \cup R^{+})$, so the result is the same independently of the order of crossing. Because the result is the same model, from Theorem \ref{redundantAtom}, the resulting atomizations for different orders of crossing can only differ in redundant atoms. 
\end{proof}
\bigskip

\begin{theorem} \label{freestModelTheorem}
The freest model $F_C(\emptyset)$ over a set $C$ of constants has $\vert C \vert$ non-redundant atoms, each with a single constant in its upper segment. 
\end{theorem}
\begin{proof}
It is a well-known result of Universal Algebra that the freest model is the term algebra, i.e.\ the algebra spawned by the terms modulus the rules of the algebra, in this case the commutative, associative and idempotent laws. For terms $s$ and $t$, the term algebra satisfies that $F_C(\emptyset) \models s \leq t$ if and only if the component constants satisfy ${\bf{C}}(s) \subseteq {\bf{C}}(t)$. Let $M$ be the atomized model obtained by $\vert C \vert$ different atoms, each with one constant in its upper segment. From the axiom of atomized models $\,\phi < a \odot b \, \Leftrightarrow  (\phi < a) \vee (\phi < b)\,$ applied to the component constants of $s$ and $t$ follows that for $M$ (and for any atomized model), ${\bf{C}}(s) \subseteq {\bf{C}}(t)$ implies $s \leq t$. Conversely, assume $M \models s \leq t$. Each atom $\phi$ in the lower atomic segment of $s$ should be in the lower segment of some component constant $c \in {\bf{C}}(s)$ (c.f.\ Theorem \ref{atomicSegmentFromTermTheorem} (ii)) and, since the atoms of $M$ have only one constant in their upper segments then $\phi < t$ can occur only if $c$ is also a component constant of $t$. Since each constant of $C$ has its own atom, every component constant of $s$ should be a component constant of $t$ otherwise there is an atom that discriminates $(s, t)$ against our assumption. This proves that $M$ models the exact same duples as $F_C(\emptyset)$ and both models are the same. Since each atom $\phi$ of $M$ has an upper segment ${U^{c}}(\phi)$ with a single constant $\phi$ is non-redundant. We have finished the proof. \\
We can easily show that $M$ is freer than any atomized model. Let $N$ be any atomized model and let $r$ be any duple such that $N \models r^{-}$. Each atom $\eta$ of $N$ is the union of the atoms of $M$ corresponding to each constant in the set ${U^{c}}(\eta)$ and this proves that $\eta$ is redundant with $M$. Since all the atoms of $N$ are redundant with $M$, using Theorem $\ref{unionWithFreer}$ (iii) we conclude that $M$ is freer or as free as $N$ and, therefore, $M \models r^{-}$ and $M$ is the freest model. From Theorem \ref{redundantAtom} any atomization of $F_C(\emptyset)$ contains the $\vert C \vert$ atoms of $M$ and only can differ from $M$ in a set of redundant atoms. 
\end{proof}
\bigskip

\begin{theorem} \label{atomizationExistsTheorem}
Any model $M$ with a finite set $C$ of constants can be atomized. 
\end{theorem}
\begin{proof}
Full-crossing provides a simple constructive proof for semilattices. Since the theory $Th^{+}(M)$ of any model $M$ over a finite set of constants is a finite a set, $M$ can be atomized by starting with the set of $\vert C\vert$ atoms that provides an atomization for the freest model $F_C(\emptyset)$ (see Theorem \ref{freestModelTheorem}), each atom contained in a single constant, and then by performing a finite sequence of full-crossing operations for each duple in $Th^{+}(M)$. As a result we obtain an atomization of model $M$ which follows from Theorem \ref{fullCrossingIsFreestTheorem} and $M = F_C(Th^{+}(M))$. 
\end{proof}
\bigskip

\begin{theorem} \label{pinningStructureTheorem}
Let $M$ be an atomized model over a set $C$ of constants. \\
i) If $M$ satisfies a negative duple $r^{-}$ then there is a pinning duple $p=(c, T_\phi)$ for some constant $c \in U^{c}(\phi)$ and the pinning term $T_\phi$ such that $p^{-} \Rightarrow r^{-}$. \\
ii) Let $PR(M) \subset Th_0^{-}(M)$ be the set of pinning duples of $M$. $PR(M)$ implies $Th_0^{-}(M)$. \\
iii) Let $\alpha$ and $\beta$ be two different atoms of a model $M$ then $M \models (T_\alpha \not= T_\beta)$.
\end{theorem}
\begin{proof}
Let $r=(r_L, r_R)$. If $M \models r^{-}$ there is an atom $\phi$ of $M$ that discriminates $r^{-}$. It follows that there is a constant $c$ such that $\phi < c$ and a sentence $\Delta = (c \leq r_L) \,\wedge\,( r_R \leq T_{\phi})$ that is true not only in $M$, but also in the freest model $F_C(\emptyset)$ and in every model with constants in $C$. Suppose a model $N$ satisfies $r^{+}$; then, $N \models \Delta \,\wedge\ r^{+}$ and it follows that $c \leq r_L \leq  r_R \leq T_{\phi}$ which implies $p^{+}$. In other words $p^{-} \Rightarrow r^{-}$. This proves proposition (i) from which (ii) follows directly. \\
(iii) Assume $M \models (T_\alpha = T_\beta)$. Since $\alpha \not< T_\alpha$ then we also must have $\alpha \not< T_\beta$, which implies ${U^{c}}(\alpha) \subseteq {U^{c}}(\beta)$. Using the symmetrical,  $\beta \not< T_\beta$ we get  $U^{c}(\beta) \subseteq U^{c}(\alpha)$. Hence, $\alpha$ and $\beta$ are equal against our hypothesis.
\end{proof}
\bigskip

\begin{theorem} \label{implicationTheorem}
Let  $\Gamma$ be a set of first order sentences, $R$ a set of positive and negative duples so $\Gamma \Rightarrow R$, and $M$ a model of $\Gamma$. Any atom of $M$ is an atom of $F(R^{+})$ or is redundant with atoms of  $F(R^{+})$. 
\end{theorem}
\begin{proof}
Since $\Gamma$ implies $R$ any model $M$ of $\Gamma$ also satisfies $R^{+}$. Therefore, $F(R^{+})$ is freer than $M$ or equal to $M$. From Theorem \ref{unionWithFreer}, each atom of $M$ is an atom of $F(R^{+})$ or is redundant with atoms of $F(R^{+})$.  
 \end{proof}	
\bigskip

\begin{theorem} \label{compositionTheorem}
Let ${R^{+}} \subset {Th_0^{+}(M)}$ be a subset of the positive theory of some atomized model M and let $F({R^{+}})$ be the freest model of ${R^{+}}$. The atoms of M are atoms of $F({R^{+}})$ or unions of atoms of $F({R^{+}})$.
\end{theorem}
\begin{proof}
Follows from Theorem \ref{unionWithFreer} (iii) and the fact that $F({R^{+}})$ is freer that $M$.
\end{proof}	
\bigskip

\subsection{Additional results} \label{appendix:additional_ressults}

In this section we present a few new results on atomized semilattices necessary to support this paper.

\begin{theorem} \label{SparseCrossing:segregationTheorem}
Let $R = R^{+}  \cup R^-$ where $R^{+}$ is a set of atomic sentences of the form $(a \leq b)$ and $R^{-}$ a set of negated atomic sentences of the form $(a \not\leq b)$. Let $r$, $s$ and $t$ be duples;\\
i) if $R \Rightarrow r^{+}$, then either $R \cup\, \{r^{+}\}$ is inconsistent or $R^{+} \Rightarrow r^{+}$, \\
ii) if $R \Rightarrow s^{-}$ there is at least one duple $t \in R^{-}$ such that $R^{+} \cup\, \{t^{-}\} \Rightarrow s^{-}$,\\
where the notation $r^+ = (r_L \leq r_R)$ and  $r^- = (r_L \not\leq r_R)$ has been used. 
\end{theorem}
\begin{proof}
(i) Let $M$ and $N$ be two atomized semilattices over the same set of constants $C$. Considering the discriminant atoms present in the model $M + N$ it follows that $M + N \models Th_0^{-}(M) \cup Th_0^{-}(N)$  and $M + N \models Th_0^{+}(M) \cap Th_0^{+}(N)$, where $Th_0^{+}(M)$ and $Th_0^{-}(M)$ are the sets of every positive and negative duples over $C$, respectively, satisfied by $M$. \\
Assume $R \cup\, \{r^{+}\}$ is consistent; there is a model $M \models R^{+} \cup\, R^{-} \cup\, \{r^{+}\}$. Suppose $\,R^{+} \not\Rightarrow r^{+}$, there is a model $N \models R^{+} \cup\, \{r^{-}\}$. Since every semilattice can be atomized, without loss of generality we can assume $M$ and $N$ are atomized over $C$, where $C$ is the set of constants  mentioned in $R$.  Let the model $M + N$ be the model spawned by the union of the atoms of $M$ and the atoms of $N$. It follows $M + N \models R^{+} \cup\, R^{-} \cup\, \{r^{-}\}$ contradicting $R \Rightarrow r^{+}$. \\
(ii) Assume that for every duple $t \in R^{-}$, $\,R^{+} \cup\, \{t^{-}\} \not\Rightarrow s^{-}$. This means that there are models $N_t$ satisfying $\forall t(t \in R^{-})\exists N_t (N_t \models \,R^{+} \cup\, \{t^{-}\} \cup\, \{s^{+}\})$. Consider the model $Q$ spawned by the atoms of all the models $N_t$. We have $Q \models \,R^{+} \cup\, R^{-} \cup\, \{s^{+}\}$ contradicting $R^{+} \,\cup\, R^{-} \Rightarrow s^{-}$.
\end{proof}
\bigskip

\begin{theorem} \label{atomGenesisTheorem}
Let $M$ be an atomized model and $r=(r_L,r_R)$ a duple so $M \models r^{-}$. \\
i) Each non-redundant atom of $M$ is non-redundant with $F(Th^{+}(M) \cup r^{+})$. \\
ii) If $K$ is an atomization of $F(Th^{+}(M) \cup r^{+})$ without redundant atoms, any atom $\phi$ of $K$ is either a non-redundant atom of $M$ or the union of two non-redundant atoms of $M$, $\lambda$ and $\rho$, such that $U^{c}(\phi) = U^{c}(\lambda) \cup U^{c}(\rho)$ and $\lambda<r_L$, $\,\lambda \not<r_R$ and $\, \rho<r_R$. 
\end{theorem}
\begin{proof}
Theorem \ref{fullCrossingIsFreestTheorem} says that a model $Q$ for $F(Th^{+}(M) \cup r^{+})$ can be obtained from $M$ by full-crossing of $r^{+}$. Suppose that a non-redundant atom $\phi \in M$ becomes redundant with a set of atoms of $Q$. Since $\phi$ was not redundant in $M$ then the set $\Gamma$ of atoms of $Q$ that makes $\phi$ redundant should contain at least one atom $\lambda \bigtriangledown \rho$ introduced by full-crossing (see the proof of Theorem \ref{fullCrossingIsFreestTheorem}), but if this is the case, then $\phi$ is redundant with the set $\Gamma \setminus \{\lambda \bigtriangledown \rho\} \cup \{\lambda\} \cup \{\rho\}$. If $\Gamma$ contains more than one new atom of $Q$ we can replace it by its components $\lambda$ and $\rho$ until forming a set of atoms of $M$ that make $\phi$ redundant contradicting that $\phi$ is non-redundant in $M$. The situation does not change if we add redundant atoms to $Q$, therefore $\phi$ is not redundant in any atomization of $Q$.    

(ii) follows from Theorem \ref{fullCrossingIsFreestTheorem} and the uniqueness of the model without redundant atoms. The atomization of $F(Th^{+}(M) \cup r^{+})$ without redundant atoms, $K$, should be the subset of non-redundant atoms of the model obtained by full-crossing of $r^{+}$ in an atomization of $M$ without redundant atoms, so every atom of $K$ is a non redundant atom of $M$ or is in the set $(M \setminus H) \cup (H\bigtriangledown B)$ where $H$ and $B$ are defined in the proof of Theorem \ref{fullCrossingIsFreestTheorem}, so any new atom of $K$ is a union of two non-redundant atoms of $M$.
\end{proof}
\bigskip

\begin{theorem} \label{atomsInDiscriminantTheorem}
Let $N$ be an atomized model, $\lambda$ an atom of $N$ and $r=(r_L,r_R)$ a duple. Let $M = \square_r N$. Either $\lambda \not\in dis_N(r)$, and then $\lambda \in M$, or $\lambda \in dis_N(r)$ and then $\lambda \not\in M$ and there is at least one and possibly various atoms in $M$ strictly wider than  $\lambda$.
\end{theorem}
\begin{proof}
The Theorem follows from Theorem \ref{atomGenesisTheorem} and from the full-crossing mechanism given in the proof of Theorem \ref{fullCrossingIsFreestTheorem}. Full-crossing transforms $N$ into a model $M$ leaving unaltered the atoms of $N$ except when the atom $\lambda \in dis_N(r)$, in this case $\lambda$ is replaced by at least one, usually multiple, completions in $M$. These completions $\phi$ are atoms which upper constant segment $\bf{U}^{c}(\phi) = \bf{U}^{c}(\lambda) \cup \bf{U}^{c}(\rho)$ for some atom $\rho \in N$ (see Theorem \ref{atomGenesisTheorem}). Since $M \models r^{+}$ then $dis_M(r) = \emptyset$, and since $\lambda$ discriminates $r$, then $\lambda$ cannot be in $M$. It follows that every completion $\phi$ must be different from $\lambda$, i.e.\ it should be strictly larger than $\lambda$.
\end{proof}
\bigskip

\begin{theorem} \label{inwarsOutwardSequenceTheorem}
Let $N_0$ be an atomized model over a set of constants $C$, $R^{+}$ a set of positive duples and $r_1, r_2,..., r_n$ an ordering of $R^{+}$.  Assume $R^{+}$ is enforced in $N_0$ by successive full-crossings in the order given and that intermediate models $N_1, N_2,...,N_n$ where  $N_{i-1} \models r_i^{-}$,$\,N_i \models r_i^{+}$  and  $N_n = F(Th^{+}(N_0) \cup R^{+})$ are obtained.  

\begin{enumerate}[label=\roman*), leftmargin=0.8cm]
    \item For each atom $\phi \in N_n$ there is at least one ``inward chain'' of atoms $\lambda_0, \lambda_1 ,..., \lambda_n$ such that $\lambda_0 \in N_0$, $\,\phi = \lambda_n$ and $\lambda_i \in N_i$ either $\lambda_i = \lambda_{i+1}$ or $\lambda_{i+1}$ is strictly larger than $\lambda_{i}$. 
    
    \item For each atom $\lambda \in N_0$ there is at least one ``outward chain'' of atoms $\lambda_0, \lambda_1 ,..., \lambda_n$ such that $\lambda_0 = \lambda$ and $\lambda_i \in N_i$ with either $\lambda_i = \lambda_{i+1}$ or $\lambda_{i+1}$ is strictly larger than $\lambda_{i}$. 
    
    \item Along any inward or outward sequence of atoms the number of times $\lambda_i \not= \lambda_{i+1}$ is at most $\vert U^{c}(\lambda_n) \vert - 1< \vert C \vert$. 
    
    \item Along any inward or outward chain an atom $\lambda_{i} \in dis_{N_i}(r_{i+1})$ if and only if $\lambda_i \not= \lambda_{i+1}$. 
    
    \item Along any inward or outward chain an atom $\lambda_{i}$ in the chain causes model $N_i$ to produce a false negative  when tested with duple $r_{i+1}^{+}$ at most $\vert U^{c}(\lambda_n) \vert - \vert U^{c}(\lambda_0) \vert< \vert C \vert$ times. 
    
    \item Along any inward or outward chain, for each step $i$ there is an atom $\rho_i \in N_{i} \cap N_{i+1}$ such that $\,\lambda_{i+1} = \lambda_i \bigtriangledown \rho_i$.
\end{enumerate}
\end{theorem}
\begin{proof}
Let $N_{i+1}$ be a model obtained as a result of full-crossing of $r_{i+1}$ in $N_i$. From Theorem \ref{atomGenesisTheorem}, or directly from the definition of full-crossing, it follows that an atom $\lambda_{i+1} \in N_{i+1}$ is either an atom of $N_i$ or has an upper constant segment that is the union of the upper constant segments of two atoms of $N_i$, say $\lambda_i$ and $\rho_i$, such that $\rho_i \in N_{i +1}$, $\,\lambda_i \not\in N_{i+1}$ and $\lambda_{i} \in dis_{N_i}(r_{i+1})$. Write $\lambda_{i+1} = \lambda_i \bigtriangledown \rho_i$. This suffices to establish that for each atom of $\lambda_{i+1} \in N_{i+1}$ there is at least one precursor atom in $\lambda_i \in N_i$, perhaps equal to $\lambda_{i+1}$. The precursor atom may not be unique because the same $\lambda_{i+1}$ may be formed by union of two different $\lambda_i, \rho_i$ pairs. This proves the existence of an inward chain, i.e.\ (i). 

To prove (ii) we can use the same reasoning. Given a $\lambda_i \in N_i$ the outward chain continues with $\lambda_{i+1} = \lambda_i$ when $\lambda_{i} \not\in dis_{N_i}(r_{i+1})$ and in case $\lambda_{i} \in dis_{N_i}(r_{i+1})$ then we can choose any of the resulting atoms from the union of $\lambda_{i}$ with any other atom $\rho_i$. This produces at least one, and typically many, possible outward chains.

Number (iii) follows trivially from the fact that any atom $\lambda_{i+1} = \lambda_i \bigtriangledown \rho_i$ formed by full-crossing is strictly larger than $\lambda_{i} \in dis_{N_i}(r_{i+1})$ and can grow at most $\vert U^{c}(\lambda_n) \vert - \vert U^{c}(\lambda_0) \vert$ times. 

Number (iv) is the result of Theorem \ref{atomsInDiscriminantTheorem} applied to successive atoms in any chain.
Number (v) is the same as (iv) but points out that an atom $\lambda_{i}$ causes model $N_{i}$ to produce a false negative for duple $r_{i+1}^{+}$, i.e.\ $N_{i} \models r_{i+1}^{-}$ only when $\lambda_{i} \in dis_{N_i}(r_{i+1})$. When an atom produces a false negative for $r_{i+1}^{+}$ its upper constant segment acquires at least one new constant in the crossing of $r_{i+1}$. Atoms along any chain can fail at most $\vert U^{c}(\lambda_n) \vert - \vert U^{c}(\lambda_0) \vert$ duples from the set $r_1, r_2,..., r_n$, irrespective of how large is $n$.  

 (vi) is a consequence of $\rho_i \in N_{i}$ that remains unaltered in $N_{i+1}$ and the recipe given above for the construction of $\lambda_{i+1}$. At the stages in the chain when $\lambda_i = \lambda_{i+1}$ we can use $\rho_i = \lambda_i$ since  $\lambda_{i+1} = \lambda_{i} = \lambda_i \bigtriangledown \lambda_i$.  
\end{proof}
\bigskip

Theorem \ref{inwarsOutwardSequenceTheorem} proves the existence of inward and outwards chains of atoms $\lambda_0, \lambda_1 ,..., \lambda_n$. Each chain ends in an atom of a model $N_n$ obtained by crossing a set of positive duples $R^{+}$ in an order $r_1, r_2,..., r_n$ producing a sequence of intermediate models $N_0,N_1, N_2,...,N_n$. Each atom along the chain belong to the corresponding intermediate model: $\lambda_i \in N_i$.

\newpage

\section{Discovery of rules in data}  \label{suppSection:discoveryOfRulesInData}

Let $P$ be a set of atomic and negated atomic sentences and $Q$ the set of its logical consequences, excluding $P$ itself. For example, $Q$ may represent data and $P$ the underlying rules implicit in the data that are not directly observable.  

In the Supplementary Section \ref{suppSection:learningTheCausalTheory} we prove that the freest model $F_C(Q^+)$ and the freest model $F_C(P^+)$ are atomized by the following disjoint sets of non-redundant atoms: \[
Nr(F_C(consequences)) = Nr(F_C(Q^+)) = \Phi \cup \Pi, 
\]\[
Nr(F_C(rules)) = Nr(F_C(P^+)) = \Phi \cup \Omega.
\]
The disjoint sets of atoms $\Phi, \,\Pi$ and $\Omega$ classify the non-redundant atoms involved into three separated classes. Since $P^+$ implies $Q^+$, the model $F_C(Q^+)$ is freer than $F_C(P^+)$ and, since $Q^+ \cap P^+ = \emptyset$, some or all the sentences of $P^+$ are negative in $F_C(Q^+)$. The atoms in the set $\Pi$ are precisely the ones discriminating duples of $P^+$ in $F_C(Q^+)$ and in Theorem \ref{causalTheorem} we prove that they have a cardinal bounded by that of $P^+$: \[
\vert \Pi \vert \leq \vert P^+\vert.
\]
The atoms in $\Pi$ encode that $P^+$ is not in $Q^+$ and are usually very wide (with very large upper constant segments) and, hence, barely discriminative. The atoms in the set $\Omega$ are non-redundant in $F_C(P^+)$ but redundant in $F_C(Q^+)$. Each atom $\omega \in \Omega$ is a union of at least one atom $\pi \in \Pi$ and other atoms of $F_C(Q^+)$, so they are very similar but even wider than those of $\Pi$. Furthermore, every duple discriminated by $\omega$ is discriminated by an atom of $\Pi$. If an atom in $\Omega$ discriminates a duple $s$, i.e.\ if it causes $F_C(P^+) \models s^-$ then $s^+$ implies (in any model) at least one duple of $P^+$. 

The set $\Phi$ contains the useful atoms: a duple that corresponds to a negative axiom, is negative in models $F_C(P^+)$ and $F_C(Q^+)$ because it is discriminated by at least one atom in the set $\Phi$. These atoms are common to both models.  

In a realistic scenario, only a small subset of the consequences, $Q^{\prime} \subset Q$, is observed (the training data). Because the upper constant segment of the atoms cannot grow without limit as a result of a sequence of crossings, the atoms mature rapidly into atoms of $\Phi$. Some intuition about this can be obtained from Supplementary Section \ref{section:genSubsetsFreestModel}, but in essence, when an atom appears in the discriminant of a duple in $Q^{\prime +}$, it is replaced by one or more wider atoms. Any atom $\phi$ in the model is the result of an ``inward chain'' of length $g(\phi)$ of increasingly wider atoms (with increasing upper constant segments). Since $g(\phi)$ cannot exceed $\vert U^c(\phi) \vert$, atoms mature after a few effective crossings. Specifically, effective crossings for $\phi$ are only the $g(\phi)$ ones that influenced the inward chain, i.e., those in which the atom in the inward chain appears in a discriminant and is, therefore, replaced. For example, most atoms of the MNIST dataset have about 10 to 12 constants, which means that they matured in fewer than 10 to 12 effective crossings. Although most crossing have no effect on the inward chain of $\phi$, if $\phi$ is not yet matured after $\vert Q^{\prime +} \vert$ crossings then $\phi$ has an expected probability of causing a false negative lower than $g(\phi) / \vert Q^{\prime +} \vert$. 

The non-redundant atoms of $F_C(Q^{\prime +})$ comprise a subset $\Phi^{\prime} \subset \Phi$ and a set of immature atoms that would develop into atoms of $\Phi$ if more data were processed. Other non-redundant atoms of $F_C(Q^{\prime +})$ include atoms of $\Pi$ and also wider atoms that, if more data is processed, will end up in the discriminant of sentences of $Q^+$ and removed from the model without leaving an offspring, as they produce only redundant atoms when they are crossed: \[
Nr(F_C(Q^{\prime +})) = \Phi^{\prime} \cup Immature  \cup Incompatible \cup \Pi.
\]
If Sparse Crossing is used, a model $N$ is obtained that depends upon the set $Q^{\prime -}$ and is atomized by a subset of the atoms of $F_C(Q^{\prime +})$: \[
Nr(N) \approx \Phi^{\prime \prime} \cup Immature^{\prime}  \cup Incompatible^{\prime} \subseteq Nr(F_C(Q^{\prime +})),
\]
where prime indicates subset, i.e.\ $\Phi^{\prime \prime} \subseteq \Phi^{\prime}$, etc. An atom of $\Pi$ can only be discovered in the improbable case that it discriminates a duple of $Q^{\prime -}$. Incompatible atoms are typically large and not very discriminative, so their presence have a relative low impact in the behavior of $N$. The atoms of $\Phi^{\prime \prime}$ and also the immature atoms provide an approximation to the rules $P$. We should take into account:
\begin{itemize}
\item Some atoms of $\Phi$ are rapidly discovered even when the set of examples seen, $Q^{\prime +}$, is a tiny fraction of the set $Q^+$.  Since every subset of $\Phi$ produces a model that satisfies the entire set $Q^+$ this leads to generalization.
\item There is a very large number of subsets of $\Phi$ that lead to generalizing models; in Supplementary Section \ref{section:genSubsetsFreestModel} and in \cite{Maroto} we discussed that small subsets of the atoms of a model can approximate the model with high accuracies even in the presence of noise. Atoms of $\Phi$ not present in a model cause false positives, but since most duples discriminated by an atom of $\Phi$ are also discriminated by other atoms of $\Phi$, small subsets of atoms can produce models that perform with low probability of false positive (but not 0).  
\item Different rules or aspects of the rules may be encoded by distinct atoms of $\Phi$. Some of these atoms may not be discovered by Sparse Crossing, depending on the counterexamples provided in $Q^{\prime -}$. For instance, consider a problem of identifying patterns in binary (black and white) images. There exist atoms in $\Phi$ that encode precisely this constraint, i.e., that each pixel location contains either a black or white pixel. Since all examples, including counterexamples, adhere to this rule, these atoms are not discriminative and will not be discoverable by Sparse Crossing.
\item The entire model $F_C(Q^{\prime +})$ is not generalizing, as it assigns negative every duple not implied by $Q^{\prime +}$, which causes false negatives. This is caused by the cumulative effect of all the incompatible and immature atoms, which cause false negatives. However, the number of atoms in $F_C(Q^{\prime +})$ is typically immense, and this loss of generalization is not observed in practice. Models with very different number of atoms tend to produce similarly generalizing models (see figure \ref{figure:mnistpfppfn}). 
\item There are other nuances that should be taken into account. For example, observable sentences in $Q^{\prime}$ typically belong to a localized subset of the space of possible duples. In the case of binary (black and white) images, all observable terms have as many constants as pixels. This can have the effect of introducing implicit rules in the data, see Theorem \ref{causalTheoremOnW}.
\item There are usually many symmetries given by permutations of the constants that interchange atoms of $\Phi$ leaving this set invariant. This property can potentially lead to reasonable guesses of $\Phi$ from a subset of it.
\end{itemize}
\subsection{Learning the causal theory} \label{suppSection:learningTheCausalTheory}

\begin{theorem} \label{causalTheorem}
Let $P$ and $Q$ be two sets of signed duples over a set of constants $C$ such that $Q$ is the set of logical consequences of $P$ minus the set $P$: 
\begin{enumerate}[label=\roman*), leftmargin=1cm]
\item The atomization without redundant atoms of $F_C(Q^{+})$ consists of non-redundant atoms of $F_C(P^{+})$ and an additional set of atoms of cardinal equal or lower than that of $\vert P^{+} \vert$. 
\item Each non-redundant atom of $F_C(Q^{+})$ is either a non-redundant atom of $F_C(P^{+})$ or is the unique discriminant of one duple of $\,P^{+}$. 
\item If a non-redundant atom of $F_C(Q^{+})$ that is the unique discriminant of a duple $\,s^+ \in P^{+}$ discriminates a duple $r=(r_L, r_R)$ then the component constants of the term $r_R$ are component constants of $s_R$. 
\item For each non-redundant atom $\omega$ of $F_C(P^{+})$ that is not a non-redundant atom of $F_C(Q^{+})$ there is a pining duple $u = (c, T_{\omega})$ of $\omega$ such that there is a $r^{+} \in P$ for which, $u^{+} \Rightarrow r^{+}$. 
\end{enumerate}
\end{theorem}
\begin{proof}
By definition $P \Rightarrow Q$ and $Q \cap P = \emptyset$. Since $P \Rightarrow Q$ all the positive duples in $Q$ are positive duples of any model of $P$ and, hence, $F_C(Q^{+})$ is freer than $F_C(P^{+})$ or equal to $F_C(P^{+})$. Choose a non-redundant atom of $F_C(Q^{+})$ and consider the model $K$ spawned by the chosen atom plus the atoms of $F_C(P^{+})$. \\ 
Suppose $K = F_C(P^{+})$. If $K$ is equal to $F_C(P^{+})$ then the chosen atom, say $\phi$, is, according to Theorem \ref{redundantAtom}, either redundant with $F_C(P^{+})$ or is a non-redundant atom of $F_C(P^{+})$. We are going to show that $\phi$ cannot be redundant with $F_C(P^{+})$. Suppose it is redundant; $\phi$ is a union of atoms of $F_C(P^{+})$. Since $F_C(Q^{+})$ is freer than $F_C(P^{+})$ Theorem \ref{compositionTheorem} tells us that each atom of $F_C(P^{+})$ is either an atom of $F_C(Q^{+})$ or a union of atoms of $F_C(Q^{+})$. It immediately follows that $\phi$ is a union of atoms of $F_C(Q^{+})$, which contradicts the assumption that we chose a non-redundant atom of $F_C(Q^{+})$. Therefore, if $K$ is equal to $F_C(P^{+})$ then $\phi$ is a non-redundant atom of $F_C(P^{+})$. \\
Suppose now that $K \not= F_C(P^{+})$. In this case we are going to refer to the chosen atom as $\pi$. Since $K$ has all the atoms of $F_C(P^{+})$, both models can differ only if there is at least one duple $r$ discriminated by $\pi$ such that $F_C(P^{+}) \models r^{+}$ and $F_C(K) \models r^{-}$. Since $\pi$ is an atom of $F_C(Q^{+})$ then $F_C(Q^{+}) \models r^{-}$. Because $r^{+}$ is modeled by the freest model of $P$ then $r^{+}$ is a logical consequence of $P^{+}$ and, by definition of $Q$, either $r^{+} \in Q^{+}$ or $r^{+} \in P^{+}$. Since $F_C(Q^{+}) \models r^{-}$, it follows that $r^{+}$ is a duple of $P^{+}$. \\
So far we have shown that each non-redundant atom of $F_C(Q^{+})$ is either like $\phi$ or like $\pi$. Atoms like $\phi$ are also non-redundant atoms of $F_C(P^{+})$ while atoms like $\pi$ are never atoms of $F_C(P^{+})$ and discriminate each at least one duple of $P^{+}$. \\
As in the proof of Theorem \ref{pinningStructureTheorem}, for each $r$ discriminated by $\pi$ there is a constant $c \in U^{c}(\pi)$ such that the pinning duple $s = (c, T_\pi)$ satisfies $s^{-} \Rightarrow r^{-}$. Each atom $\pi$ discriminates at least one duple of $P^{+}$, so let $r$ be one of such duples. The equivalent $r^{+} \Rightarrow s^{+}$ implies that $s^{+}$ is either a consequence of $P^{+}$ or in $P^{+}$. Since $F_C(Q^{+}) \models s^{-}$ we can discard that $s^{+}$ is a consequence of $P^{+}$ and, therefore, $s^{+}$ should be a duple of $P^{+}$. Since the pinning duples of different atoms are all different (see Theorem \ref{pinningStructureTheorem} (iii)) then each atom $\pi$ can be mapped to a different duple of $P^{+}$ and, hence, $F_C(Q^{+})$ cannot have more atoms like $\pi$ than duples are in $P^{+}$.  We have proven (i) and (ii). \\
iii) We have just seen that a non-redundant atom of $F_C(Q^{+})$ that is the unique discriminant of a duple $s^+ \in P^{+}$ is an atom $\pi_0$ of type $\pi$ and $s = (c, T_{\pi_0})$ is a pinning duple of $\pi_0$. If $\pi_0$ discriminates a duple $r = (r_L, r_R)$ then $\pi_0 < r_L$ and $r_R < T_{\pi_0} = s_R$ is true in any model over $C$, i.e.\ the component constants of the term $r_R$ are component constants of $s_R$. \\
iv) Let $\omega$ be a non-redundant atom of $F_C(P^{+})$ that is not a non-redundant atom of $F_C(Q^{+})$. According to Theorem \ref{uniquedupleTheorem}, if $\omega$ is non-redundant in $F_C(P^{+})$ there is at least one pinning duple $u = (c, T_{\omega})$ discriminated only by $\omega$ in $F_C(P^{+})$. Because $F_C(Q^{+})$ is freer than $F_C(P^{+})$ then $F_C(Q^{+}) \models u^{-}$. The non-redundant atoms of $F_C(Q^{+})$ discriminating $u^{-}$ cannot be of type $\phi$ because these atoms are also atoms of $F_C(P^{+})$ and that would contradict that $u$ is discriminated only by $\omega$ in $F_C(P^{+})$. Therefore, the atoms of $F_C(Q^{+})$ that discriminate $u$ are of type $\pi$. \\
Since $F_C(Q^{+})$ is freer than $F_C(P^{+})$ then $\omega$ is a union of atoms of $F_C(Q^{+})$.  Among these atoms at least one is of type $\pi$, otherwise $\omega$ would be redundant in $F_C(P^{+})$ and it is not. In fact, at least one of these atoms of type $\pi$ should discriminate $u$ simply because if atom $\alpha \bigtriangledown \beta$ discriminates $u$ then either $\alpha$ or $\beta$ discriminates $u$. Let $\pi_0 \in F_C(Q^{+})$ be one of such atoms of type $\pi$ that discriminates $u$, i.e.\ $\pi_0 < c$ and $\pi_0 \not< T_{\omega}$. Since $\omega$ is wider than $\pi_0$, $T_{\omega} < T_{\pi_0}$. It follows $(c \not\leq T_{\pi_0}) \Rightarrow u^{-}$. We showed above that the pinning duples of the atoms of type $\pi$ are elements of the set $P^{+}$. Therefore, $u \in P^{+}$ and $u^{+} \Rightarrow (c \leq T_{\pi_0})$. 
\end{proof}	
\bigskip

Often the duples that can be observed lie in a region, $W \subseteq F_C(\emptyset) \times F_C(\emptyset)$, of the space of duples. For example, if we are learning images, the duples observed may all have a term at the right-hand side that correspond to an image; a term formed as a summation of as many constants as pixels. Terms with other number of component constants may not be accessible for training. The following theorem clarifies this situation:  

\bigskip

\begin{theorem} \label{causalTheoremOnW}
Let $P$ and $Q$ be two sets of signed duples over a set of constants $C$ such that $Q$ is the set of logical consequences of $P$ minus the set $P$. Let $W$ by the set of ``observable duples'': 
\begin{enumerate}[label=\roman*), leftmargin=1cm]
\item The non-redundant atoms of $F_C(P^{+})$ that are non-redundant atoms of $F_C(Q^+ \cap W)$  are also non-redundant atoms of $F_C(Q^+)$.   
\item The non-redundant atoms of $F_C(P^{+})$ that are non-redundant atoms of $F_C(Q^+)$ are atoms of $F_C(Q^+ \cap W)$, redundant or non-redundant. 
\item A non-redundant atom $\eta$ of $F_C(P^{+})$ that is a non-redundant atom of $F_C(Q^+)$ but a redundant atom of $F_C(Q^+ \cap W)$ exists only if there is a non-redundant atom $\delta$ of $F_C(Q^+ \cap W)$ that is narrower that $\eta$ and discriminated by at least one duple of $Q^+ \cap \overline{W}$ and by no duple of $Q^+ \cap W$.
\end{enumerate}
\end{theorem}
\begin{proof}
Consider the sequence of models: \[
F_C(Q^+ \cap W) \rightarrow F_C(Q^+) \rightarrow  F_C(P^+) = F_C(P^+ \cup (Q^+ \cap \overline{W} ) ) 
\]  were each $\rightarrow$ represent a constant-preserving homomorphism. The last equality follows easily as $P^+ \Rightarrow Q^+$ and then the freest model of $P^+ \cup (Q^+ \cap \overline{W} )$ is the same of that of  $P^+$ alone.  It is clear that the models in the sequence are increasingly less free as we move to the right;  $F_C(Q^+ \cap W)$ is obviously freer or equal to $F_C(Q^+)$ and $F_C(Q^+)$ is freer than $F_C(P^+)$ because every duple in $Q^+$ is a logical consequence of $P^+$.  \\
Let's refer to the pair of models $F_C(Q^+ \cap W)$ and $F_C(P^+ \cup (Q^+ \cap \overline{W} ))$ as the ``external pair'' and to the pair of models $F_C(Q^+)$ and $F_C(P^+)$ as the ``internal pair''.  \\
For the internal pair, the first model, i.e.\ $F_C(Q^+)$ contains the logical consequences (but not the premises) of the second, $F_C(P^+)$, with an empty intersection $Q^+ \cap P^+ = \emptyset$. Just like the internal pair, the external pair behaves in the same manner; $Q^+ \cap W$ is the set of logical consequences of $P^+ \cup (Q^+ \cap \overline{W} )$ and both sets have empty intersection. \\
In the proof of Theorem \ref{causalTheorem} we showed that the non-redundant atoms of the (freest model of) the consequences where, either non-redundant atoms of the causes, we called them type $\phi$ atoms, or they were external to the model of the causes and we refereed to that atoms as of type $\pi$.  Also, we used $\omega$ for the non-redundant atoms of the causes that were redundant atoms of the consequences; let's refer to them as atoms of type  $\omega$. \\
i) The atoms of type $\phi$ for the external pair are non-redundant atoms of $F_C(P^+ \cup (Q^+ \cap \overline{W} )$, which is equal to $F_C(P^+)$. Since a non-redundant atom of the first model of the sequence that is also an atom of the last model of the sequence is a non-redundant atom of the four models of the sequence, it is clear that the atoms of type $\phi$ for the external pair are also atoms of type $\phi$ for the internal pair. The atoms of type $\phi$ we are discovering using duples of $W$ are atoms of type $\phi$ that we would discover using all the duples. \\ 
ii) It follows from the fact that $F_C(Q^+ \cap W)$ is freer than  $F_C(Q^+)$. Notice that the type $\phi$ atoms of the internal pair that are not type $\phi$ atoms of the external pair are redundant atoms of $F_C(Q^+ \cap W)$ that become non-redundant when crossing some duple of $Q^+ \cap \overline{W}$, since they are non-redundant for $F_C(Q^+)$.  \\ 
iii) It is natural to wonder which type $\phi$ atoms of the internal pair are not discovered when using only duples of $W$. Since $\eta$ is redundant in $F_C(Q^+ \cap W)$ it must be equal to a union of non-redundant atoms of $F_C(Q^+ \cap W)$. As $\eta$ becomes non-redundant when crossing the duples in $Q^+ \cap \overline{W}$ there must be some non-redundant atom $\delta$ of $F_C(Q^+ \cap W)$ narrower than $\eta$, i.e.\ $U_C(\delta) \subseteq U_C(\eta)$, that is discriminated away by some duple of $Q^+ \cap \overline{W}$. Notice that $\delta$ is of type $\pi$ for the external pair.  \\
\end{proof}	
\bigskip

\noindent\textbf{Example.} Consider the task of, given a black and white $n \times n$ image, distinguishing whether or not it has at least one black vertical bar. We may have an underlying theory  \[ 
P^+ = \{ (v \leq bar_1), (v \leq bar_2),...., (v \leq bar_n) \}  \]
with $2n^2 + 1$ constants, one for each color of the $n^2$ pixels plus one constant for the vertical bar $v$. If the observable duples belong to well-formed images, then $W$ consists of the set of all the terms with $n^2$ constants, one for each pixel. The constant are $C =  \{ v, b_{ij}, w_{ij} : i, j \in  \{1,..., n \}  \}$ and the atomization of the freest model is \[ 
 F_C(\emptyset) = \{ \phi_{v}, \phi_{b_{ij}}, \phi_{w_{ij}} : i, j \in  \{1,..., n \}  \} 
\]
where the $b_{ij}$ are the constants for the black color pixels and $w_{ij}$ the constants for the white color pixels. We then have: \[ 
 F_C(P^+) = \{ \phi_{b_{ij}}, \phi_{w_{ij}},  \phi_{v, b_{\sigma(1)1}, b_{\sigma(2)2},...,b_{\sigma(n)n} }  : i, j \in  \{1,...,n \}  \,\,\,\, \sigma : \{1,...,n \} \rightarrow \{1,...,n \}   \} 
\]
where $\sigma$ runs along the set of all functions from  $\{1,...,n \}$ to $ \{1,...,n \}$. The atomization has $2n^2$ atoms with one constant in its upper constant segment plus $n^n$ atoms with $n$ constants in their upper constant segments, one constant for a black pixel selected for each of the $n$ columns. In total $2n^2 + n^n$ non--redundant atoms.
\[ 
 F_C(Q^+) = \{ \phi_{b_{ij}}, \phi_{w_{ij}},  \phi_{v, b_{\sigma(1)1}, b_{\sigma(2)2},...,b_{\sigma(n)n} }  : i, j \in  \{1,...,n \}  \,\,\,\, \sigma : \{1,...,n \} \rightarrow \{1,...,n \}   \}  \, \cup 
\] \[ 
\cup   \{ \phi_{C - \{b_{1j}, b_{2j},..., b_{nj} \}},   :  j \in  \{1,...,n \}   \} .
\]
Every non-redundant atom of $F_C(P^+)$ is discovered by $F_C(Q^+)$.  The $n$ atoms $\phi_{C - \{b_{1j}, b_{2j},..., b_{nj} \}}$ are type $\pi$ atoms in the proof of Theorem $\ref{causalTheorem}$ (i) and each of them have a duple of $P$ as pining duple.  We will see later that:
\[ 
 F_C(Q^+ \cap W) = \{ \phi_{b_{ij}}, \phi_{w_{ij}},  \phi_{v, b_{\sigma(1)1}, b_{\sigma(2)2},...,b_{\sigma(n)n} }  : i, j \in  \{1,...,n \}  \,\,\,\, \sigma : \{1,...,n \} \rightarrow \{1,...,n \}   \}  \, \cup 
\] \[ 
\cup   \{ \phi_{v, b_{ij}, w_{ij}}  :  i, j \in  \{1,...,n \}   \} .
\]
Again, every non-redundant atom of $F_C(P^+)$ is discovered by $F_C(Q^+ \cap W)$. The $n^2$ atoms $\phi_{v, b_{ij}, w_{ij}}$ actually characterize the set $W$. Notice that the $\pi$ atoms $\phi_{C - \{b_{1j}, b_{2j},..., b_{nj} \}}$ are equal to unions of atoms of the form  $\phi_{v, b_{ij}, w_{ij}}$ and $\phi_{w_{ij}}$, so they are all redundant with $F_C(Q^+ \cap W)$.
From Theorem  \ref{causalTheoremOnW} (iv) we can deduce that there is no non-redundant atom of $F_C(P^{+})$ that is non-redundant in $F_C(Q^+)$ and redundant in $F_C(Q^+ \cap W)$; an atom $\delta$ narrower to $\phi_{b, b_{\sigma(1)1}, b_{\sigma(2)2},...,b_{\sigma(n)n} }$ is necessarily missing some black pixel constant in some of the $n$ columns and it is, therefore, discriminated by some duple $Q^+ \cap W$, i.e.\ there is a duple with a term of a well-formed image that has $\delta$ in its discriminant. Therefore, every type $\phi$ atom of  $F_C(Q^+)$ is discovered by $F_C(Q^+ \cap W)$.
\bigskip

Notice that for any pattern recognition problem of this kind, with a property $v$ shared by a set of binary images, we can say the same as for the vertical bar.  The atoms of $v$ that discriminate images without the property $v$ cannot have the black and white constant for the same pixel. Therefore, there is no non-redundant atom of $F_C(P^{+})$ that is a non-redundant in $F_C(Q^+)$ and redundant in $F_C(Q^+ \cap W)$. Hence, $F_C(Q^+ \cap W)$ discovers every non-redundant atom of $F_C(P^+)$ discovered by $F_C(Q^+)$.

\newpage

\section{Generalizing Subsets of the Freest Model} \label{section:genSubsetsFreestModel}

Consider the freest model $M =F_C(R^{+})$ of an ideal set of positive duples containing every valid sample of a task, for example, a classification problem. Since $R^{+}$ contains every valid sample, $M$ must perform with 0 error. Given an ordering $R^{+} =\{s_1^{+}, s_2^{+}, s_3^{+},...,s_J^{+}\}$ and applying full-crossing, we obtain the sequence of models $F_0, F_1,...,F_j,...,F_J$, with $F_j = F_C(s_1^{+}, s_2^{+},...,s_j^{+})$ for $0 \leq j  \leq J = \vert R^{+} \vert$ with $F_0 = F_C(\emptyset)$ and $F_J = M$, which is independent of the order chosen. We are interested in calculating how well a subset of atoms of the intermediate model $F_j$ approaches $M$.

In  Supplementary Section \ref{suppSection:statistical_model} we show that, in the task of approaching $M$ with a subset of $F_j$, after crossing $j$ positive duples, the expected value for the probability of false negative of a subset of $Z$ atoms of $F_j$ is given by \[
\overline{\operatorname{PFN}}(\phi_1,\dots,\phi_Z)\approx  \sum_{i=1}^{Z} \min \left( \frac{1}{h(\phi_i) + 1},  \frac{g(\phi_i) + 1}{j + 1} \right),
\]
where $g(\phi_i)$ is the length of the inner chain of atom $\phi_i$, i.e.\ the number of atom unions that took the formation of atom $\phi_i$, and $h(\phi_i)$ is the number of positive duples crossed after the final atom union that produced $\phi_i$.  
 The expected $\operatorname{PFN}$ decreases with $j$ so, for a sufficiently large $j$ it becomes as small as desired. Eventually, the expected $\operatorname{PFN}$ is dominated by the length of the tails $h(\phi_i)$;  \[
\overline{\operatorname{PFN}}(\phi_1,\dots,\phi_Z) \approx  \sum_{i=1}^{Z} \frac{1}{h(\phi_i) + 1},
\]
and, since the tails grow linearly with $j$ the collective $\operatorname{PFN}$ still decreases with $j$ and grows linearly with the number of atoms in the set, $Z$.

To find out if our subset is generalizing we must also consider the probability of false positive. A duple that must be negative gives a false positive if every atom in our subset fails to discriminate the duple. In this case  \[
\overline{\operatorname{PFP}}(\phi_1,\dots,\phi_Z) = \prod_{i=1}^{Z} \overline{\operatorname{PFP}}(\phi_i),
\]
which assumes the $\overline{\operatorname{PFP}}(\phi_i)$ of individual atoms are approximately independent of each other. The assumption of independence may lead to a safe overestimation of $\overline{\operatorname{PFN}}$ but to an underestimation of the $\overline{\operatorname{PFP}}$.  However, because every $F_j$ is freer than $M$, there must be atoms in $F_j$ that discriminate every duple that is negative in $M$. With a large enough subset, and even if there are correlations, we must get a collective $\overline{\operatorname{PFP}}$ as small as desired, even equal to 0. 

To minimize the $\overline{\operatorname{PFP}}$ we have to select the atoms for our subset as those having a low and uncorrelated individual $\overline{\operatorname{PFP}}(\phi_i)$, which we can do by computing $\overline{\operatorname{PFP}}(\phi_i)$ using the negative duples of the training set. In addition, we should also make sure to select a subset that, collectively, discriminates all the negative duples of the training set. Since correlated atoms discriminate similar duples, the smaller the resulting subset of atoms the less correlation we must have. 

Since the upper constant segment of an atom is limited by the number of constants, atoms cannot ``grow'' without limit. Every atom present in the hypothetical model $M$ is the result of a finite number of atom unions and after that, the atom ``matures'' and remains unchanged. Notice that an atom is removed from an intermediate model when it causes the first false negative in the sequence $s_1^{+}, s_2^{+}, s_3^{+},...,s_J^{+}$, then the atoms grows, but this cannot happen many times. In fact, atoms obtained in classification problems have just a few constants in its upper constant segment: $5$ to $50$ constants, with most atoms around $12$ to $15$ constants, are typical values observed, so maturity is typically reached after $g(\phi_i) < 15$ crossings. 

We are in an advantageous situation where the $\overline{\operatorname{PFP}}$ decreases geometrically with $Z$ while the $\overline{\operatorname{PFN}}$ increases linearly with $Z$ but decreases inversely proportional to $j$. We then expect subsets of atoms of the freest model $F_j$ with low and uncorrelated individual $\overline{\operatorname{PFP}}$ to do a good job approaching $M$, with an error as small as desired as $j$ increases, which is, indeed, observed experimentally in a wide range of problems.

\subsection{Generalizing Subsets selected by Sparse Crossing} \label{Supplementary Section:section_subsets}

\begin{figure}[hbt!]
\centering\includegraphics[width=1\linewidth]{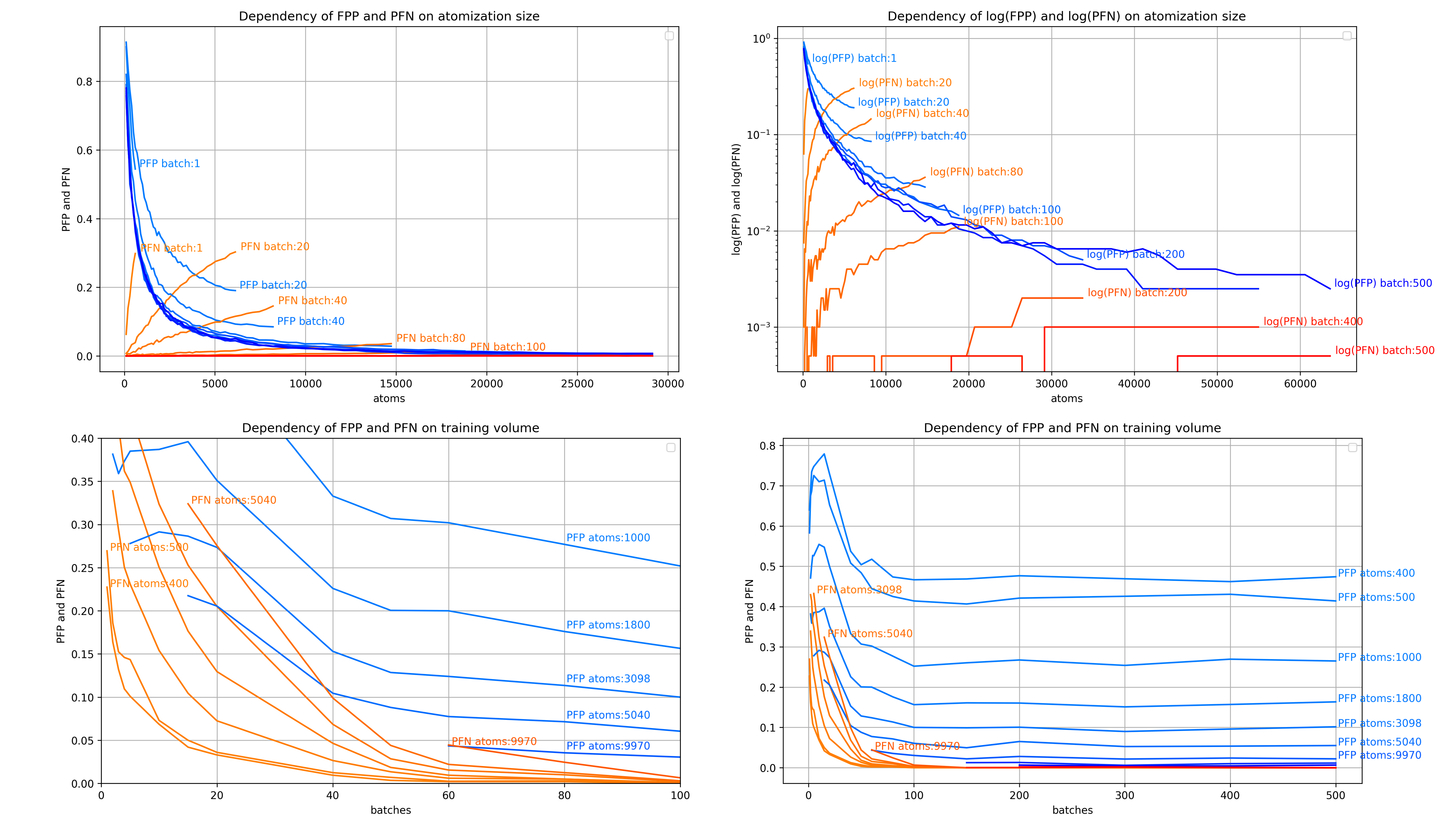}
\caption{\textbf{Evolution of the $\operatorname{PFP}$ and $\operatorname{PFN}$ for the problem of classifying images with or without a complete vertical black bar in an $8\times 8$ grid in the presence of $50\%$ background noise}. Every batch consisted of $1{,}000$ examples and $1{,}000$ counterexample images. As expected, the $\operatorname{PFN}$ (in orange/red) increases linearly with the number of atoms in the model, with a slight sublinear growth due to atom correlations. For the $\operatorname{PFP}$ (in blue), we observe a very fast geometric decrease with the number of atoms, with diminishing rate due to correlations. As training progresses, an strong decrease of both the $\operatorname{PFP}$ and $\operatorname{PFN}$ is observed for models with the same number of atoms. After about 100 batches, most atoms reach maturity; after that, the larger the subset of atoms the better the performance. However,  more training does not result in better performance for sets of atoms with the same size. The models were obtained using Sparse Crossing, which computes a model $M_b$ at batch $b$, that is a subset of the freest model of the data. At different stages of the training, subsets of $M_b$ with good discriminative capacity and low correlation were obtained to calculate the figures. Discriminative, low-correlation subsets were calculated by selecting a set of atoms, as small as possible, that collectively sufficed to discriminate $5{,}000$ counterexamples. To obtain larger subsets multiple subset selections were carried out, with replacement, and using the same set of counterexamples. The resulting subsets were added. The performance of the subsets was measured using a test set of $2{,}000$ examples and $2{,}000$ counterexamples. }
\label{figure:verticalLines}
\end{figure}

\begin{figure}[hbt!]
    \centering \includegraphics[width=0.8\textwidth,height=0.75\textheight,keepaspectratio,bb=0 0 2231 3730]{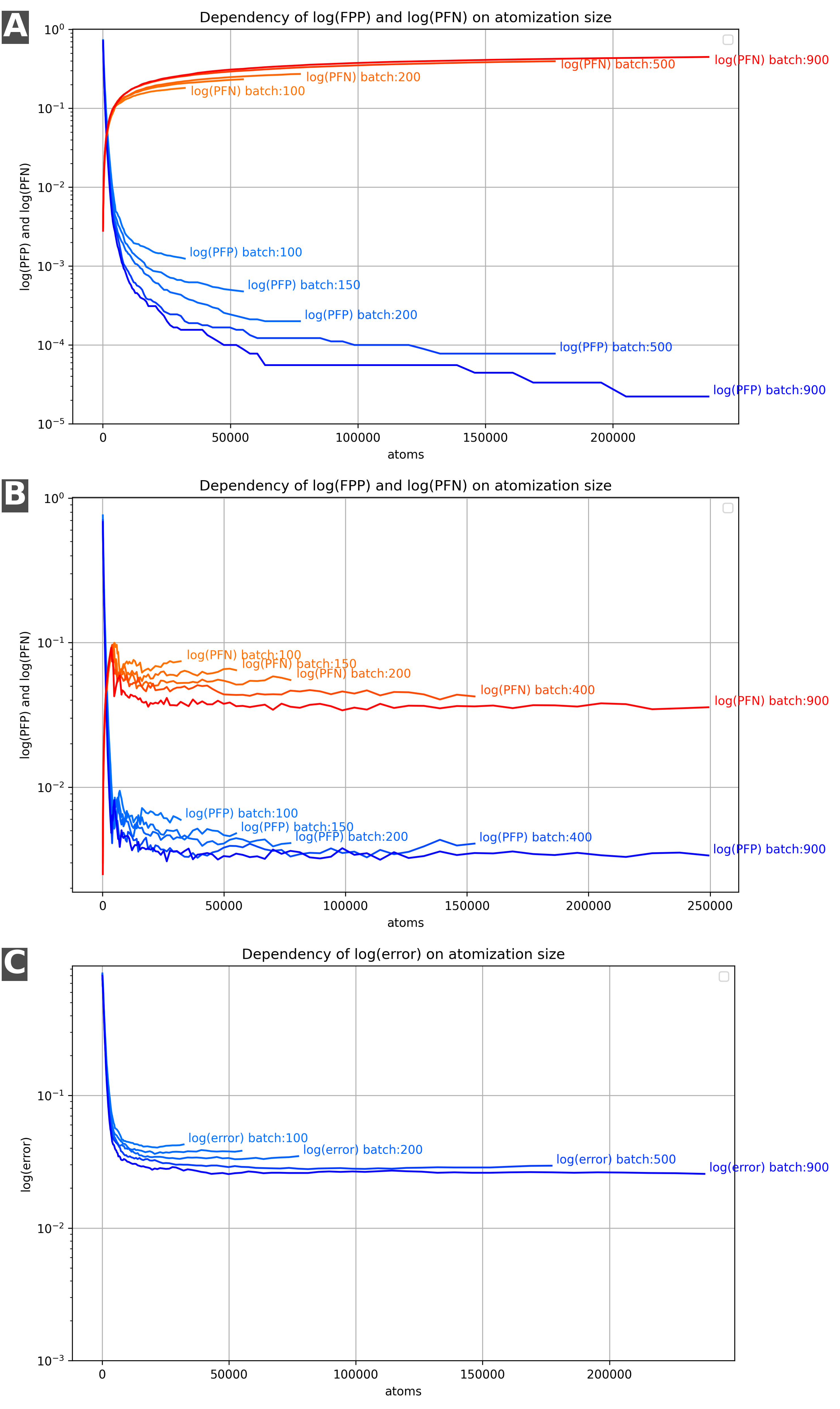}
	\caption{\textbf{Repeated training with the same data for the MNIST handwritten digit classification.} The evolution of the $\operatorname{PFP}$ and $\operatorname{PFN}$ is shown in plot A. In plot B, the $\operatorname{PFP}$ and $\operatorname{PFN}$ calculated from an optimal number of atom misses (beyond which a duple is considered negative) estimated using the validation dataset. Plot C, corresponds to the classification error rate obtained by selecting the class with fewer misses, without using the validation. The figure illustrates that Sparse Crossing produces models that improve with the amount of training even if the same data is presented 490 times, whiteout overfitting, and given constant results across a wide range of atom subset sizes. }
     \label{figure:mnistpfppfn}
\end{figure}

The freest model, $F_C(R^+)$, of a set $R^+$ of atomic sentences is usually too large to compute. However, using Sparse Crossing, it is possible to directly calculate a subset of its atoms without the need of computing the entire freest model first. For that to be possible, a set of negative duples, $R^-$, is also required. In a classification problem, counterexamples can be provided as negative duples.   

The subset of atoms computed by sparse-crossing at batch $i$ is such that it satisfies $R_i = R_i^+ \cup R_i^-$ with just the necessary atoms needed to discriminate $R_i^-$, i.e., at most $\vert R_i^-\vert$ atoms. This results in a small subset of atoms that are good at discriminating the negative duples. Moreover, the less correlation the atoms in the subset have, the fewer atoms are needed to discriminate $R_i^-$, so Sparse Crossing also tends to produce atoms that are relatively uncorrelated.

In addition to the ``master'' model of small cardinality obtained from the sparse-crossing of a batch, a ``union model'' obtained as the union of the master models of all previous batches is also computed (see \textbf{Supplementary Section, Section \ref{suppSection:SparseCrossingInDepth}}). The union model undergoes a sieve in which atoms inconsistent with $\cup^i_{k=0} R^+_k$ are removed resulting in a model that satisfies this set of positive duples. The union model is atomized by a subset of atoms of the freest model of $\cup^i_{k=0} R^+_k$, and is typically much larger than the master model but still very small compared to the size of the freest model. 

The mental experiment of Section \ref{section:genSubsetsFreestModel} can be done in practice using Sparse Crossing. In this experiment, we study how well subsets of the model $F_j = F_C(s_1^{+}, s_2^{+},...,s_j^{+})$ approximate a hypothetical model $M$ of perfect performance. Assume the first $i$ batches are such that $\cup^i_{k=0} R^+_k = \{s_1^{+}, s_2^{+},...,s_j^{+} \}$. To study the evolution of the expected $\operatorname{PFP}$ and $\operatorname{PFN}$ we can compute subsets of atoms of $F_j$ of various sizes by selecting atoms from the union model. We can extract from the union model subsets of any size of discriminative, low-correlation atoms, for which we use the following selection method: we select a small subset of atoms that suffices to discriminate the training set or part of it and, by adding multiple such subsets of smallest cardinality we can extract larger subsets of $F_j$ of any size up to the size of the union model.

Figures \ref{figure:verticalLines} and \ref{figure:mnistpfppfn} show the observed $\overline{\operatorname{PFP}}$ and $\overline{\operatorname{PFN}}$ for the problem of identifying vertical bars in a noisy background and for the problem of handwritten digits, respectively. For the first problem there is a very large supply of examples while for the second, the number of examples is limited to 50K. The calculations of the expected $\operatorname{PFP}$ and $\operatorname{PFN}$ assumes a sequence $s_1^{+}, s_2^{+}, s_3^{+},...,s_j^{+}$ of different positive examples, which is what happens in the first problem. However, for the second problem the sequence of examples contains repeated samples and after a few batches, only contains repeated examples. While with Full Crossing training using repeated examples has no effect, with Sparse Crossing it is useful as it leads to the discovery of more and more discriminative atoms of the freest model as explained in Section \ref{SparseCrossing:smarter} or Supplementary Section \ref{suppSection:SparseCrossingInDepth}. 

The experimental evolution of the $\overline{\operatorname{PFP}}$ and $\overline{\operatorname{PFN}}$ in both problems matches very well the theoretical calculations. However, for MNIST, it does so only up to the batch 50 or so, after which most examples are repeated and our theoretical assumptions no longer hold. 

The experimental result using repeated examples in MNIST is shown in Figure \ref{figure:mnistpfppfn}. A model for MNIST was calculated training for $900$ batches using Sparse Crossing.  The number of examples in a training batch grew linearly from $500$ examples in batch $1$ to two/thirds of the training set at batch $500$. After batch $500$, batch size remained constant until batch $900$.  In Figure \ref{figure:mnistpfppfn}A, we can see that the experimental $\operatorname{PFP}$ behaves similar to that expected for training with non-repeated data, with a very strong decrease with the number of atoms and the amount of training. Since the expected $\operatorname{PFN}$ decreases geometrically with $Z$, as Sparse Crossing discovers more and more atoms of $F_j$ it is not surprising that we see a strong decrease with the size of the atom subset despite training occurring with the same data. For the experimental $\operatorname{PFP}$, we observe, as expected, an increase with the number of atoms that flattens quite rapidly. However, here we observe that the $\operatorname{PFN}$ worsens with the amount of training. Narrower atoms discovered by Sparse Crossing are more effective at correctly discriminating negative duples but they also have a higher risk of causing a false negative. In addition, the theoretical $\overline{\operatorname{PFP}}$ only decreases with the number of different positive duples presented to the algebra, so we expect the $\operatorname{PFP}$ to be affected by the presence of repeated examples. 

In this situation, after around batch 100, we have a $\operatorname{PFN}$ that improves with training and a $\operatorname{PFP}$ that worsens, so we should wonder if the algebra is still learning or not. To answer this question consider Figure \ref{figure:mnistpfppfn}B. Now we allow a few misses before a duple is considered negative instead of just 1 miss, in order to evaluate a subset we used a validation dataset of the same size to determine the optimal number of atom misses needed to declare a duple negative. We can see now two important differences: first, both the $\operatorname{PFP}$ and $\operatorname{PFN}$ improve with the amount of training and, second, results no longer depend upon subset size for subsets of more than 30K atoms. In Figure \ref{figure:mnistpfppfn}C we use a simpler heuristic; to classify a sample we selected the class with fewer misses. This does not require a validation dataset and, since there is a false negative for each false positive, $\operatorname{PFP}$ and $\operatorname{PFN}$ both become equal to the classification error rate, which reached about $2.5\%$. Again we observe the same as before; error decreases with training and is independent of subset size, actually flat, in a wide range of sizes beyond 30K atoms. After about 100 epochs the error in the training set 0, but this does not stop learning, as it continues for hundreds of epochs and, probably, for much longer. Learning continues without overfitting even after the training set have been presented 490 times. There is a clear performance improvement for subsets of the same size as training progresses using the same data. This is indication that Sparse Crossing is able to find increasingly better atoms with more training. The mechanism involves leveraging previously discovered atoms to find narrower atoms (with smaller upper constant segments) and non-redundant atoms of the model $F_j$.

\subsection{Derivation of the Probability of False Negative}
\label{suppSection:statistical_model}

According to Theorem \ref{fullCrossingIsFreestTheorem} we can calculate the freest model $F_C(R^{+})$ using the full-crossing of the duples in $R^{+}$ in any order. Suppose we want the algebra to learn a model $M = F_C(R^{+})$. We can choose an ordering for the set of its positive duples $R^{+} =\{s_1^{+}, s_2^{+}, s_3^{+},...,s_J^{+}\}$ and apply full-crossing. Consider the sequence of models $F_0, F_1,...,F_j,...,F_J$, with $F_j = F_C(s_1^{+}, s_2^{+},...,s_j^{+})$ for $0 \leq j  \leq J = \vert R^{+} \vert$ with $F_0 = F_C(\emptyset)$ and $F_J = M$. We want to derive the probability of false negative, $\omega$, incurred by one or many atoms of model $F_j$ in the task of approaching $M$. 

The more widely used False Positive Ratio (FPR) becomes approximately equal to the probability of false negative (false negatives divided by true positives) when it is small, i.e.\ when the number of false negatives is much smaller than the number of true positives.

For a typical pattern recognition task,  $R^{+}$ would correspond with the set of all valid positive samples that could exist, so $M$ gives no error in the task. The positive duples of the training set would then correspond to the initial samples of an ordering given to $R^{+}$. 

When $s_j$ is presented to $F_{j - 1}$ it produces a false negative if $F_{j - 1} \models s_j^{-}$, which implies that there is one or many atoms of $F_{j - 1}$ in the discriminant $dis_{F_{j - 1}}(s_j^{-})$. In the crossing of duple $s_j$ each atom in the discriminant produces wider atoms (perhaps redundant) and then is removed from the model. This occurs if and only if $s_j$ is a false negative of  $F_{j - 1}$. 

Full-crossing is a process of elimination of atoms and creation of new atoms as unions, $\phi \bigtriangledown \psi$, of two existing atoms.  Suppose an atom $\phi \bigtriangledown \psi$ is formed when crossing duple $s_k$, i.e.\ at ``inception time'' $k$, and assume that this atom has survived until time $j$ with $k < j << J$, i.e.\ assume $\phi \bigtriangledown \psi$ is an atom of $F_j$.  If $\phi \bigtriangledown \psi$ belongs to $F_j$ then $\phi \bigtriangledown \psi$ is compatible with the $j - k$ duples $s_{k +1}^{+}, s_{k + 2}^{+},...,s_j^{+}$ or, in other words, the atom is in the discriminant of none of these $j - k$  duples. Knowing that the atom has been compatible with $j - k$ positive training duples can be used to calculate a Bayesian estimation of the cumulative distribution of the probability of false negative $\omega(\phi \bigtriangledown \psi)$; this calculation is carried out in Section \ref{FNROneAtom}. The cumulative distribution for the probability of false negative of atom $\phi \bigtriangledown \psi$ is: 
\[
P(\omega(\phi \bigtriangledown \psi)  \geq \delta) \,  = \,   (1 - \delta) ^ {j - k(\phi \bigtriangledown \psi) + 1}
\]
where $j$ is the moment at which we measure the probability of false negative, $k(\phi \bigtriangledown \psi)$ is the inception time of the atom and $P(\omega(\phi \bigtriangledown \psi)  \geq \delta)$ is the probability for the probability of false negative of $\phi \bigtriangledown \psi$ to be larger than a value $\delta$.  

Without loss of generality assume that atom $\phi$ is in the discriminant of duple $s_{k(\phi \bigtriangledown \psi)}$ so $\phi \bigtriangledown \psi$ is created from atom $\phi$ during the crossing of $s_{k(\phi \bigtriangledown \psi)}$. The cumulative distribution for atom $\phi$ can be approached as $P(\omega(\phi)  \geq \delta)\,  \approx \,  (1 - \delta) ^ {k(\phi \bigtriangledown \psi) - k(\phi)}$ and then: 
\[
 P(\omega(\phi)  \geq \delta) P(\omega(\phi \bigtriangledown \psi)  \geq \delta) \,  \approx  \,  (1 - \delta) ^ {j - k(\phi) + 1}.
\]
see Section \ref{FNROneAtom} for details. We can apply the same formula along any inward chain $\varphi_0, \varphi_1,...,\varphi_g$ leading to atom $\varphi_g = \phi \bigtriangledown \psi$ that exists at time $j$ (see Theorem \ref{inwarsOutwardSequenceTheorem}), to get the product: 
\[
\Pi_{n = 0}^{g} P(\omega(\varphi_n)  \geq \delta)  \, \approx  \,  (1 - \delta) ^ {j + 1}.
\]
where $k(\varphi_0) = 0$ has been used. As more positive duples are presented to the algebra the atoms become wider and tend to produce fewer false negatives. Assuming the last atom $\varphi_g$ in the chain does not produce false negatives with significantly higher probability that the atoms along the chain, we get:
\[
P(\omega(\varphi_g)  \geq \delta)  \,  \leq  \,  (1 - \delta) ^ {\frac{j + 1}{g + 1}}.
\]
This bound has a straightforward interpretation; atoms in the chain cause false negatives at times $k(\varphi_0),k(\varphi_1),...,k(\varphi_{g})$. If the atoms get increasingly better at producing false negatives, as we progress along the chain the intervals between false negatives $k(\varphi_1) - k(\varphi_0), k(\varphi_2) - k(\varphi_1),...,k(\varphi_{g}) - k(\varphi_{g - 1}), j - k(\varphi_{g})$ would tend to be increasingly larger, so at time $j$ we expect to have $j - k(\varphi_{g-1}) \geq {\frac{j}{g + 1}}$. Since the atoms cannot grow without limit, for a sufficiently large amount of training, i.e.\ a sufficiently large $j$, the value of $g$ becomes much smaller than $j$ and $P(\omega(\varphi_g)  \geq \delta)$ becomes small.  For a large enough $j$, an atom of $F_j$ produces as few false negatives as desired or no false negatives at all if the atom has matured into an atom of $M$. 

It is clear that $g$ can be at most equal to $j$ and usually is much smaller than $j$. Each of the chains corresponds to an atom formed during a crossing and is wider than the previous atom, i.e.,  $U^c(\varphi_n) \subseteq U^c(\varphi_{n +1})$. Since atoms grow at each of the $g$ links, the value of $g$ can be bounded by how much an atom can grow:
\[
g \leq \vert U^{c}(\varphi_g) \vert - \vert U^{c}(\varphi_0) \vert< \vert C \vert.
\]
The number of crossings, $g$, can be as large as $\vert U^{c}(\varphi_g) \vert - \vert U^{c}(\varphi_0) \vert$ but also as small as 0. After a crossing, an atom of the discriminant produces unions with other atoms of any size, so its upper segment $U^{c}(\varphi)$ can grow in a single step by one or many constants.

 In conclusion,  after a sufficiently large number of positive training examples $j$, we can expect an atom of $F_j$ to produce very few or no false negatives. We are assuming that the distribution of the test and training sets are equal and that we do no have positive duples that are false in the training set; that would be the case if we had mislabels, for example. 

In Section \ref{FNROneAtom} we study various approximations to
 $P(\omega(\varphi_g)  \geq \delta)$ for an atom $\varphi_g$ present in $F_j$, and conclude that
\[
P(\omega(\varphi_g)  \geq \delta) = (1 - \delta)^ {a(\phi)}   \,\,\,\,\,\text{with} \, \,\,\,\, a(\phi) = \max\left(h_g + 1, \frac{j + 1}{g + 1}\right)
\]
where $h_g = j - k(\varphi_g)$ is the length of the last interval, approaches the confidence quite well.  For this approximation, the expected value for the probability of false negative is:
\[
\overline{\omega}(\varphi_g) = \min{\left(\frac{g + 1}{j + 2}, \frac{1}{h_g + 2}\right)}.
\]
Since $a(\phi)$ is large, we can approach the expected value of the individual probability of false negative with $\overline{\omega}(\varphi_g) \approx \frac{1}{a(\phi)}$ which has a quite straightforward interpretation.  

In Section \ref{FNRforZatoms} we derive the probability of false negative produced by a set of atoms.  Let $W$ be the collective probability of false negative $\omega(\phi_1,\phi_2,...,\phi_Z)$, for a set of $Z$ atoms each with a probability of false negative given by $\omega(\phi_i)$. Since the model produces a false negative if any of the Z atoms produce a false negative we get:
\[
\omega(\phi_1,....\phi_Z) =  1 -  \prod_{i=1}^{Z} (1 - \omega(\phi_i)).
\]
We are assuming here that atoms produce false negatives in an independent manner, which is false but safe, as it overestimates the probability of false negative.
 
To get a small probability of false negative for $Z$ atoms the individual probability of false negative for each atom should be quite small, so we can do a first order approximation:
\[
\overline{\omega}(\phi_1,....\phi_Z) \approx  \sum_{i=1}^{Z} \overline{\omega}(\phi_i) \approx  \sum_{i=1}^{Z}  \frac{1}{a(\phi_i)}  = \sum_{i=1}^{Z} \min \left( \frac{1}{h(\phi_i) + 1},  \frac{g(\phi_i) + 1}{j + 1} \right).
\]
Since the values of $a(\phi_i)$ are known, we can compute the expected probability of false negative. Notice that this calculation can produce values larger than 1. In this case we expect a probability of false negative equal to 1, but we can still use the formula to measure how much learning is taking place even in the regime when learning cannot be measured by computing the probability of false negative experimentally with a validation dataset.

Given a value $\Delta$ for the probability of false negative of $Z$ atoms, in Section \ref{FNRforZatoms} we calculate $P(\omega(\phi_1,....\phi_Z)  \geq \Delta)$, and show that it is a function that transitions very abruptly from high to low values. This means that there is a moment from which we suddenly have high confidence our collective probability of false negative is better than  $\Delta$.  We give an expression for the transition point from which we can estimate how much training is still needed to achieve with confidence a desired probability of false negative value. 

Although the calculation becomes a bit technical, we can give an intuition of what happens, as follows. If we want the probability of false negative of a model of $Z$ atoms to be bounded by $\Delta$, i.e.\ $\overline{\omega}(\phi_1,....\phi_Z) < \Delta$, it suffices with having each atom $\phi_i$ with an individual probability of false negative of $\omega(\phi_i)  < \frac{\Delta}{Z}$:
\[
P\left(\omega(\phi_i)  \geq \frac{\Delta}{Z}\right)  \,  \leq  \,  \left(1 - \frac{\Delta}{Z} \right) ^ {a(\phi_i)}.
\]
If we want $P(\omega(\phi_i)  \geq \frac{\Delta}{Z})$ to be smaller than $\frac{1}{2}$, it is enough with requiring
\[
\frac{Z}{a(\phi_i) \, \, \Delta} < 1,
\]
that is valid for small values of$\Delta$. i.e.\ when $\Delta  \approx -\ln(1-\Delta)$ is a valid approximation. We also show that the better approximation
\[
\frac{1}{\Delta} \sum_{i=1}^{Z} \frac{1}{a(\phi_i)} < 1
\]
also works.  To derive from here an upper bound for the necessary number of training examples we can use $a(\phi_i)  \leq \frac{j + 1}{g + 1}$ which gives: 
\[
j \geq \frac{\sum_{i=1}^{Z} (g(\phi_i) + 1)}{\Delta},
\]
and, since the values $g(\phi_i)$ are known, we can use it to estimate how much training is left to achieve a probability of false negative better than $\Delta$. 

We can use the fact that $g(\phi_i) + 1$ is always smaller than the number of constants $\vert C \vert$ (each time an atom causes a false negative its upper segment grows by at least one constant), to get the bound:
\[
j \geq \frac{ Z \, \vert C \vert }{\Delta},
\]
that is simple but significantly overestimates the amount of training needed. It implies that any model over $C$ is approached with  better probability of false negative than $\Delta$ by a set of $Z$ atoms obtained after crossing fewer than $\frac{ Z \, \vert C \vert }{\Delta}$ positive duples. The bound works even for $Z = 0$, which gives a probability of false negative equal to 0 at the expense of obtaining a probability of false positive equal to 1.

\subsubsection{Probability of False Negative of a single atom} \label{FNROneAtom}

The false negatives produced by an individual atom can be illustrated with a game of biased coins. We toss $j$ times and get $g$ tails with $g << j$. This is a well known game but with an added difficulty: each time we get tails we have the coin replaced with a new coin that is more biased towards heads than the previous one. If we obtain $g$ tails $g + 1$ coins are played, the coins $\varphi_0, \varphi_1, ...., \varphi_g$. We are interested in determining the probability distribution for the coin $\varphi_g$ at the end of the game. 

We can make a first approximation using the fact that the coin we get at the end of the game, the final coin $\varphi_g$, should be more biased towards heads than a coin from which we expect $g$ tails in $j$ tosses. If $\omega(\varphi_g)$ is the probability for the final coin to produce tails, a lower bound for the confidence, i.e.\ the probability of having $\omega(\varphi_g) \geq \delta$ is:
\[
P(\omega(\varphi_g)  \geq \delta)  \,  \leq  \,  \frac{  {\binom{j}{g}} \int_{\delta}^{1} {  {t} ^g (1 - t) ^{j - g}} dt    }{  {\binom{j}{g}} \int_{0}^{1} {  {t} ^g (1 - t) ^{j - g}} dt } = 1 - I_{\delta}(g + 1, j - g + 1) 
\]
where $I_{\delta}(g + 1, j - g + 1)$ is the regularized incomplete beta function. This expression corresponds to the Bayesian estimation of the conditional probability of having $\omega(\varphi_g) \geq \delta$ when $g$ tails and $j - g$ heads are produced, assuming a flat prior distribution for $\omega(\varphi_g)$. This approximation uses all the information available; however, if the final coin is significantly more biased towards heads than the previous coins it overestimates the probability of obtaining tails (i.e.\ the probability of getting a false negative). 

In the case where the final coin is significantly more biased than the previous coins a simpler approximation can yield better results.  We can count how many times the final coin $\varphi_g$, produces heads. We never see the final coin producing tails; we only see the final coin producing a final subsequence of $h_g$ heads:
\[
P(\omega(\varphi_g)  \geq \delta)  \,  =  \,  (1 - \delta) ^{h_g + 1},
\]
that corresponds with $1 - I_{\delta}(1, h_g + 1)$, the formula above for $h_g$ heads and $0$ tails.  Note that when $h_g = 0$ we get $P(\omega(\varphi_g)  \geq \delta)  \,  =  \,  (1 - \delta)$ which reflects the lack of prior information regarding the distribution of $\omega(\varphi_g)$ we have assumed. This approximation can be good under some circumstances but bad in others, particularly when the final coin has been obtained close to the end of the game and it has been tossed only a few times.

We can improve our calculations by using $P(\omega(\varphi_g)  \geq \delta)  \,  =  \,  (1 - \delta) ^{h_g + 1}$ for the final coin and:
\[
P(\omega(\varphi_n)  \geq \delta)  \,  =  \,  (1 - \delta) ^{h_n + 1} (h_n\delta + \delta + 1),
\]
for the intermediate coins, which corresponds with $1 - I_{\delta}(2, h_n + 1)$, the formula above for $h_n$ heads and $1$ tail at the end.  Multiplying the $g + 1$ confidence values:
\[
\Pi_{n = 0}^{g} P(\omega(\varphi_n)  \geq \delta)  \,  =  \,  (1 - \delta) ^ {h_g + 1} \Pi_{n = 0}^{g - 1}  (1 - \delta) ^ {h_n + 1} (h_n\delta + \delta + 1)  = 
\]
\[
=  \,  (1 - \delta) ^ {1 + g + \Sigma_{n = 0}^{g} h_n } \Pi_{n = 0}^{g - 1} (h_n\delta + \delta + 1) =  \,  (1 - \delta) ^ {j + 1} \Pi_{n = 0}^{g - 1} (h_n\delta + \delta + 1)
\]
where we have used $g + \Sigma_{n = 0}^{g} h_n  = j$, and taking into account that the coins get increasingly more biased:
\[
P(\omega(\varphi_{n + 1})  \geq \delta)  \,  \leq  P(\omega(\varphi_n)  \geq \delta)
\]
which implies $\Pi_{n = 0}^{g} P(\omega(\varphi_n)  \geq \delta) \geq P(\omega(\varphi_g)  \geq \delta)^{g + 1}$, we get an upper bound for $P(\omega(\varphi_g)  \geq \delta)$:
\[
P(\omega(\varphi_g)  \geq \delta) \,  \leq  \,   (1 - \delta) ^ {\frac{j + 1}{g + 1}} \Pi_{n = 0}^{g - 1} (h_n\delta + \delta + 1)^ {\frac{1}{g + 1}}.  
\]

\begin{figure}
    \centering
	\includegraphics[width=0.4\linewidth]{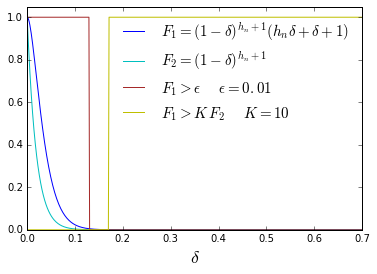}
	\caption{
\textbf{Approximation to the confidence $P(\omega(\varphi_n)  \geq \delta)$.  }  $F_1 = (1 - \delta)^{\frac{j + 1}{g + 1}} \Pi_{n = 0}^{g - 1} (h_n\delta + \delta + 1)^{\frac{1}{g + 1}}$ in dark blue and its approximation $F_2 = (1 - \delta) ^ {h_n + 1}$ in cyan color. The region of interest is located to the left of the red line while the region of excessive underestimation starts at the yellow line. 
}
\label{fig:figApproach}
\end{figure}

Conveniently, we can ignore the multiplying factor $\Pi_{n = 0}^{g - 1} (h_n\delta + \delta + 1)^ {\frac{1}{g + 1}}$ because it is important only for values of $\delta$ we are not interested in. To be more specific, by omitting the multiplicative factor $(h_n\delta + \delta + 1)$  the probability $P(\omega(\varphi_{n})  \geq \delta)$ becomes underestimated by a factor of $K$ when
$(1 - \delta) ^ {h_n + 1} (h_n\delta + \delta + 1) > K  \, (1 - \delta) ^ {h_n + 1}$ and, since we are interested in the region where $\omega(\varphi_{n})$ is most likely higher than $\delta$ and not in the region where we can assure that (i.e.\ $\omega(\varphi_{n}) \geq \delta$) with unnecessarily high confidence, we need a good estimation in the region $(1 - \delta) ^ {h_n + 1} (h_n\delta + \delta + 1)> \epsilon$ for a sufficient but modest confidence value of, say, $\epsilon > 0.005$. It should be clear that with this inequality we are limiting the confidence for the estimation of the probability of false negative and not the actual value $\omega(\varphi_{n})$. The two inequalities can be rewritten as:
\[
(h_n + 1)\delta  >  K  - 1
\]\[
\ln((h_n + 1)\delta + 1) - (h_n + 1) \delta > \ln(\epsilon)
\]
where we have approached $\ln(1 - \delta) \approx -\delta$. Both inequalities hold simultaneously for some domain of $\delta$ if and only if $\ln(K) - K + 1 > \ln(\epsilon)$, an inequality that does not depend upon $h_n$ or $\delta$. Therefore, provided that the inequalities do not hold simultaneously, i.e.\ when:
\[
\ln(K) - K + 1 < \ln(\epsilon),
\]
we can neglect the multiplicative factors $(h_n\delta + \delta + 1)$ and write: 
\[
P(\omega(\varphi_g)  \geq \delta) \,  \leq  \,   (1 - \delta) ^ {\frac{j + 1}{g + 1}}.
\]
For example, for an overestimation of the confidence of, say, at most $K = 10$ and an $\epsilon > 0.005$ our approximation is valid. Figure $\ref{fig:figApproach}$ illustrates the regions where each of the two inequalities hold and show that there is no overlap.

The estimation $F_1 = (1 - \delta)^{\frac{j + 1}{g + 1}} \Pi_{n = 0}^{g - 1} (h_n\delta + \delta + 1)^{\frac{1}{g + 1}}$, its approximation $F_2 = (1 - \delta) ^ {\frac{j + 1}{g + 1}}$ as well as $F_3 = 1 - I_{\delta}(g + 1, j - g + 1)$ transition at the same point $\frac{g + 1}{j + 2}$ given by the rule of succession of Laplace and are all very similar, with $F_2$ having the advantage of the simplicity. The estimation based on the regularized incomplete beta function, $F_3$, transitions very fast and resembles a step function $F_4 = \Theta(\delta - \frac{g + 1}{j + 2})$. They all overestimate the probability of obtaining tails with the coin $\varphi_g$, i.e.\ the probability of false negative. On the other hand, the estimation based in the final sequence of heads $F_5 = (1 - \delta)^{h_g + 1}$ tend to better resemble the actual behavior of $\varphi_g$ provided that $h_g$ is sufficiently large, i.e.\ when $h_g > {\frac{j + 1}{g + 1}}$, which is usually the case. An even better estimation is given by $F_{6}=(1 - \delta)^{\max(h_g + 1, \frac{j + 1}{g + 1})}$, which behaves as $F_2$ when $h_g < {\frac{j + 1}{g + 1}}$ and as $F_5$ when $h_g > {\frac{j + 1}{g + 1}}$. Figure \ref{fig:apprAllLogCombinedAverage} gives the averaged distributions for the different estimations against the numerically obtained cumulative distribution. In practice we can use:
\[
P(\omega(\varphi_g)  \geq \delta) = (1 - \delta)^{\max(h_g + 1, \frac{j + 1}{g + 1})}
\]
and an expected value for the probability of false negative:
\[
\overline{\omega}(\varphi_g) = \min{\left(\frac{g + 1}{j + 2}, \frac{1}{h_g + 2}\right)}.
\]

\begin{figure}
    \centering
	\includegraphics[width=0.8\linewidth]{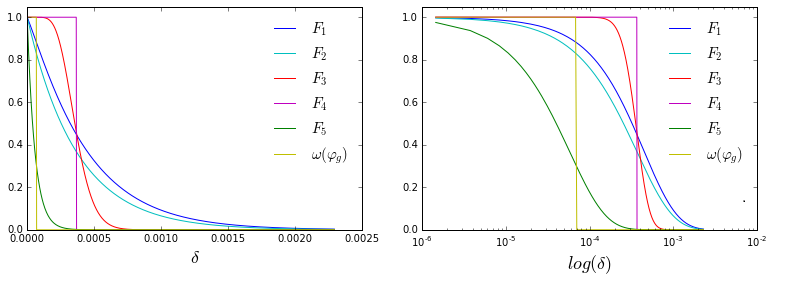}
	\caption{
\textbf{Various estimations of $P(\omega(\varphi_g)  \geq \delta)$} A single coin is tossed. The probability of tails starts at $\omega(\varphi_0) = 0.5$. Each time a tail is obtained the coin is manipulated to increase the probability of tails that becomes smaller by a dividing factor uniformly sampled between 1 and 4. The number of trials is $j=30{,}000$. Cumulative distributions $F_1,...,F_5$ for the different estimations of $\omega(\varphi_g)$ are plot. $F_6$ (not shown) is equal to $F_5$. In yellow, the final value of $\omega(\varphi_g)$ that was $0.00007$. }

\label{fig:apprAllLogCombined}
\end{figure}

\begin{figure}
    \centering
	\includegraphics[width=0.8\linewidth]{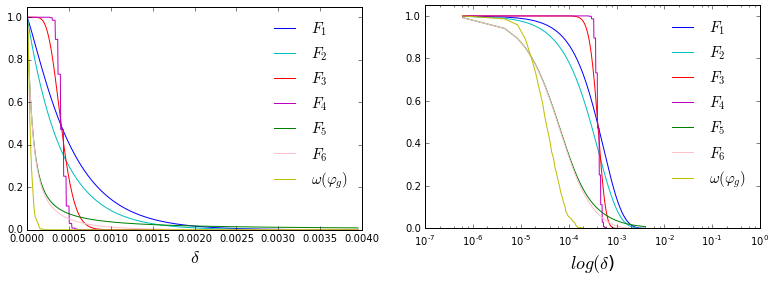}
	\caption{
\textbf{Average estimations of $P(\omega(\varphi_0)  \geq \delta)$ calculated with a simulation. }
The probability of tails starts at $\omega(\varphi_0) = 0.5$. Each time a tail is obtained the probability of tails becomes smaller by a dividing factor uniformly sampled between 1 and 4. The number of trials is $j=30{,}000$ and the experiment is repeated 500 times. Average cumulative distributions for the different estimations $F_1,...,F_6$ are plotted as well as the actual cumulative distribution of $\omega(\varphi_g)$, in yellow. 
}
\label{fig:apprAllLogCombinedAverage}
\end{figure}

\subsection{Probability of False Negative of Z atoms}  \label{FNRforZatoms}

\subsubsection{The Expected Probability of False Negative}

We calculated that the probability of false negative for a single atom is given by the exponential:
\[
P(\omega(\phi)  \geq \delta)  \,  =  \,  (1 - \delta) ^ {a(\phi)}.
\]
for a positive value $a(\phi) >> 1$: 
\[
a(\phi) = \max \left( h(\phi) + 1, \frac{j + 1}{g(\phi) + 1}\right)
\]
where $h(\phi)$ is the number of training positive duples used since the moment atom $\phi$ appears in the model with $g(\phi)$ is the number of links in an inward chain ending at $\phi$. 

The expected value of the probability of false negative, $\omega(\phi_1,....\phi_Z)$, produced by a model of $Z$ atoms is
\[
\overline{\omega}(\phi_1,....\phi_Z) \,  =  \int \omega(\phi_1,....\phi_Z)  \prod_{i=1}^{Z} \frac{dP(\omega_i)}{d\omega_i} d\omega_i 
\]
where
\[
\omega(\phi_1,....\phi_Z) =  1 -  \prod_{i=1}^{Z} (1 - \omega_i)
\]
and
\[
\frac{dP(\omega_i)}{d\omega_i} =  a_i (1 - \omega_i) ^{a_i - 1}.
\]
The solution of the integral is straightforward:
\[
\overline{\omega}(\phi_1,....\phi_Z) \,  =  1 -  \prod_{i=1}^{Z} \frac{a_i}{a_i + 1}. 
\]
This expression assumes the atoms in the model produce false negatives in an independent manner which overestimates the number of false negatives expected.

At first order, and assuming every $a_i$ is a large number, we can approach the expected probability of false negative with:
\[
\overline{\omega}(\phi_1,....\phi_Z) \,  \approx  \sum_{i=1}^{Z} \frac{1}{a_i}. 
\]

\subsubsection{Cumulative distribution function for the PFN of Z atoms}

The cumulative distribution for the probability of false negative, $\omega(\phi_1,....\phi_Z)$, of a model with $Z$ atoms can be computed as the integral:
\[
P(\omega(\phi_1,....\phi_Z)  \geq \Delta) \,  =  \int_{\omega(\phi_1,....\phi_Z)\geq\Delta} \prod_{i=1}^{Z} \frac{dP(\omega_i)}{d\omega_i} d\omega_i,
\]
where the region of integration goes from $\omega(\phi_1,....\phi_Z) = \Delta$ to $\omega(\phi_1,....\phi_Z) = 1$. This integral measures how confident can we be the algebra will indeed behave with a probability of false negative better than $\Delta$. 

If we assume all the values of $a(\phi_i)$ are all different, the result of this integral is:
\[
P(\omega(\phi_1,....\phi_Z)  \geq \Delta) \,  =  \sum_{i=1}^{Z}  (1 - \Delta)^{a_i}  \prod_{k=1, k \neq i}^{Z} \frac{a_k}{a_k - a_i}.
\]
Interestingly, this expression happens to be equal to the value at $x = 0$ of the Lagrange interpolation polynomial passing through the $Z$ points: $(a_1, (1 - \Delta)^{a_1}),...,(a_Z, (1 - \Delta)^{a_Z})$. The coefficients $\frac{a_k}{a_k - a_i}$ usually take very large positive and negative values (compared to 1) which makes the expression opaque and difficult to handle. Surprisingly, the coefficients add up to 1:
\[
\sum_{i=1}^{Z}  \prod_{k=1, k \neq i}^{Z} \frac{a_k}{a_k - a_i} = 1.
\]

The result of the integral when the values of $a(\phi_i)$ are all equal for the $Z$ atoms is more illuminating:
\[
P(\omega(\phi_1,....\phi_Z)  \geq \Delta) \,  =  (1 - \Delta)^{a} \,\,T_a((1 - \Delta)^{-a}, Z - 1) 
\]
where $T_a((1 - \Delta)^{-a}, Z - 1)$ is the sum of the first $Z - 1$ terms of the Taylor series expansion at $a = 0$ of the function $(1 - \Delta)^{-a}$ as a function of $a$. This produces the series:
\[
P(\omega(\phi_1,....\phi_Z)  \geq \Delta) \,  =  (1 - \Delta)^{a} \sum_{n=0}^{Z - 1} \frac{(-a \, \ln(1 - \Delta))^{n}}{n!},
\]
that, defining
\[
\tilde{\Delta} = -\ln(1-\Delta)
\]
becomes
\[
P(\omega(\phi_1,....\phi_Z)  \geq \Delta) \,  =  \sum_{n=0}^{Z - 1} \frac{ e^{-a\tilde{\Delta}} (a \tilde{\Delta})^{n}}{n!},
\]
a function of two variables: $a \tilde{\Delta}$ and $Z$. This expansion corresponds with a regularized gamma function $Q(Z, a \tilde{\Delta})$.

\subsubsection{Point of abrupt confidence transition and the expected training set size}

\begin{figure}
    \centering
	\includegraphics[width=0.4\linewidth]{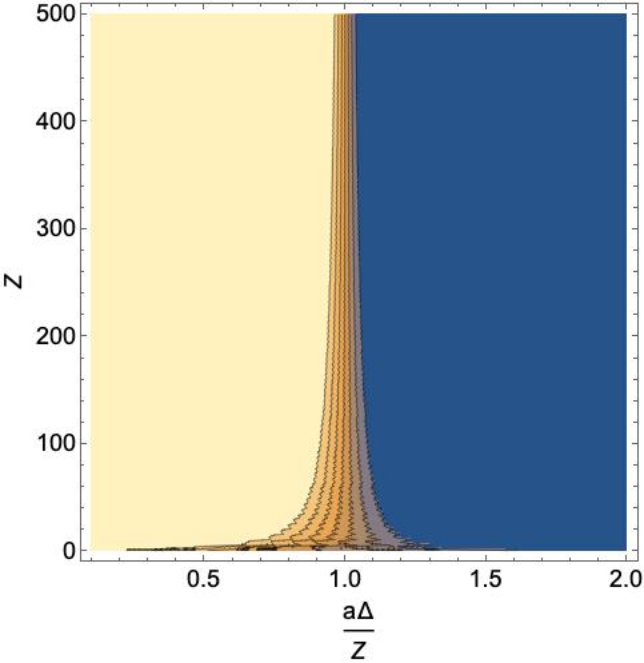}
	\caption{
\textbf{Probability $P(\omega(\phi_1,....\phi_Z)  \geq \Delta)$ that the probability of false negative produced by Z atoms of equal value of $a$ is greater than a value $\Delta$ as a function of $a \tilde{\Delta} / Z$ and $Z$.} 
Warmer colors indicate higher probability. Contours are plotted at values of the probability of 0.2, 0.3,0.4,0.5,0.6,0.7 and 0.8. Note the abrupt transition from high to low probability at $\frac{a \tilde{\Delta}}{Z}=1$ for any value of Z.  }
\label{fig:figX}
\end{figure}

Figure \ref{fig:figX} shows $P(\omega(\phi_1,....\phi_Z)  \geq \Delta)$ as a function $F(y, Z)$ of the variable $y = \frac{a\tilde{\Delta} }{Z}$:
\[
F(y, Z) = \sum_{n=0}^{Z - 1} \frac{ e^{-yZ} (yZ)^{n}}{n!}.
\]
For $y = 1$ this function takes values close to $1/2$ for all $Z$, getting even closer to $1/2$ the larger the $Z$. For a fixed $Z$ the function $F(y, Z)$  transitions from a value very close to 1 to a value very close to 0 abruptly and always at the same point $y = 1$. The transition is more abrupt the larger is $Z$. 

The confidence $P(\omega(\phi_1,....\phi_Z)  \geq \Delta)$ as a function of $\Delta$ monotonically decreases, from $1$ at $\Delta = 0$ to $0$ at $\Delta = 1$. The confidence has an inflection point for any $Z>1$. At the inflection point, $P(\omega(\phi_1,....\phi_Z)  \geq \Delta) \approx \frac{1}{2}$ so, in practice, we can calculate the inflection point instead of resolving the equation $P(\omega(\phi_1,....\phi_Z)  \geq \Delta) = \frac{1}{2}$ which does not lead to a closed formula. The expression for the inflection point i.e.\ the $\Delta$ that makes $\frac{\partial^{2} P}{\partial \Delta^{2}}=0$, is quite simple:
\[
\frac{(a \tilde{\Delta})^{Z} e^{-a \tilde{\Delta}} (1 + a \tilde{\Delta}-Z)}{Z^2 (Z-1)!} = 0.
\]
Resolving this equation we get that the inflection point is located at $\tilde{\Delta} = \frac{Z-1}{a}$. At the inflection point, $P(\omega(\phi_1,....\phi_Z)  \geq \Delta)$ is equal to $\Gamma(Z,Z-1)/(Z-1)!$, with $\Gamma(z,z)$ the incomplete gamma function, which gives values slightly smaller than $1/2$ and approaches $1/2$ for large $Z$. For the neighboring value $\tilde{\Delta} = \frac{Z}{a}$ the confidence $P(\omega(\phi_1,....\phi_Z)  \geq \tilde{\Delta})$ is equal to $\Gamma(Z,Z)/(Z-1)$ and yields values slightly larger than $1/2$ that also tend to $1/2$ for large $Z$. Hence, the transition point we are looking for, the $\Delta$ that makes $P(\omega(\phi_1,....\phi_Z) \geq \Delta) = \frac{1}{2}$, occurs somewhere between $\frac{Z-1}{a}$ and $\frac{Z}{a}$. In practice, the formula:
\[
\frac{Z}{a\tilde{\Delta}} = 1
\]
can be used to calculate the transition point and it has the advantage of working also for the case $Z = 1$. Substituting the value of $\frac{1}{a} = \min(\frac{1}{h(\phi) + 1},  \frac{g(\phi) + 1}{j + 1})$ we get the equation for the transition point:
\[
\frac{Z}{\tilde{\Delta}} \min \left( \frac{1}{h(\phi) + 1},  \frac{g(\phi) + 1}{j + 1} \right)= 1
\]
beyond which $P(\omega(\phi_1,....\phi_Z)  \geq \Delta)$ approaches 0 quite rapidly as $\Delta$ or $j$ increase. From here we can calculate a bound for the necessary number of training examples:
\[
j > \frac{Z \, (g(\phi) + 1) }{\tilde{\Delta}},
\] 
where we have used $h(\phi)  < j$, assumed $j$ large and replaced $j + 1$ by $j$. 

The abrupt transition of the regularized gamma function implies, in practical terms, that a value $\Delta$ (or better) for the probability of false negative can be established with confidence at a particular number of training examples. This confidence changes abruptly from low to high in a narrow window of examples. This does not mean that learning occurs abruptly. Increasingly better probability of false negative values can be gained gradually as the number of training examples increase, but the confidence that we can assign to a given bound $\Delta$ for the probability of false negative changes suddenly.  

We have derived an expression for the transition point that is valid when all the Z atoms have the same $a(\phi)$. We are interested in the value of $\Delta$ at which the transition occurs for sets of atoms with arbitrary values of $a(\phi_i)$, i.e.\ arbitrary value of $h(\phi_i)$ and $g(\phi_i)$.

For the more general case, with atoms of different $a(\phi_i)$, we can use the formula with the Lagrange polynomials $P(\omega(\phi_1,....\phi_Z)  \geq \Delta) \,  =  \sum_{i=1}^{Z}  (1 - \Delta)^{a_i}  \prod_{k=1, k \neq i}^{Z} \frac{a_k}{a_k - a_i}$. We show at the end of this section that the transition occurs for this cumulative distribution at:
\[
\frac{1}{\tilde{\Delta}} \sum_{i=1}^{Z} \frac{1}{a_i} = 1,
\]
which generalizes our previous result $\frac{Z}{a\tilde{\Delta}} = 1$ for $Z$ atoms with the same exponent $a$. Substituting the value of each $a_i$ we get the equation for the transition point:
\[
\frac{L}{\tilde{\Delta}}= 1 \,\,\,\, \text{where} \,\,\,\, L = \sum_{i=1}^{Z} \min \left( \frac{1}{h(\phi_i) + 1},  \frac{g(\phi_i) + 1}{j + 1} \right).
\]
Taking into account $\tilde{\Delta} = -\ln(1-\Delta)$, if we solve for $\Delta$: 
\[
\Delta_T = 1 - e^{-\sum_{i=1}^{Z} \frac{1}{a_i}}.
\] 
we obtain the value of the transition point  $\Delta_T$. This formula and the formula above for the average $\overline{\omega}(\phi_1,....\phi_Z) \,  =  1 -  \prod_{i=1}^{Z} \frac{a_i}{a_i + 1}$ are numerically very close (just do a first order approximation). We have:
\[
 \Delta_T \approx \overline{\omega}(\phi_1,....\phi_Z),
\] which occurs because most of the weight for the integral yielding the average distribution is located at the narrow transition region. Substituting the value of $a_i$
\[
\overline{\omega}(\phi_1,....\phi_Z) = 1 - e^{-L}.
\]
We refer to $L$ as ``the load'', a value that can be computed in practice and can be used to determine, much before it can be established using a test set, how close the algebra is to learning the problem at hand. The load measures the immaturity of the atoms and how much extra learning is still required. The larger is the load, the more far away the model is from producing a good probability of false negative. The load can take values much larger than $1$ at the beginning of learning and, as the algebra learns, it goes down. The algebra produces small values for the probability of false negative when the load is smaller than $1$. For a pattern recognition problem the load tends to first increase and then decrease until it reaches the region with values $L\leq 1$. It is only in the region $L\leq 1$ that the amount of learning can be measured with a test set. When the load is larger than $1$ the algebra produces a $100\%$ false negatives in the test set. In this region learning occurs at constant probability of false negative equal to 1, but can be measured using $L$. 

Solving for $j$ we get the number of training examples needed to get a probability of false negative better than $\Delta$ or, more precisely, to get the number of training examples needed to go beyond the transition point for $\Delta$:
\[
j \geq \frac{\sum_{i=1}^{Z} (g(\phi_i) + 1)}{\tilde{\Delta}}.
\]
where, again, we have assumed $j$ large and replaced $j + 1$ by $j$.

\subsubsection{General case for the transition point}

We are going to do a calculation of the transition point of $P(\omega(\phi_1,....\phi_Z)  \geq \Delta)$ for $Z$ atoms $\phi_i$ with different exponents $a_i$. First notice that, still in this more general case, $P(\omega(\phi_1,....\phi_Z)  \geq \Delta)$ is a function of $\tilde{\Delta} \bf{a}$ rather than a function of $\Delta$ and $\bf{a}$, where we have used $\bf{a}$, in bold, to represent the vector with $Z$ components, $a_i$. We can indeed rewrite the confidence as:
\[
P(\omega(\phi_1,....\phi_Z)  \geq \Delta)  =  \sum_{i=1}^{Z}  e^{\tilde{\Delta} a_i}  \prod_{k=1, k \neq i}^{Z} \frac{\tilde{\Delta} a_k}{\tilde{\Delta} a_k - \tilde{\Delta} a_i}.
\]
Its derivative is:
\[
\frac{d P(\omega(\phi_1,....\phi_Z)  \geq \Delta)}{d \Delta} = \frac{d \tilde{\Delta}}{d\Delta} \, \left(\prod_{j=1}^{Z} a_j\right) \,\,  \sum_{i=1}^{Z}  \left( e^{\tilde{\Delta} a_i} \prod_{k=1, k \neq i}^{Z} \frac{1}{a_k - a_i} \right),
\]
which is not a function of $\tilde{\Delta} \bf{a}$, but the closely related function:
\[
G(\tilde{\Delta} {\bf{a}}) \equiv - \left(\prod_{j=1}^{Z} \tilde{\Delta}  a_j\right) \,\,  \sum_{i=1}^{Z}  \left( e^{-\tilde{\Delta}  a_i}  \prod_{k=1, k \neq i}^{Z} \frac{1}{\tilde{\Delta} a_k - \tilde{\Delta} a_i} \right)
\]
is, and it can be used to calculate the derivative as:
\[
\frac{d P(\omega(\phi_1,....\phi_Z)  \geq \Delta)}{d \Delta}  = G(\tilde{\Delta} {\bf{a}}) \frac{1}{\tilde{\Delta}}   \frac{d \tilde{\Delta}}{d\Delta}. 
\]
The equation $P(\omega(\phi_1,....\phi_Z)  \geq \Delta) = 0.5$ yields the location at which the confidence transitions but it cannot be resolved. We could try, as we did before, to search for the inflection point but this time it does not lead to a closed expression. Fortunately, it is enough with finding a value in the narrow region where the cumulative distribution transitions. Any value on this region works for us.  

We are going to use a convenient translational property of $G(\tilde{\Delta} {\bf{a}})$. If we add the same value $x$ to each of the components of the vector $\tilde{\Delta} \bf{a}$ we get:
\[
G(\tilde{\Delta} {\bf{a}} + x) = \frac{\prod_{j=1}^{Z}  (\tilde{\Delta} a_j + x)}{\prod_{j=1}^{Z} \tilde{\Delta} a_j} e^{-x} G(\tilde{\Delta} {\bf{a}}).
\]
Let us shift the vector $\tilde{\Delta} \bf{a}$ by adding $x$ to each of its components in order to find a value of $x$ that minimizes $\frac{d P(\omega(\phi_1,....\phi_Z)  \geq \Delta)}{d \Delta}$. The value of the transition can be found by locating the stationary point of $G(\tilde{\Delta} {\bf{a}})$:
\[
\frac{dG(\tilde{\Delta} {\bf{a}} + x) }{dx}= 0
\]
from which we get:
\[
\left( - \prod_{j=1}^{Z}  (\tilde{\Delta} a_j + x) + \sum_{i=1}^{Z} \prod_{j=1, i \neq j}^{Z} (\tilde{\Delta} a_j + x)  \right)  \frac{e^{-x} G(\tilde{\Delta} {\bf{a}})}{\prod_{j=1}^{Z} \tilde{\Delta} a_j} = 0,
\]
and, finally:
$
\sum_{i=1}^{Z} \frac{1}{\tilde{\Delta} a_i + x} = 1$, which locates the transition of $P(\omega(\phi_1,....\phi_Z)  \geq \Delta)$ at 
\[
\frac{1}{\tilde{\Delta}} \sum_{i=1}^{Z} \frac{1}{a_i} = 1.
\]

\bigskip

\newpage

\section{Sparse crossing}  \label{suppSection:SparseCrossingInDepth}

The Full Crossing operator can be utilized to construct the freest semilattice that satisfies a given set of axioms. These axioms, denoted as $R = R^{+} \cup R^{-}$, consist of atomic sentences ($R^{+}$) and negated atomic sentences ($R^{-}$) without quantifiers.  The computation of the freest model of $R$ necessitates a full-crossing operation for each atomic sentence in $R^{+}$.

The freest semilattice model satisfying a set of axioms typically has a substantial number of subdirectly irreducible components. Consequently, it is often impractical to first calculate the freest model and then identify a generalizing subset, as the freest model is computationally intractable. To address this challenge, we employ a sparse version of the Full Crossing operator. Sparse Crossing retains the atoms required to maintain the satisfaction of negative axioms while selectively eliminating other atoms. 

The atom elimination process is guided by the following considerations: to discover generalizing models, we aim to identify atoms of the freest model of the positive axioms that most effectively discriminate the negative axioms (see Definition \ref{definition:discriminant}). We also must avoid discarding atoms that, although removable at a particular crossing step without affecting the negative axioms, are essential for maintaining the satisfaction of negative axioms in subsequent crossing steps.

\subsection{Building free models iteratively} \label{SparseCrossing:iterative}

Besides sparsity, several other differences exist between Full and Sparse Crossing. First, while Sparse Crossing utilizes $R^{-}$, Full Crossing does not, as every negative atomic sentence consistent with $R^{+}$ is always satisfied by the freest model of $R^{+}$. Second, Sparse Crossing typically employs batches, breaking $R$ into subsets (often with replacement) and proceeding iteratively, batch by batch, in a process resembling neural network training. In contrast, Full Crossing performs a crossing operation for each axiom of $R^{+}$ in any order, obtaining the final model after crossing each positive tuple exactly once.  An iterative method is necessary because discovering the most discriminative (useful) atoms of the freest model of $R^{+}$ in a single attempt is unlikely to occur.

The processing of a batch $i$ with axioms $R_i = R_i^{+} \cup R_i^{-}$ starts with a model $M_{i-1}$ and ends with a model $M_i$ that satisfies $R_i$. We thus obtain a sequence of atomized models ${M_0, M_1, \ldots}$. The sequence starts with an initial model $M_0$, usually equal to the freest semilattice over the empty set of axioms, $M_0 = F_C(\emptyset)$. The model $F_C(\emptyset)$ has one atom for each constant in $C$, so it can be easily computed (see Supplementary Section \ref{supplementary:review}, Theorem \ref{freestModelTheorem}). We refer to the models $M_i$ as ``master models".

After each master model $M_i$ is calculated, we add it to a model that contains the result of processing previous batches, generating a second sequence of models $\{N_0, N_1, \ldots\}$. We refer to these models as ``union models"
The model $N_{i-1} + M_i$ always satisfies $R_i^{-}$; however, it does not necessarily satisfy $R_i^{+}$. To resolve this, we discard the atoms of $N_{i-1}$ that are not consistent with $R_i^{+}$: \[
N_i = \{ \phi \in N_{i - 1 }: \phi \models R_i^{+} \} \cup M_i.
\]
Since $ \{ \phi \in N_{i - 1 }: \phi \models R_i^{+} \} \models R_i^{+}$ and $M_i \models R_i^{+}$, then  $N_i \models R_i^{+}$. Since  $M_i \models R_i^{-}$ and the atoms of  $M_i$ are atoms of $N_i$ then $N_i$  satisfies $R_i^{-}$. It follows, $N_i \models R_i$.

As batches are processed, the union models become more consistent with $R$, usually larger, and always composed of atoms that are more effective at discriminating the axioms of $R^{-}$. After a number of batches, the union models satisfy the entire set of axioms $R$ (the training set), and the training reaches a point after which no atoms are removed and union models only grow larger. We can say that, at that point, all the atoms in the union model have reached full maturity.

\subsection{Keeping model size under control}  \label{SparseCrossing:modelSizeUnderControl}

The full-crossing (see Definition \ref{definition:full_crossing}) of a duple $r = (r_L, r_R) \in R^{+}$ over a model $M$ involves removing the atoms of the discriminant $dis_M(r)$ and then adding the atoms in $dis_M(r) \bigtriangledown L_M^a(r_R)$ where the lower atomic segment $L_M^a(r_R) = \{ \varphi \in M : \varphi < r_R\}$ is the set with the atoms of the right-hand side $r_R$ of the duple. The full-crossing operation transforms a model $M$ into the model: \[
\square_r M = (M - dis_M(r)) + dis_M(r) \bigtriangledown L_M^a(r_R)
\]
where we have used $-$ for set subtraction and $+$ for set union. 

We can view a full-crossing operation as an atom replacement process. An atom $\phi \in dis_M(r)$ is replaced by multiple atoms $\phi \bigtriangledown \psi_k$, where $\psi_k$ are the atoms in the lower atomic segment $L^a(r_R)$. Specifically, an atom $\phi$ in the discriminant is replaced by the set of atoms:\[
\phi \rightarrow \{ \phi \bigtriangledown \psi_k : \psi_k \in L^a(r_R) \}.
\]

To ensure that atom discarding allows for the satisfaction of all negative axioms, even after subsequent crossing steps, Sparse Crossing relies on an invariant called the trace. The trace is a function that maps an element of an atomized semilattice model (either an atom or a regular element) to a set. The rationale behind this approach is that we can eliminate an atom generated by Full Crossing as long as we maintain the trace of all elements in the model unchanged. 

The trace function (see Theorem \ref{SparseCrossing:coreProperty}) has the property: \[
Tr(\phi) = \cap_{\psi_k \in L_M^a(r_R)} \, Tr(\phi \bigtriangledown \psi_k),
\]
so it is possible to discard some of the atoms $\phi \bigtriangledown \psi_k$ in the right-hand side if the trace $Tr(\phi)$ given by this expression remains the same. In other words, instead of the replacement above, Sparse Crossing uses: \[
\phi \rightarrow \{ \phi \bigtriangledown \psi_k : \psi_k \in S_{\phi} \subseteq L_M^a(r_R) \},
\] 
where $S_{\phi}$ is any subset of $ L_M^a(r_R) $ that satisfies \[
Tr(\phi) = \cap_{\psi_k \in S_{\phi}} \, Tr(\phi \bigtriangledown \psi_k).
\]
This replacement operation keeps the traces of every element of the atomized semilattice unaltered as proven in Theorem \ref{SparseCrossing:invarianceOfSparseCrossing}, and corresponds with Algorithm \ref{SparseCrossing:sparseCrossing} in the pseudocode. 

Assume a model $M \models R^{-}_i$. The sparse-crossing of a duple $r$ over $M$ transforms $M$ in a model $M^\prime$ that satisfies $r^{+}$ and $R^{-}_i$  and it is atomized by a subset of the atoms of $\,\square_r M$, i.e.\ it is less free than $\square_r M$ since the atoms of $M^\prime$ are all atoms of $\,\square_r M$.

\subsection{Building the dual} \label{SparseCrossing:buildingTheDual}

Before calculating the sparse-crossing of the positive tuples in $R^{+}_i$ of batch $i$, we must first determine an appropriate codomain for the trace function. We refer to such codomain as ``the dual''.

Consider the set of constants and terms: \[
E_i = C \cup Terms(R_i^{+} \cup R_i^{-}) \cup \{ T_{\phi} : \phi \in N_{i - 1} \wedge \phi \models R_i^{+} \} 
\]
where $T_{\phi}$ is the pinning term of an atom $\phi$, i.e., the term with component constants $C - U^c(\phi)$, and: \[
Terms(R_i) = \{ r_L : r \in R_i \} \cup \{ r_R : r \in R_i \}
\]
are the terms mentioned in the set of axioms $R_i$ of the batch $i$. 

Let the dual $D_i$ be an atomized semilattice generated by a set of constants $C^{*}$ with cardinality $\vert C^{*} \vert =\vert E_i \vert$. The set $C^{*}$ contains a constant for each constant of $C$, a constant for each term mentioned in $R_i^{+} \cup R_i^{-}$ and a constant for each pinning term $T_{\phi}$ of the union model $N_{i - 1}$ after it has been filtered out using $R_i^{+}$. We denote the constant in $C^{*}$ corresponding to the element $x \in E_i$ with the symbol $[x]$, which we refer to as ``the dual of $x$."

The dual $D_i$ satisfies the dual of the axioms of $R_i$: \[
(r_L, r_R) \in R_i \rightarrow ([r_R], [r_L]) \in R_i^{*}, 
\]
i.e.\ the axioms $R^{*} =R_i^{*+} \cup R_i^{*-}$. 

Since every element mentioned in $R_i^{*}$ is a constant of $D_i$, atomizing the dual is very easy; it is enough with adding an atom $\xi_c$ to each constant $c$ of $C^{*}$, building a directed graph connecting the constants and then calculating the transitive closure: 

The formal definition of the dual is given in Definition \ref{SparseCrossing:dualDefinition}, and the method to compute it is presented in Section \ref{SparseCrossing:dualCalculation}.

An important result is that the dual $D_i$ satisfies $R_i^{*-}$ if and only if $R_i$ is a consistent set of axioms (see Theorem \ref{SparseCrossing:weakDuality}). The fact that $D_i$ can be atomized if and only if the axioms $R_i$ are consistent is particularly advantageous as it provides a consistency check at no additional computational cost.

Optionally, we can discard atoms in the dual as long as $R_i^{*-}$ remains satisfied, which can always be achieved with no more than $\vert R_i^{-} \vert$ atoms (see Theorem \ref{SparseCrossing:smallerModelThatSatisfiesR}).

Theorem \ref{SparseCrossing:weakDuality} demonstrates that $D_i$ is the dual semilattice of the intersection of $F_C(R_i^{+})$ with $E_i$.  Furthermore, $D_i$ is equal to $F_{E_i}(R_i^{*+})$. It is important to note that although we refer to $D_i$ as ``the dual", it only corresponds to a small subset of the dual of $F_C(R_i^{+})$.

\subsection{Building the trace}

Once the dual $D_i$ is constructed it is possible to build a trace function for batch $i$.  Given a dual $D$, the trace is a function that maps elements of a model $M$ to a subset of atoms of $D$. To calculate the trace of any atom over $C$, we must use:\[
Tr(\phi) = \{ \xi : \exists c  \in C, \,\, (\phi < c) \wedge (D \models \, \xi < [c]) \} 
\] 
where $c \in C$ is a constant and $\xi$ is an atom of the dual. It follows that $Tr(\phi) = \cup_{c \in U^c(\phi)} L^a([c])$, where we write $L^a([c])$ instead of $L^a_D([c])$ for brevity, as there is no ambiguity. This formula corresponds to Algorithm \ref{SparseCrossing:traceOfAtom}. The trace of a union of atoms satisfies (see Theorem \ref{SparseCrossing:unionOfTraces}): \[ 
Tr(\phi \bigtriangledown \psi) = Tr(\phi) \cup Tr(\psi).
\]
The trace of an atom is universal, i.e., it is the same for any model and depends solely on the dual. Given the traces of the atoms, we can calculate the trace of a regular element (i.e., a constant or an idempotent summation of constants) in a model $M$ using: \[
Tr_M(x) = \cap_{\phi \in L_M^a(x)} Tr(\phi),
\] 
which is what Algorithm \ref{SparseCrossing:traceOfTerm} does. Unlike the traces of atoms, the traces of regular elements depend on the specific model $M$ and are therefore not universal. However, in Theorem \ref{SparseCrossing:nonRedundant}, we demonstrate that the traces of regular elements are well-defined, meaning they are determined by the model $M$ itself and not by a particular atomization of $M$. In general, and as long as there is no ambiguity, we omit the model subindex from the notation and write $Tr(x)$ for the trace of a regular element.

\subsection{Trace constraints and $\Lambda_i$}

Before initiating the sparse-crossing of the duples in $R_i^+$, we must ensure that our initial model satisfies $R_i^{-}$. This can always be achieved by adding to the model $M_{i-1}$, obtained from the previous batch, a set of atoms $\Lambda_i$;  in this section we will learn how to compute  $\Lambda_{i}$. 

We begin by calculating the trace of every atom in our initial model $M_{i-1}$. Using these atom traces, we compute the traces of the terms in $R_i$ and ensure that the following two sets of constraints are satisfied: \[
r \in R_i^{+} \rightarrow \,\, Tr(r_R) \subseteq Tr(r_L), 
\] \[
r \in R_i^{-} \rightarrow \,\, Tr(r_R) \not\subseteq Tr(r_L). 
\] 
We refer to these sets of constraints as positive and negative trace constraints, respectively.

The negative trace constraints are necessary to ensure that $R_i^{-}$ holds at all times. According to Theorem \ref{SparseCrossing:traceImplication}, as long as a model satisfies $T(b) \not\subseteq T(a)$, then $M \models a \not\leq b$. It follows that one or multiple trace-preserving operations on a model that satisfies $T(b) \not\subseteq T(a)$ will always yield a model that satisfies $a \not\leq b$. Since all the operations we use are trace-invariant, we have guaranteed that at the end of the crossing phase, all positive and negative axioms are satisfied.

Positive trace constraints are necessary as they imply that an atom in the discriminant has a trace $Tr(\phi)$ equal to the intersection of the traces $Tr(\phi \bigtriangledown \psi_k)$ of the atoms generated during the crossing (see Theorem \ref{SparseCrossing:coreProperty}). If the positive trace constraints are not satisfied, the crossing operation is not trace-invariant.

Theorem \ref{SparseCrossing:traceConstraints} demonstrates that both sets of trace constraints can be satisfied by merely adding new atoms to $M_{i - 1}$, each with a single constant in its upper constant segment. In fact, adding one atom for each constant of $C$ always produces a model that satisfies the trace constraints (see Theorems \ref{SparseCrossing:freeModelConstraints} and \ref{SparseCrossing:nonRedundant}). However, we aim to minimize the number of added atoms, as an increase in atoms leads to more computational work. Trace constraints can typically be satisfied by adding only a few atoms to $M_{i - 1}$. As shown in the proof of Theorem \ref{SparseCrossing:traceConstraints}, a negative trace constraint for a duple $r \in R^{-}$ that is not satisfied can always be rectified by adding new atoms, each under a single constant in the lower segment of $r_L$. Similarly, a positive trace constraint for a duple $r \in R^{+}$ can be enforced by adding new atoms, each under a single constant, in the lower segment of $r_R$.

The procedure to compute the new atoms that ensure all trace constraints are satisfied, the set $\Lambda_i$, is presented in pseudocode in Algorithm \ref{SparseCrossing:traceConstraintsAlgorithm}. To understand why this works see the proof of Theorem \ref{SparseCrossing:traceConstraints}.

\subsection{A note on duality and trace constraints}

To better understand the role played by the trace constraints consider the following. Restricted to the set $E_i$, there is a duality given by $\cup \leftrightarrow \cap$ and $L^a_{F_C(R_i^{+})}(t) \leftrightarrow L_{D_i}^{a}([t])$. To be more specific, for $a,b \in E_i$ we have $F_C(R_i^{+}) \models a \leq b$ if and only if $L^{a}([b]) \subseteq L^{a}([a])$ where we have omitted the subindex in $L_{D_i}^{a}$ for less clutter. We also have $L^a([b]) = \cap_{c \in C(b)} L^a([c])$ (see Theorem \ref{SparseCrossing:termDualAtoms}). We saw in Section \ref{SparseCrossing:buildingTheDual} that $D_i$ corresponds, via duality, to a small subset $E_i$ of $F_C(R_i^{+})$. 

This duality extends beyond the small subset $E_i$ through the trace. For any two terms of any semilattice $M$, we have that $M \models a \leq b$ implies $Tr(b) \subseteq Tr(a)$, which holds true even if the trace constraints are not satisfied. This duality is weaker for $M$ than for the dual $D_i$, as it only operates from left to right; however, it is also stronger as it is not restricted to $E_i$ and applies to all elements of $M$. If the positive trace constraints hold for $M$, we also have $F_C(R_i^{+}) \models a \leq b$ implies $Tr(b) \subseteq Tr(a)$. Furthermore, when the positive trace constraints hold for $M$, for any two terms $a, b$ such that $\square_{R^{+}} M \models a \leq b$, the trace in $M$ satisfies $Tr(b) \subseteq Tr(a)$ (although the implication from right to left is not necessarily true). This is demonstrated in Theorem \ref{SparseCrossing:traceConstraintsUniversal}.

\subsection{Trace-invariant simplification of the master}

The trace is linear with respect to the idempotent operator (see Theorem \ref{SparseCrossing:unionOfTraces}): \[
Tr(a \odot b) = Tr(a) \cap Tr(b).
\] 
This property also implies that a regular element $x$ of the atomized semilattice: \[
Tr(x) = \cap_{c \in C(x)} Tr(c).
\] 
where $C(x)$ is the set of component constants of $x$ (see also Theorem \ref{SparseCrossing:traceEqualitiesForAtoms}).

By construction, sparse-crossing a duple is a trace-invariant operation as it leaves the traces of all elements unchanged. Since the trace of a term is equal to the intersection of the traces of its component constants we can build another trace-invariant operation; removing atoms while leaving the trace of every constant unaltered shall be trace-invariant. 

Discarding atoms while preserving the traces of the constants is a relatively efficient and effective operation that can be applied at any time. This operation requires iterating through the constants of the model rather than through the elements mentioned in the axioms. This simplification procedure is implemented in Algorithm \ref{SparseCrossing:simplifyFromConstants}.

A master model that satisfies the trace constraints can be computed with a size that depends only on the size of the dual (see Theorem \ref{SparseCrossing:masterSize}). Since the atomization of the dual requires no more than $\vert R^{-} \vert$ atoms, we have a guarantee that, by using trace-invariant simplification and sparse-crossing, the master model maintains a manageable size at all times. 

\subsection{Crossing smarter every batch} \label{SparseCrossing:smarter}

The pinning terms: \[
T_{\phi} : \phi \in N_{i - 1} \wedge \phi \models R_i^{+} 
\]
are present in the set $E_i$, which determines the constants of the dual. However, a dual could be constructed without pinning terms, using the smaller set $E_i = C \cup Terms(R_i^{+} \cup R_i^{-})$, and a valid trace would be obtained, allowing for the calculation of Sparse Crossing.

By including the pinning terms in $E_i$, we allow the atoms learned in previous batches to influence the crossing process.  The effect of adding the pining terms of (the atoms of) the union model $N_{i - 1}$ is faster convergence and better master models.

The atomization of the dual can be simplified by discarding atoms, following the straightforward procedure discussed in the proof of Theorem \ref{SparseCrossing:smallerModelThatSatisfiesR}. This simplification becomes particularly effective when pinning terms are taken into account, often resulting in a dual with significantly fewer than $\vert R_i^{-} \vert$ atoms. The size of the master can be bounded by $\alpha\vert D \vert$, where $\alpha$ is the average size of an atom in $N_{i - 1}$ and $\vert D \vert$ is the number of atoms in the dual. A smaller dual not only leads to a smaller master $M_i$ (see Theorem \ref{SparseCrossing:masterSize}) but also accelerates its calculation.

Theorem \ref{SparseCrossing:indicatorsTheorem} (ii) clarifies the guiding role played by the pinning terms.  There is a pinning term $T_{\psi}$ for each atom $\psi \in N_{i  -1}$. This Theorem states that, given any atom $\phi$ over $C$, the atom $\xi_{\psi}$ of the dual that is edged to $[T_{\psi}]$ satisfies: \[
\xi_{\psi} \not\in Tr(\phi)\text{ if and only if } U^c(\phi) \subseteq U^c(\psi). 
\]
Assume $\phi$ is replaced with $\phi \rightarrow \{ \phi \bigtriangledown \varphi_k : \varphi_k \in S_{\phi} \subseteq L_M^a(r_R) \}$ where the set $S_{\phi}$ is selected so the replacement is trace-invariant. Since $Tr(\phi \bigtriangledown \varphi_k) = Tr(\phi) \cup Tr(\varphi_k)$: \[
\xi_{\psi} \not\in Tr(\phi) \cup Tr(\varphi_k) \text{ if and only if } U^c(\phi \bigtriangledown \varphi_k) \subseteq U^c(\psi). 
\]
The invariance of the trace, $Tr(\phi) = \cap_{\varphi_k \in S_{\phi}} , Tr(\phi \bigtriangledown \varphi_k)$, requires that for each $\xi_{\psi} \not\in Tr(\phi)$, there exists at least one $\varphi_k \in S_{\phi}$ such that $\xi_{\psi} \not\in Tr(\varphi_k)$. It follows that when $U^c(\phi) \subseteq U^c(\psi)$, the atom $\phi$ is replaced by at least one atom with $U^c(\phi \bigtriangledown \varphi_k) \subseteq U^c(\psi)$. This must occur for the potentially numerous atoms $\psi$ that satisfy $\xi_{\psi} \not\in Tr(\phi)$. Since the selection of the atoms $\varphi_k \in S_{\phi}$ is simultaneously trace-invariant and of minimal cardinality (minimal in a best-effort sense), the selection results in atoms $\phi \bigtriangledown \varphi_k$ with upper constant segments that are, often, subsets of $U^c(\psi)$ for many atoms $\psi \in N_{i-1}$.

This mechanism generates atoms that are common components of previously learned atoms, which explains why non-redundant atoms can be discovered from the pinning terms of redundant atoms. The inclusion of pinning terms results in models that are not only consistent with $R_i$ but also more consistent with previous batches and, consequently, more consistent with $R$, an observation supported by experimental evidence. Without pinning terms in the dual, the union of master models still approximates the freest model of the axioms as training progresses; however, learning may be significantly slower.

\subsection{Efficient computation of the dual} \label{SparseCrossing:dualCalculation}

The construction of the dual can be carried out as described in Definition \ref{SparseCrossing:dualDefinition}, which involves constructing a graph and computing its transitive closure. Although this formal definition is useful and facilitates Theorem proving, graph manipulations can be computationally expensive and are not necessary. 

We present here a simple method that involves a few stages, all parallelizable.  This method corresponds to Algorithm \ref{SparseCrossing:dualAlgebraAlgorithm}, and it is based on computing ``indicator sets'' that become atoms of the dual at the end of the computation:

1 - For each term at the right hand side of any duple in $R^{-}$ create an ``indicator set" equal to component constants set of the term.

2 - For each duple $r = (r_L, r_R)$ in $R^{+}$ and for each indicator set, if the set of component constants of $r_R$, the set $C(r_R)$, is a subset of the indicator set, add $C(r_L)$ to the indicator set. 

3 - Repeat step 3 until not indicator set undergoes further changes. 

4 - Create a vector $v_1$ with the indicator sets, in any order. Make sure each indicator set only appears once by removing repeated indicator sets.

5 - Transform the set of atoms $\phi \in N_{i - 1}$ into a vector $v_2$, in any order.  Concatenate vectors $v_1$ and $v_2$ in a single vector $v$. Each position (an index) of the vector $v$ becomes the index of an atom of the dual.

6 - For each element $s$ in the set $C \cup Terms(R_i^{+} \cup R_i^{-})$ compute the lower atomic segment of its dual as follows: the lower atomic segment of $[s]$ contains the atom of the dual corresponding to position $i$ of vector $v_1$ if the set of component constants of $s$ is a subset of the indicator set $v_1(i)$. In addition, the lower atomic segment  $[s]$ contains the atom with index $\vert v_1 \vert + i$ if  the set of component constants of $s$ is disjoint with $U^c(v_2(i))$.

7 - Optionally, reduce the set of atoms of the dual as in Theorem \ref{SparseCrossing:smallerModelThatSatisfiesR}.

The upper constant segments of the atoms of the dual are not explicitly computed because they are never needed; it is enough with knowing the lower atomic segment of the elements in the set $C \cup Terms(R_i^{+} \cup R_i^{-})$.

The method assumes that every pinning term $T_{\psi}$ corresponds to an atom $\psi$ consistent with $R_i^+$. Theorem \ref{SparseCrossing:validExclusion} states that for such terms, every constant $c$ for which $[c]$ is in the upper segment of $[T_{\psi}]$ is a component constant of $T_{\psi}$. Consequently, there is no need to apply Step 2 to the pinning terms as it would have no effect on the trace function.

A potential alternative approach to building a dual is given by Haidar et al.\ in \cite{haidar}. 

\subsection{Putting all together}

The main loop of Sparse Crossing corresponds to Algorithm \ref{SparseCrossing:batchTraining}. The input consists of a set of axioms $R$ and an initial model $M_0$. If no $M_0$ is provided, the freest model $M_0 = F_C(\emptyset)$ can be used. The axioms may be divided into any number of batches $R_i$, with or without replacement, typically with replacement.

Initialize the union model $N_0 = M_0$. For each batch:

Step 1: Trim the union model $N_{i - 1}$ removing the atoms inconsistent with $R_i^{+}$.

Step 2: Build the dual $D_i$.

Step 3: Use $D_i$ to calculate the trace of every atom in $M_{i - 1}$. Using these traces, compute and evaluate each trace constraint. For each unsatisfied constraint, add atoms to $M_{i - 1}$ until the constraint is satisfied. Repeat this step as necessary (typically twice) until all trace constraints are satisfied.

Step 4: Calculate the sparse-crossing of every duple in $R_i^{+}$. The resulting model after enforcing all positive duples is $M_i$, which satisfies $R_i$. If at any point the model exceeds a predetermined threshold, discard atoms while maintaining the traces of all constants invariant.

Step 5: Obtain $N_i$ as the union of $M_i$ with the trimmed union model $N_{i -1}$. 

\subsection{Definitions} 

\begin{definition} \label{SparseCrossing:dualDefinition}
Given the set $E_i$ of terms and the axioms of the dual $R_i^{*+}$ as in Section \ref{SparseCrossing:buildingTheDual}, let $G$ be the graph with the following edges: 

1 - a directed edge, $[r_R] \rightarrow [r_L]$, for each duple in $R_i^{*+}$ 

2 - directed edges $[x] \rightarrow [c]$ for each component constant $c$ of every term $x$ in $E_i$.

3 - a directed edge $[y] \rightarrow [x]$ between any two elements $x$ and $y$ of $E_i$ such that $C(x) \subseteq EC(y)$. Here, $C(x)$ is the set of component constants of $x$ and $EC(y)$ is a superset of $C(y)$, the ``extended set" $EC(y) = \{y \} \cup \{c : [y] \rightarrow [c] \text{ is in G and } c\in C \}$. 

4 - a directed edge $\xi_x \rightarrow [x]$ for each element $x \in E$, where $\xi_x$ is an atom of the dual specifically created for $x$. 

After the transitive closure of the graph $G$ is calculated, step 3 must be reviewed until no more edges are added and the graph is transitively closed. We refer as \textbf{``the dual"} to the semilattice $D$ atomized by the atoms $\{ \xi_x : x \in E \}$ with upper constant segments given by $G$.
\end{definition}
\bigskip

\subsection{Theorems} \label{suppSection:sparseCrossingTheorems}

\begin{theorem} \label{SparseCrossing:unionOfTraces}
The trace satisfies the following properties: \\
i) $Tr(\phi \bigtriangledown \psi) = Tr(\phi) \cup Tr(\psi)$ for any two atoms $\phi$ and $\psi$. \\
ii) $Tr(r \odot s) = Tr(r) \cap Tr(s)$ for any two terms $r$ and $s$, \\
iii) $Tr(t) = \cap_{c \in C(t)} Tr(c)$ for any term $t$.
\end{theorem}
\begin{proof}
(i) The upper constant segment $U^c(\phi \bigtriangledown \psi) = \{c < \phi \bigtriangledown \psi : c \in C \}$ of a union of atoms $\phi \bigtriangledown \psi$ is defined by $U^c(\phi \bigtriangledown \psi) = U^c(\phi) \cup U^c(\psi)$. Notice that this is a universal property, i.e.\ it does not assume the atoms belong to any model. From the definition of the trace of an atom $Tr(\phi) = \{ \xi : \exists c \in C, \,\, (\phi < c) \wedge (D_i \models \, \xi < [c]) \}$, follows that the trace of a union of atoms is $Tr(\phi \bigtriangledown \psi) = \{ \xi : \exists c  \in C, \,\, (\phi < c) \vee (\psi < c) \wedge (D_i \models \, \xi < [c]) \} =Tr(\phi) \cup Tr(\psi)$. \\
(ii) Let $M$ be any atomized model. $Tr(r \odot s) = \cap_{\varphi \in L_M^a(r \odot s)} Tr(\varphi)$ and using $L_M^a(r \odot s)= L_M^a(r) \cup L_M^a(s)$ then $Tr(r \odot s) = (\cap_{\varphi \in L_M^a(r)} Tr(\varphi)) \cap (\cap_{\varphi \in L_M^a(s)} Tr(\varphi)) = Tr(r) \cap Tr(s)$. \\
(iii) $Tr(t) = \cap_{\phi \in L_M^a(t)} Tr(\phi) = \cap_{c \in C(t)} \cap_{\phi \in L_M^a(c)} Tr(\phi) = \cap_{c \in C(t)} Tr(c)$ where we have used $L_M^a(t) = \cup_{c \in C(t)} L_M^a(c)$.
\end{proof}
\bigskip

\begin{theorem} \label{SparseCrossing:gainOfTr}
Let $\phi$ be any atom and $t$ any term: \[
\cap_{\psi_k \in L_M^a(t)} \, Tr(\phi \bigtriangledown \psi_k) = Tr(\phi) \cup Tr(t). \]
\end{theorem}
\begin{proof}
This is a consequence of the linearity of the trace (see Theorem \ref{SparseCrossing:unionOfTraces}) for the atoms; \[
\cap_{\psi_k \in L_M^a(t)} \, Tr(\phi \bigtriangledown \psi_k) = \cap_{\psi_k \in L_M^a(t)} \, (Tr(\phi) \cup Tr(\psi_k)) = Tr(\phi) \cup \,( \cap_{\psi_k \in L_M^a(t)} \, Tr(\psi_k) )
\] and using $\cap_{\psi_k \in L_M^a(t)} \, Tr(\psi_k) = Tr(t)$ we get to $ Tr(\phi) \cup Tr(t)$.
\end{proof}
\bigskip

\begin{theorem} \label{SparseCrossing:coreProperty} 
Let $\phi \in dis_M(r)$ in a model $M$. If $M$ satisfies the positive trace constraint $r \in R_i^{+} \rightarrow \,\, Tr(r_R) \subseteq Tr(r_L)$, then $Tr(\phi) = \cap_{\psi_k \in L_M^a(r_R)} \, Tr(\phi \bigtriangledown \psi_k)$.
\end{theorem}
\begin{proof}
According to Theorem \ref{SparseCrossing:gainOfTr},  $\cap_{\psi_k \in L_M^a(r_R)} \, Tr(\phi \bigtriangledown \psi_k) = Tr(\phi) \cup Tr(r_R)$. Assume the positive trace constraint holds. Then $Tr(r_R) \subseteq Tr(r_L) = \cap_{\varphi \in L_M^a(r_L)} Tr(\varphi)$ where the definition for the trace of a regular element $Tr(x) = \cap_{\varphi \in L_M^a(x)} Tr(\varphi)$ has been used. Since $\phi \in dis_M(r)$ implies $\phi \in L_M^a(r_L)$ it follows that $Tr(r_L) \subseteq Tr(\phi)$. Therefore, $Tr(r_R) \subseteq Tr(r_L) \subseteq Tr(\phi)$. Hence, $Tr(\phi) \cup Tr(r_R) = Tr(\phi)$, and also $\, \cap_{\psi_k \in L_M^a(r_R)} \, Tr(\phi \bigtriangledown \psi_k) = Tr(\phi)$.
\end{proof}
\bigskip

\begin{theorem} \label{SparseCrossing:invarianceOfSparseCrossing} 
Let $\phi \in dis_M(r)$. The replacement $\phi \rightarrow \{ \phi \bigtriangledown \psi_k : \psi_k \in S_{\phi} \subseteq L_M^a(r_R) \}$ where $S_{\phi}$ is a subset of $ L_M^a(r_R) $ that satisfies $
Tr(\phi) = \cap_{\psi_k \in S_{\phi}} \, Tr(\phi \bigtriangledown \psi_k)$ is a trace invariant operation, i.e.\ it keeps the trace of all atoms and regular elements of $M$ unchanged. If $M$ satisfies the positive trace constraint for $r \in R_i^{+} \rightarrow \,\, Tr(r_R) \subseteq Tr(r_L)$ such set $S_{\phi}$ always exists.
\end{theorem}
\begin{proof}
Let $N$ be the model resulting from replacing $\phi$ in $M$.  Traces of atoms depend upon the dual only and not upon the presence of other atoms in $M$, therefore their traces remain unchanged.   

Let $x$ be a regular element and assume $\psi_k < x$ and $\phi \not< x$ for some $k$. If $x$ is in the upper segment of $\psi_k$, then some component constant $c$ of $x$ is in the upper segment of $\psi_k$ and it follows that $c$ is also in the upper segment of $\phi \bigtriangledown \psi_k$. Therefore, any regular element $x$ that satisfies $\psi_k \in L_M^a(x)$ in $M$ satisfies $\{ \psi_k, \, \phi \bigtriangledown \psi_k\} \subseteq L_N^a(x)$ in $N$. However, because $Tr(\psi_k)  \subseteq  Tr(\phi \bigtriangledown \psi_k)$, which is a consequence of Theorem \ref{SparseCrossing:unionOfTraces}, the new atoms $\phi \bigtriangledown \psi_k$ have no effect over the trace of $x$; indeed, $Tr_N(x) = Tr_M(x) \cap Tr(\phi \bigtriangledown \psi_k)$ and, since $Tr_M(x) \subseteq Tr(\psi_k)$, then $Tr_N(x) = Tr_M(x)$. \\
Assume now that $\phi < x$. If $x$ is in the upper segment of $\phi$ then some constant $c$ of $x$ is in the upper segment of $\phi$ and it follows $\phi \bigtriangledown \psi_k < c$ for every $\psi_k$ because $c$ is also in the upper segment of $\phi \bigtriangledown \psi_k$. Now, $\phi$ no longer remains on $N$ because it is replaced. In $M$ the trace of $x$ is $Tr_M(x) = \cap_{\varphi \in L_M^a(x)} Tr(\varphi) = \cap_{\varphi < x : \varphi \in M} Tr(\varphi)$, and in $N$ we have: \[
Tr_N(x) = \cap_{\varphi \in L_N^a(x)} Tr(\varphi) = \cap_{\varphi < x : \varphi \in N} Tr(\varphi) = \cap_{\varphi < x : (\varphi \in M) \wedge (\varphi \not= \phi ) } Tr(\varphi) \cap_{\psi_k \in S_{\phi}} \, Tr(\phi \bigtriangledown \psi_k) = 
\]\[
= \cap_{\varphi < x : (\varphi \in M) \wedge (\varphi \not= \phi ) } Tr(\varphi) \cap Tr(\phi) = \cap_{\varphi < x : \varphi \in M } Tr(\varphi) = Tr_M(x)
\]
where we have used our assumption  $Tr(\phi) = \cap_{\psi_k \in S_{\phi}} \, Tr(\phi \bigtriangledown \psi_k)$. It follows, $Tr_N(x) = Tr_M(x)$. Since the trace of every atom and every regular element is unaffected by the replacement we conclude that it is a trace invariant operation.  

Finally, notice that, according to Theorem \ref{SparseCrossing:coreProperty}, if the positive trace constraint is satisfied then $S_{\phi} = L_M^a(r_R)$ fulfills the required condition, so the set $S_{\phi}$ always exists.
\end{proof}
\bigskip

\begin{theorem} \label{SparseCrossing:traceImplication} 
For any two regular elements $a$ and $b$ and any model $M$: \[
M \models a \leq b \,\, \Rightarrow Tr(b) \subseteq Tr(a)
\]\[
Tr(b) \not\subseteq Tr(a) \,\, \Rightarrow M \models a \not\leq b,
\]
so, if $M$ satisfies the negative trace constraints for $R^{-}$ then $M \models R^{-}$.
\end{theorem}
\begin{proof}
$M \models a \leq b$ is equivalent to $L_M^a(a) \subseteq L_M^a(b)$ and the traces $Tr(b) = \cap_{\phi \in L_M^a(b)} Tr(\phi) \subseteq \cap_{\phi \in L_M^a(a)} Tr(\phi) = Tr(a)$. The implication below is just the negation of the first and is the basis for the mechanism that prevents negative duples from becoming positive during the sparse-crossing of the positive duples. 
\end{proof}
\bigskip

\begin{theorem} \label{SparseCrossing:weakDuality} 
Let $D$ be a dual built for the axioms $R$.  Assume $a$ and $b$ are terms over $C$ that have duals in $D$; \\
i) $D \models [b] \leq [a] \,\,\, \Leftrightarrow \,\,\, F_C(R^{+}) \models a \leq b$, \\
ii) $D \models R^{*}$ if and only $R$ is consistent. \\
iii) let $\xi_b$ be the introduced in the graph of the dual with an initial edge to $[b]$. Then, \[
\xi_b \in L^{a}([a])\text{ if and only if } \xi_b \text{ exists and } b \in W(a) = \{ e : (e \in E) \wedge F_C(R^{+}) \models a \leq e \} .\] Remark: only the terms mentioned in $R$ and the pinning terms have a dual, i.e.\ the elements in the set $E$ of Section \ref{SparseCrossing:buildingTheDual}. This result does not apply to the remaining $2^{\vert C \vert} - 1 - \vert E \vert $ terms. 
\end{theorem}
\begin{proof}
(i) It is well known that $F_C(R^{+}) \models (a \leq b)$ if and only if $R^{+} \Rightarrow (a \leq b)$ (see \cite{SecondPaperArX} and, for example, \cite{Burris}). Since $R^{+} \Rightarrow (a \leq b)$ is equivalent to $R^{+} \Rightarrow (b = b \odot a)$, it easily follows that $R^{+} \Rightarrow (a \leq b)$ is true if and only if $C(a) \subseteq EC(b)$, where $C(a)$ is the set of component constants of $a$ and $EC(b)$ is an extended set of constants that contains $C(b)$.  In fact, $EC(b) = C(b^{\prime})$ where $b^{\prime}$ is the term with most component constants that is equal to $b$ modulus $R^{+}$. In other words, $b^{\prime}$ is the largest term that satisfies $R^{+} \Rightarrow (b = b^{\prime})$, which always exists and is unique.  \\
Suppose a constant $c$ such that $F_C(R^{+}) \models c \leq b = b^{\prime}$. Since $F_C(R^{+})$ is the freest model of $R^+$, we must have that  $R^{+}$ implies $c \leq b = b^{\prime}$. It follows that $c$ must be either in $C(b)$ or in $C(r_L)$ for some duple $r \in R^{+}$ such that $C(r_R) \subseteq C(b^{\prime})$. This leads to the equation $EC(b) = C(b) \cup_{r \in R^{+}: C(r_R) \subseteq EC(b)} C(r_L)$. We can compute $EC(b)$ iteratively starting with $EC_0(b) := C(b)$ and: \[
EC_n(b) := EC_{n - 1}(b) \cup_{r \in R^{+}: C(r_R) \subseteq EC_{n - 1}(b)} C(r_L),
\] 
which, since $C$ and $R^{+}$ are finite, reaches a limit where $EC_n(b) = EC_{n - 1}(b)$ in a finite (usually very small) number of iterations. Indeed, this iterative computation is carried out for each element $b$ with a dual, i.e.\ for each element of $E$ as part of the procedure to build the graph $G$ that underlies the construction of $D$; there is an edge $([b] \rightarrow [a]) \in G$ if and only if $C(a) \subseteq EC(b)$, which occurs if and only if $F_C(R^{+}) \models a \leq b$. \\
(ii) According to Theorem \ref{SparseCrossing:segregationTheorem}, negative duples never imply positive duples. In fact, $R$ implies a positive duple $p$ if and only if $R^+$ implies $p$. It follows that the axioms $R$ are inconsistent if and only if  there is a duple $s = (s_L, s_R)$ in $R^{-}$ such that $F_C(R^+) \models s_L \leq s_R$.  Since the left and right-hand terms of $s$ are in $E$, we can use proposition (i) to state that $R$ is inconsistent if and only if $D \models [s_R] \leq [s_L]$. By construction $D$ satisfies the dual of every positive duple $r = (r_L, r_R)$ in $R^{+}$, as it is enforced with an edge $([r_R] \rightarrow [r_L]) \in G$. Therefore, $R$ is inconsistent if and only if $D$ does not satisfy $R^{*-}$ and the result follows. \\
(iii) We enforced each positive duple $r \in R^{+}$ with edges $([r_R] \rightarrow [r_L]) \in G$. The atomization of $D$ is constructed by creating an atom $\xi_x$ for each constant $[x]$ of the dual (i.e., for each $x \in E$), by appending the atom to the graph with an edge $\xi_x \rightarrow [x]$, and finally by computing the transitive closure of the graph. There is a one-to-one map between the atoms in the lower atomic segment $L^{a}([d])$ for some $d \in E$ and the set of elements $W(d) = \{ e : (e \in E) \wedge D \models [e] \leq [d] \}$. Using (i), there is a one-to-one map between the atoms in the lower atomic segment $L^{a}([d])$ for some $d \in E$ and the set of elements $W(d) = \{ e : (e \in E) \wedge F_C(R^{+}) \models d \leq e \}$; in fact, the map is given by $\xi_e \in L^{a}([d])$ if and only if $e \in W(d)$. \\
The simplification of the atomization of the dual implies removing atoms from the dual with the constraint of having $R^{*-}$ satisfied. In this case some $\xi_b$ atoms may be removed. The removal of atoms does not change the one to one correspondence for the remaining atoms. To extend the validity of (iii), to simplified duals, we must write ``$\xi_b \in L^{a}([a])$ if and only if $\xi_b$ exists and  $b \in W(a)$" rather than just ``$\xi_b \in L^{a}([a])$ if and only if $b \in W(a)$".
\end{proof}
\bigskip

\begin{theorem} \label{SparseCrossing:termDualAtoms} 
Let $D$ be a dual built for a set of axioms $R$ and a set $E$ of terms with a dual, as defined in Section \ref{SparseCrossing:buildingTheDual}. Let $t \in E$ and let $\xi$ be an atom of the dual; \\
i) there is a regular element $s \in E$ such that $U^c(\xi) = \{ [s] \} \cup U^c([s])$, \\ 
ii) if $\xi < [c]$ for every component constant of $t$ then $\xi < [t]$, \\
iii) $\cap_{c \in C(t)} L^a([c]) = L^a([t])$, \\
iv) for $t \in E$, there is a duality given by $\cup \leftrightarrow \cap$ and $L^a_{F_C(R^{+})}(t) \leftrightarrow L_D^{a}([t])$.
\end{theorem}
\begin{proof}
(i) An atom $\xi$ in the graph of $D$ is first appended with an edge to a single node, say $[s]$, and only after transitive closure it may gain more edges; let's use $\xi_s$ instead of $\xi$ for more clarity. After transitive closure we may have more edges, like $\xi_s \rightarrow [t]$, if either $[s] \rightarrow [t]$ or $t = s$. It follows that the upper constant segment of $\xi_s$ is $[s]$ plus the upper segment $U([s])$, i.e.\ the set of nodes reachable from $[s]$ in the graph of $D$. \\ 
(ii) Let $s$ be the regular element such that $U^c(\xi) = \{ [s] \} \cup U^c([s])$, so  $\xi_s = \xi$. Assume $\xi_s < [c]$ for every component constant of $t$. Then, $([s] = [c]) \vee ([s] \rightarrow [c])$ for every component constant of $t$, so $D \models [s] \leq [c]$ for every component constant of $t$. Immediately follows that $C(t) \subseteq EC(s)$. 
Since $s$ and $t$ have duals (are elements of $E$) then, while building the graph $G$ for $D$, we should have compared $t$ and $s$ and introduced the edge $[s] \rightarrow [t]$. After transitive closure we have $\xi_s \rightarrow [s] \rightarrow [t]$ from which it follows $\xi_s < [t]$. This result can also be proven from Theorem \ref{SparseCrossing:weakDuality} (iii). \\
(iii) We can rewrite (ii) as $\cap_{c \in C(t)} L^a([c]) \subseteq L^a([t])$. For every component constant $c$ of $t$ there is an edge $[t] \rightarrow [c]$ in the graph of $D$ hence, after transitive closure $L^a([t]) \subseteq L^a([c])$, which completes the proof. \\ 
Another way to prove this is, again, by using Theorem \ref{SparseCrossing:weakDuality} (iii). In any model $t \leq e$ if and only if $c\leq e$ for each $c \in C(t)$, then $\xi_s \in L^{a}([t])$ if and only if $s \in \{ e : (e \in E) \wedge F_C(R^{+}) \models t \leq e \}$ if and only if $s \in \cap_{c \in C(t)} \{ e : (e \in E) \wedge F_C(R^{+}) \models c \leq e \}$. \\
(iv) According to Theorem \ref{SparseCrossing:weakDuality} (i) when  $a,b \in E$ we must have $D \models [b] \leq [a] \,\, \Leftrightarrow \,\, F_C(R^{+}) \models a \leq b$ and, since $D$ is atomized, $F_C(R^{+}) \models a \leq b$ if and only if $L_D^{a}([b]) \subseteq L_D^{a}([a])$.  From proposition iii, in the dual $\cap_{c \in C(t)} L_D^a([c]) = L^a([t])$, while, $\cup_{c \in C(t)} L_{F_C(R^{+})}^a(c) = L_{F_C(R^{+})}^a(t)$ for  $F_C(R^{+})$. Therefore, there is a duality, that only applies to terms with a dual, given by $\cup \leftrightarrow \cap$ and $L^a_{F_C(R^{+})}(t) \leftrightarrow L_D^{a}([t])$.
\end{proof}
\bigskip

\begin{theorem} \label{SparseCrossing:traceEqualitiesForAtoms} 
Let $W(t) = \{ e : (e \in E) \wedge F_C(R^{+}) \models t \leq e \}$ be the upper segment restricted to $E$, introduced in Theorem \ref{SparseCrossing:weakDuality} (iii); \\
i) $Tr(\phi) = \cup_{c \in U^c(\phi)} L^a([c])$ for any atom $\phi$, \\
ii) If $\xi$ is an atom of the dual, then $\xi \in Tr(\phi)$ if and only if $s \in \cup_{c \in U^c(\phi)} W(c)$, where $s \in E$ is the regular element associated to $\xi$ of Theorem \ref{SparseCrossing:termDualAtoms} (i),\\
iii) $\xi_s \in Tr(t)$ if and only if $s \in \cap_{\phi \in L_M^a(t)} \cup_{c \in U^c(\phi)} W(c)$,\\
iv) If $t$ has a dual, $L^a([t]) \subseteq Tr(t)$.
\end{theorem}
\begin{proof}
(i) It follows from the definition $Tr(\phi) = \{ \xi : \exists c \in C, \, (\phi < c ) \wedge (D \models \xi < [c]) \} = \cup_{c \in U^c(\phi)} \{ \xi : D \models \, \xi < [c] \} = \cup_{c \in U^c(\phi)} L^a([c])$ where $D$ is the dual. \\
(ii) Let $\xi_e$ be the atom corresponding with $[e]$ in Theorem \ref{SparseCrossing:termDualAtoms} (i): \[
Tr(\phi) := \{ \xi : \exists c \,\, (\phi < c) \wedge (c \in C) \wedge (D \models \, \xi < [c]) \} =\] \[
= \{ \xi_e : \exists c,e \,\, (\phi < c) \wedge (c \in C) \wedge (e \in E) \wedge F_C(R^{+}) \models c \leq e \} = \] \[
= \cup_{c \in U^c(\phi)}  \{ \xi_e :  (e \in E) \wedge F_C(R^{+}) \models c \leq e \} = \cup_{c \in U^c(\phi)}  \{ \xi_e :  e \in W(c) \} \]
where we have used Theorem \ref{SparseCrossing:weakDuality} (i) and (iii). From here, it follows that there is a one-to-one map between the trace of an atom $\phi$ and the intersection with $E$ of the union of the upper segments of the constants $U^c(\phi)$ in the free model of $R^+$. \\
(iii) We just showed that $\xi_s \in Tr(\phi)$ if and only if $s \in \cup_{c \in U^c(\phi)} W(c)$. Since $Tr(t) = \cap_{\phi \in L_M^a(t)} Tr(\phi)$, we have that $\xi_s \in Tr(t)$ holds if and only if $s \in \cup_{c \in U^c(\phi)} W(c)$ is true for every atom $\phi \in L_M^a(t)$. \\
(iv) an atom $\xi$ in the lower segment of $[t]$ is edged to $[t]$ in the graph of $D$. Since there are edges from $[t]$ to its component constants, by transitive closure $\xi$ is in the lower segment of every component constant of $[t]$. From $Tr(t) = \cap_{c \in C(t)} Tr(c) =\cap_{c \in C(t)} \cap_{\phi \in L_M^a(c)} Tr(\phi)$ and $Tr(\phi) = \{ \eta : \exists c \,\, (\phi < c) \wedge (c \in C) \wedge (D \models \, \eta < [c]) \} = \cup_{\phi < c} \{ \eta : D \models \, \eta < [c] \}$ it follows that $\xi \in Tr(t)$.
\end{proof}
\bigskip

\begin{theorem} \label{SparseCrossing:nonRedundant} 
Redundant atoms of $M$ do not change the traces in $M$. The trace is well defined, i.e.\ the trace is the same for a model independently of how it is atomized.
\end{theorem}
\begin{proof}
Let $x$ be any term. The trace of $x$ in the model $M$ is given by $Tr(x) = \cap_{\phi \in L_M^a(x)} Tr(\phi)$. Suppose $\varphi$ and $\psi$ are two atoms of $M$. If neither $\varphi$ nor $\psi$ are in the lower atomic segment of $x$ then the trace of $x$ in the model $M + \{ \varphi \bigtriangledown \psi \}$ is the same that the trace in $M$. Assume that $\varphi$ is in the lower atomic segment of $x$. The trace of $x$ in the model $M + \{ \varphi \bigtriangledown \psi \}$ becomes $Tr^{\prime}(x) = Tr(\varphi \bigtriangledown \psi) \cap_{\phi \in L_M^a(x)} Tr(\phi) = (Tr(\varphi) \cup Tr(\psi)) \cap Tr(x)= (Tr(\varphi) \cap Tr(x)) \cup (Tr(\psi) \cap Tr(x)) = Tr(x) \cup (Tr(\psi) \cap Tr(x)) = Tr(x)$. Therefore, adding redundant atoms does not alter the traces. On the other hand, the non-redundant atoms of a model are always present in any atomization of a model. It follows that the traces are determined only by the non-redundant atoms of the model $M$ and, hence, are independent of the atomization of the model, see \cite{SecondPaperArX}. Notice that the traces of atoms are universal and do not depend upon $M$, so we do not need to distinguish $Tr^{\prime}(\phi)$ and $Tr(\phi)$ as they are equal.
\end{proof}
\bigskip

\begin{theorem} \label{SparseCrossing:freeModelConstraints} 
The freest semilattice model, $F_C(\emptyset)$, satisfies the trace constraints of any consistent set of axioms.
\end{theorem}
\begin{proof}
Theorem \ref{freestModelTheorem} in Supplementary Section \ref{supplementary:review} shows that the freest model generated by a set of constants $C$ can be atomized by a set $A$ with as many atoms as constants in $C$, each with a single (and different) constant in its upper constant segment.  Let $D$ be a dual built for a consistent set of axioms $R$, and a set of elements $E$. Let $t$ be a regular element (a term) with a dual and let $\phi_c$ be the atom of $F_C(\emptyset)$ such that $U^c(\phi_c) = \{ c \}$. Using Theorem  \ref{SparseCrossing:nonRedundant}, we can compute the trace of an element of $F_C(\emptyset)$ assuming that $F_C(\emptyset)$ is atomized by $A$. For any constant $c \in C$, in the freest model $Tr(c) = \cap_{\varphi \in L_{F_C(\emptyset)}(c) } Tr(\varphi) = \cap_{\varphi < c, \varphi \in A} Tr(\varphi) = Tr(\phi_c)$. Using Theorem \ref{SparseCrossing:traceEqualitiesForAtoms} (i) we get $Tr(\phi_c) = L^a([c])$ and it follows $Tr(c) =  L^a([c])$, Therefore, in the freest model, trace and lower segment of the dual coincide for the constants. We will see now that they also coincide for every element. \\
The trace  of a regular element of $F_C(\emptyset)$ represented by a term $t$ is $Tr(t) = \cap_{\varphi \in L_{F_C(\emptyset)}^a(t)} Tr(\varphi) =\cap_{\varphi \in L_A^a(t)} Tr(\varphi) = \cap_{c \in C(t)} Tr(\varphi_c) = \cap_{c \in C(t)} L^a([c])$.  
Now using that $\cap_{c \in C(t)} L^a([c]) = L^a([t])$ was proven in Theorem \ref{SparseCrossing:termDualAtoms} (iii), we can write $Tr(t) = \cap_{c \in C(t)} L^a([c]) = L^a([t])$. It follows that, for the freest model $F_C(\emptyset)$ the equality $Tr(t) = L^a([t])$ holds for every element of $E$. Therefore, for $F_C(\emptyset)$, the trace constraints $Tr(r_R) \subseteq Tr(r_L)$ for $r \in R^+$ and $Tr(r_R) \not\subseteq Tr(r_L)$ for $r \in R^-$ become $L^a([r_R]) \subseteq L^a([r_L])$ for $r \in R^+$ and  $L^a([r_R]) \not\subseteq L^a([r_L])$ for $r \in R^-$, which correspond to axioms in the set $R^*$ that axiomatizes $D$, so they are always satisfied by the dual if $R$ is consistent, as proven in Theorem \ref{SparseCrossing:weakDuality}.
\end{proof}
\bigskip

\begin{theorem} \label{SparseCrossing:traceConstraints} 
Given a model $M$ and a dual $D$ built for a consistent set of axioms $R$, there is a set $\Lambda$ of atoms with a single constant in their upper constant segments such that $M + \Lambda$ satisfy the trace constraints.  The set $\Lambda$ can be obtained with Algorithm \ref{SparseCrossing:traceConstraintsAlgorithm}.
\end{theorem}
\begin{proof}
The first part can be proven by making $\Lambda = \{ \phi_c : U^c(\phi_c) = \{c\} \, \wedge  \, c \in C\}$; since every atom over $C$ is redundant with such set $\Lambda$, the atomization $M + \Lambda$ atomizes $F_C(\emptyset)$ and the result follows by using Theorem \ref{SparseCrossing:freeModelConstraints}. However, since trace constraints can usually be fixed with a set $\Lambda$ of a carnality significantly smaller than $\vert C \vert$, we provide a proof here more in line of what is needed to actually compute $\Lambda$.

In Algorithm \ref{SparseCrossing:traceConstraintsAlgorithm} we start with $\Lambda = \emptyset$ and add atoms to $\Lambda$ as needed. Let $N = M + \Lambda$ at some intermediate step. 

Let $r \in R^{+}$ and assume the trace constraint for $r$ is ``broken", i.e.\ $Tr_N(r_R) \not\subseteq Tr_N(r_L)$. There must be at least one atom $\xi$ of $D$ such that $\xi \in Tr_N(r_R) = \cap_{\varphi \in L_N^a(r_R)} Tr(\varphi)$ and $\xi \not\in Tr_N(r_L) = \cap_{\varphi \in L_N^a(r_L)} Tr(\varphi)$. Suppose that $\xi < [c]$ for every constant $c$ in the component constants of $r_R$. Theorem \ref{SparseCrossing:termDualAtoms} establishes that $\xi < [r_R]$ and, since $r$ is in the set of positive axioms, the graph for $D$ has an edge $[r_R] \rightarrow [r_L]$; it follows $\xi < [r_L]$ and, since there are edges from $[r_L]$ to the duals of the component constants of $r_L$, every atom $\varphi \in L_M^a(r_L)$ must satisfy $\xi \in Tr(\varphi)$ which is clear from the definition $Tr(\varphi) = \{ \xi : \exists c \,\, (\varphi< c) \wedge (c \in C) \wedge (D \models \, \xi < [c]) \}$. This implies $\xi \in Tr_N(r_L) = \cap_{\varphi \in L_N^a(r_L)} Tr(\varphi)$; a contradiction. Therefore, there is a component constant $c$ of $r_R$ for which $\xi \not< [c]$. The atom $\phi_c$ with $U^c(\phi_c) = \{ c \}$ has a trace $Tr(\phi_c) = L^a([c])$ that does not contain $\xi$.  The model $N +  \{ \phi_c \}$ satisfies $\xi \not\in Tr_{N +  \{ \phi_c \}}(r_R) = \cap_{\varphi \in L_{N +  \{ \phi_c \}}^a(r_R)} Tr(\varphi) = Tr_N(r_R) \cap L^a([c])$.  Add $\phi_c$ to $\Lambda$ and update $N = M + \Lambda$. Repeating the same procedure for every atom in $Tr_N(r_L) - Tr_N(r_R)$ we can ``fix" the positive trace constraint, i.e.\ we reach a set $\Lambda$ such that $Tr_{M + \Lambda}(r_R) \subseteq Tr_{M + \Lambda}(r_L)$.

Assume now $r \in R^{-}$ and the trace constraint for $r$ is broken, i.e.\ $Tr_N(r_R) \subseteq Tr_N(r_L)$. Since $D$ has been built for $R$, a consistent set of axioms, it must satisfy $D \models [r_R] \not\leq [r_L]$ from which it follows that there is at least one atom $\xi$ of $D$ such that $\xi < [r_R]$ and $\xi \not< [r_L]$. Since there are edges in the graph of $D$ from $r_R$ to the dual of each component constants $c$ of $r_R$, the duple $\xi < [r_R]$ implies $\xi < [c]$ for each component constant and then $\xi \in Tr(r_R)$. On the other hand, using Theorem $\ref{SparseCrossing:termDualAtoms}$ and $\xi \not< [r_L]$ there must be a component constant $c$ of $r_L$ for which $\xi \not< [c]$. The atom $\phi_c$ with $U^c(\phi_c) = \{ c \}$ has a trace $Tr(\phi_c) = L^a([c])$ that does not contain $\xi$. The model $N +  \{ \phi_c \}$ satisfies $\xi \not\in Tr_{N +  \{ \phi_c \}}(r_L) = \cap_{\varphi \in L_{N +  \{ \phi_c \}}^a(r_L)} Tr(\varphi) = Tr_N(r_L) \cap L^a([c])$, and the negative trace constraint is fixed, i.e.\ $Tr_{N +  \{ \phi_c \}}(r_R) \not\subseteq Tr_{N +  \{ \phi_c \}}(r_L)$. Add $\phi_c$ to $\Lambda$ and update $N$.

Since adding atoms to a model alters its traces, fixing one trace constraint with an atom implies that the whole set of trace constraints must be reviewed. If a trace constraint breaks by the insertion of an atom in $\Lambda$ we can always fix it by adding more atoms to $\Lambda$. This process necessarily ends as there are at most $\vert C \vert$ atoms that can added. In the worst case, $\Lambda$ contains the $\vert C \vert$ atoms with a single constant in its upper constant segment, and $M + \Lambda$ becomes the freest model, i.e.\ $F_C(\emptyset)$, which satisfies the trace constraints for any dual as discussed above.  
\end{proof}
\bigskip

\begin{theorem} \label{SparseCrossing:traceConstraintsUniversal} 
Let $D$ be a dual defined for a set of axioms $R$. Let $M$ be a model that satisfies the positive trace constraints for $R^{+}$. Let $a$ and $b$ be any two terms with or without a dual; \\ 
i) If $M \models a \leq b$ then the traces of $M$ satisfy $Tr_M(b) \subseteq Tr_M(a)$, \\ 
ii) if $F_C(R^{+}) \models a \leq b$ then the traces of $M$ satisfy  $Tr_M(b) \subseteq Tr_M(a)$, \\
iii) if $\square_{R^{+}} M \models  a \leq b$ then the traces of $M$ satisfy $Tr_M(b) \subseteq Tr_M(a)$.
\end{theorem}
\begin{proof}
(i) This was proven in Theorem \ref{SparseCrossing:traceImplication}.   \\
(ii) In the proof of Theorem \ref{SparseCrossing:weakDuality} we used that $F_C(R^{+}) \models a \leq b$ if and only if $R^{+} \Rightarrow a \leq b$ simply because $F_C(R^{+})$ is the free model of $R^{+}$. We also showed that there is a term $b^{\prime}$ such that $F_C(R^{+}) \models (b = b^{\prime})$ that has the property $F_C(R^{+}) \models a \leq b$ if and only if $C(a) \subseteq C(b^{\prime})$. The term $b^{\prime}$ is the one with most component constants in the class of terms equivalent to $b$ in $F_C(R^{+})$. In fact, $b^{\prime}$ is the only term equivalent to $b$ such that for every $r \in R^{+}$ and $C(r_R) \subseteq C(b^{\prime})$ then $C(r_L) \subseteq C(b^{\prime})$. We argued in the proof of Theorem \ref{SparseCrossing:weakDuality} that we can construct $b^{\prime}$ starting from $b_0 = b$ and extending $b$ to $b_{n + 1} = b_n \odot r_L$ if $C(r_R) \subseteq C(b_n)$ for $r \in R^{+}$, obtaining a series of terms that converges to $b^{\prime}$ in a finite number of steps. From Theorem \ref{SparseCrossing:unionOfTraces} (iii), it is clear that $C(r_R) \subseteq C(b_n)$ implies $Tr(b_n) \subseteq Tr(r_R)$. Assume the positive trace constraints are satisfied. We have $Tr_M(b_n) \subseteq Tr_M(r_R) \subseteq Tr_M(r_L)$ and then $Tr_M(b_{n + 1}) = Tr_M(b_n) \cap Tr_M(r_L) = Tr_M(b_n)$. It follows that $Tr_M(b^{\prime}) = Tr_M(b)$,  in other words, $b$ converges towards $b^{\prime}$ in a series of terms with increasingly more constants but always with the same trace. Even more, since every term $t$ in the set $\{ t: F_C(R^{+}) \models (b = t) \}$ can be extended to $b^{\prime}$ we have that $F_C(R^{+}) \models (b = t)$ implies $Tr_M(b) = Tr_M(t)$. Using this result and the fact that $a \leq b$ is equivalent to $b = a \odot b$, we have  $F_C(R^{+}) \models a \leq b$ implies $Tr_M(b) = Tr_M(a \odot b)$. From $Tr_M(a \odot b) = Tr(a) \cap Tr(b)$ (Theorem \ref{SparseCrossing:unionOfTraces} (ii)) it follows that $Tr_M(b) \subseteq Tr_M(a)$. \\
(iii) The model $\square_{R^{+}} M$ is the freest model $F_C(Th^{+}(M) \cup R^{+})$. We can follow the same reasoning that we applied to $F_C(R^{+})$ in (ii).  We have $F_C(Th^{+}(M) \cup R^{+}) \models a \leq b$ if and only if $Th^{+}(M) \cup R^{+} \Rightarrow a \leq b$.  Again, there must be some $b^{\prime}$ such that $C(a) \subseteq C(b^{\prime})$ and we could construct such $b^{\prime}$ using an iterative process with the duples of $Th^{+}(M) \cup R^{+}$ instead of just $R^{+}$. The term $b_0 = b$ grows into $b^{\prime}$ in a series of steps for which $Tr_M(b_{n + 1}) = Tr_M(b_n) \cap Tr_M(r_L)$ for some $r \in Th^{+}(M) \cup R^{+}$. We showed in (i) and (ii) that for every  $r \in Th^{+}(M) \cup R^{+}$ the traces of $M$ satisfy $Tr_M(r_R) \subseteq Tr_M(r_L)$. As before, it follows  $Tr_M(b_{n + 1}) = Tr_M(b_n) \cap Tr_M(r_L) = Tr_M(b_{n})$ and then, $Th^{+}(M) \cup R^{+} \Rightarrow a \leq b$ implies $Tr_M(b) = Tr_M(b)^{\prime} \subseteq Tr_M(a)$. Therefore, if $\square_{R^{+}} M \models  a \leq b$ the traces of $M$ satisfy $Tr_M(b) \subseteq Tr_M(a)$. 
\end{proof}
\bigskip

\begin{theorem} \label{SparseCrossing:validExclusion} 
Let $D$ be a dual built for a set $E$ of terms and axioms $R$. Let $T_{\psi} \in E$ be the pining term of an atom $\psi$ and $c \in C \subset E$ a constant such that $c \not\in C(T_{\psi})$. If $D \models [T_{\psi}] \leq [c]$ then $\psi$ is not an atom of  $F_C(R^+)$.
\end{theorem}
\begin{proof}
If $D \models [T_{\psi}] \leq [c]$ and $c \not\in C(T_{\psi})$ the dual must have acquired an edge $[T_{\psi}] \rightarrow [c]$ as a consequence of the enforcing of $R^{*+}$.  There should be some mentioned terms $s_1,s_2,...s_n$ (therefore with a dual) such that $[T_{\psi}] \rightarrow [s_1]  \rightarrow ... \rightarrow [s_n] \rightarrow [c]$ by which $[T_{\psi}]$ has acquired the edge $[T_{\psi}] \rightarrow [c]$ from transitive closure. Since $T_{\psi}, s_i, c \in E$, from Theorem \ref{SparseCrossing:weakDuality} (i), $F_C(R^{+}) \models c \leq s_n \leq ... \leq s_1  \leq T_{\psi}$. On the other hand, $c$ is not a component constant of $T_{\psi}$, therefore $\psi < c$. From $\psi < c$ and $F_C(R^{+}) \models c \leq T_{\psi}$ we conclude that $\psi$ is not an atom of the model $F_C(R^{+})$ because  $\psi < T_{\psi}$ is false. 
\end{proof}
\bigskip

\begin{theorem} \label{SparseCrossing:indicatorsTheorem} 
Let $\phi$ be any atom, $D$ a dual built for a set of axioms $R$ and $\xi$ an atom of the dual. Consider the correspondence for atoms of the dual of Theorem \ref{SparseCrossing:termDualAtoms} (i): \\
i) if $\xi$ corresponds with the pinning term $T_{\psi}$ of an atom $\psi$, i.e.\ if $\xi$ was introduced in the dual with an edge $\xi \rightarrow [T_{\psi}]$, then $\xi \not\in Tr(\phi)$ implies $U^c(\phi) \subseteq U^c(\psi)$. \\
ii) if $\xi$ corresponds with the pinning term of $\psi$ and $\psi \models R^{+}$ then $\xi \not\in Tr(\phi)$ if and only if $U^c(\phi) \subseteq U^c(\psi)$. \\
iii) if $\xi$ corresponds with the right-hand side of a duple $r \in R^{-}$, i.e.\  $\xi$ was introduced with an edge $\xi \rightarrow [r_R]$, then $\xi \not\in Tr(\phi)$ implies $\phi \not< r_R$. 
\end{theorem}
\begin{proof}
(i) Suppose $U^c(\phi) \not\subseteq U^c(\psi)$. There is a constant $k$ such that $\phi < k$ and  $\psi \not< k$.  The pinning term $T_{\psi}$ has component constants $C - U^c(\psi)$, so $k \in C(T_{\psi})$.  In $D$ there is an edge $\xi \rightarrow [T_{\psi}]$ and edges $[T_{\psi}] \rightarrow [c]$ for each constant $c \in C - U^c(\psi)$, including $k$.  Therefore, $D \models \xi < [k]$. Since $Tr(\phi) = \{ \eta : \exists c \,\, (\phi < c) \wedge (c \in C) \wedge (D \models \, \eta < [c]) \} = \cup_{\phi < c} \{ \eta : D \models \, \eta < [c] \}$, it follows  $\xi \in Tr(\phi)$.  Hence, $\xi \not\in Tr(\phi)$ implies $U^c(\phi) \subseteq U^c(\psi)$. \\ 
(ii) Left to right is proven in (i).  Assume $U^c(\phi) \subseteq U^c(\psi)$. Since $C(T_{\psi}) = C - U^c(\psi)$, we get $C(T_{\psi})\cap U^c(\phi) = \emptyset$. Suppose $\xi  \in Tr(\phi) = \cup_{\phi < c} \{ \eta : D \models \, \eta < [c] \}$; there must be at least one constant $c \in U^c(\phi)$ such that $D \models \xi < [c]$, that is not a component constant of $T_{\psi}$ (because $C(T_{\psi})\cap U^c(\phi) = \emptyset$), such that the dual has acquired an edge $[T_{\psi}] \rightarrow [c]$ as a consequence of the enforcing of $R^{*+}$. Theorem \ref{SparseCrossing:validExclusion} says that in this situation $\psi$ is not an atom of the model $F_C(R^{+})$. Hence, if  $\psi \models R^{+}$ then $U^c(\phi) \not\subseteq U^c(\psi)$ implies $\xi \not\in Tr(\phi)$. \\
(iii) In the graph of $D$ there is an edge $[r_R] \rightarrow [c]$ for each component constant $c$ of $r_R$ and an edge for the atom $\xi \rightarrow [r_R]$, so $D \models \xi < [c]$ for every component constant of $r_R$. Suppose $\phi < r_R$. There is a constant $k \in  U^c(\phi)\cap C(r_R)$.  Since $Tr(\phi) = \cup_{\phi < c} \{ \eta : D \models \, \eta < [c] \}$ and $\phi < k$ and $\xi < [r_R]  < [k]$ it follows $\xi \in Tr(\phi)$. Hence,  if $\xi \not\in Tr(\phi)$  then $\phi \not< r_R$. 
\end{proof}
\bigskip

\begin{theorem} \label{SparseCrossing:smallerModelThatSatisfiesR} 
If a model $M$ over a set of constants $C$ satisfies the axioms $R = R^{+} \cup R^{-}$, there is a model $N$ over $C$ spawned by at most $\vert R^{-} \vert + 1$ atoms of $M$, such that $N$ also satisfies $R$.
\end{theorem}
\begin{proof}
If $M \models R$, any model $N$ atomized with a subset of atoms of $M$ satisfies $N \models R^{+}$; this is clear, since the discriminant of a duple $r \in R^{+}$ that is empty in $M$ it is also empty in $N$. For each duple in $s \in R^{-}$ choose one atom of $M$ in the discriminant of $s$: the resulting subset of atoms discriminates every duple in $R^{-}$. Since $N$ also satisfies $R^{+}$, we have $N \models R$. We have to make a remark; in \cite{SecondPaperArX} it is discussed that although not every subset of atoms of $M$ atomizes a model over $C$ we can make any subset of atoms of $M$ a model over $C$ by adding the atom $\ominus_C$, an atom that is in the lower segment of every constant of $C$. The atom $\ominus_C$ is compatible with every semilattice model over $C$ and it discriminates no duple so, in practice, it can be ignored while doing computations. It follows that a subset of at most $\vert R^{-} \vert + 1$ atoms of $M$ atomizes a model over $C$ that satisfies $R$.
\end{proof}
\bigskip

\begin{theorem} \label{SparseCrossing:masterSize} 
Let $D$ be a dual built for axioms $R$ and atomized by a set of atoms $D_N \cup D_R$ where atoms in $D_N$ correspond to pinning terms, and atoms in $D_R$ correspond to right-hand side terms of duples of $R^-$, according to the correspondence of Theorem \ref{SparseCrossing:termDualAtoms} (i). Let $M$ be a model that satisfies the trace constraints for the dual $D$: \\ 
i) There is a subset of atoms of $M$, of at most  $n \leq \sum_{\xi_{\psi} \in D_N} \vert U^c(\psi) \vert + \sum_{\xi_{r_R} \in D_R} \vert C - C(r_R) \vert$. atoms, that atomizes a model that also satisfies the trace constraints. This subset can be obtained by discarding atoms of $M$ while keeping the traces of the constants invariant. \\
ii) If the atomization of the dual is dominated by the pinning terms, i.e, when $\vert D_R \vert << \vert D_N \vert$, an upper bound for $n$ is $\,n \leq  \alpha \vert D \vert$ where $\alpha$ is the average size of the upper constant segment of the atoms in the set $\{ \psi :  \xi_{\psi} \in D_N \}$ and $\vert D \vert \approx  \vert D_N \vert$ is the size of the atomization of the dual.
\end{theorem}
\begin{proof}
The dual $D$ is built for a set $E = C \cup Terms(R^{+} \cup R^{-}) \cup \{ T_{\psi} : \psi \in N \}$ where $N$ is a set of atoms and $T_{\psi}$ the pinning term of $\psi \in N$.  The atoms of the dual are introduced either edged to the dual of a pinning term $[T_{\psi}]$ or under the dual $[r_R]$ of a duple $r \in R^{-}$.   Since the atoms of the dual associated to pinning terms tend to discriminate many duples of $R^{*-}$, a dual with pinning terms can often be atomized by a handful of atoms, mostly under pinning terms.  In fact, Theorem \ref{SparseCrossing:smallerModelThatSatisfiesR} says  that to atomize the dual we need no more than $\vert R^{*-} \vert + 1 = \vert R^{-} \vert + 1$ atoms.  The result we are going to prove is valid no matter if we chose to reduce the size of the atomization of the dual or if we choose to keep all of the atoms, i.e., as many as $\vert R^{-} \vert +\vert N \vert$.  \\ 
(i) The trace is a function that maps a term to a subset of the atoms of the dual. For any term $x$ the trace in $M$ is $Tr(x) = \cap_{c \in C(x)} Tr(c)$. Hence, a discarding of atoms from $M$ that preserves the traces of the constants preserves the traces of all the terms and then, it also preserves the trace constraints. Since, for a constant $c$ the trace in $M$ is $Tr(c) = \cap_{\phi \in L_M^a(c)} Tr(\phi)$, discarding atoms while preserving the trace $Tr(c)$ requires keeping, for each atom $\xi \in D - Tr(c)$, at least one atom $\phi \in L_M^a(c)$ such that $\xi \not\in Tr(\phi)$.  Since $\xi \not\in Tr(c)$ implies $D \models [\xi] \not< [c]$, to preserve the trace of $c$ we need at most $ \vert \{ \xi \in D : D \models [\xi] \not< [c]\} \vert$ atoms of $M$. In total, the number $n$ of atoms of $M$ that we must keep is bounded by \[n \leq \sum_{c \in C} \vert \{ \xi \in D : D \models [\xi] \not< [c]\} \vert = \sum_{\xi \in D} \vert \{ c: D \models [\xi] \not< [c]\} \vert.
\]
We can separate  $\sum_{\xi \in D} \vert \{ c: D \models [\xi] \not< [c]\} \vert$ into two summands, one for atoms of the dual associated to pinning terms and another for atoms of the dual associated to the right-hand side of the duples of $R^{-}$, i.e., $n \leq \sum_{\xi_{\psi} \in D_N} \vert \{ c: D \models [T_{\psi}] \not< [c]\} \vert + \sum_{\xi_{r_R} \in D_R} \vert \{ c: D \models [r_R] \not< [c]\} \vert$. Since there are edges from the dual of a term to its component constants we have the following two set inclusions: $\{ c: D \models [T_{\psi}] \not< [c]\} \subseteq U^c(\psi)$ and $\{ c: D \models [r_R] \not< [c]\} \subseteq C - C(r_R)$, so: \[
n \leq \sum_{\xi_{\psi} \in D_N} \vert \{ c: D \models [T_{\psi}] \not< [c]\} \vert + \sum_{\xi_{r_R} \in D_R} \vert \{ c: D \models [r_R] \not< [c]\} \vert \lesssim \sum_{\xi_{\psi} \in D_N} \vert U^c(\psi) \vert + \sum_{\xi_{r_R} \in D_R} \vert C - C(r_R) \vert.
\]
We use $\lesssim$ because the two set inclusions are, approximately, set equalities. According to Theorem  \ref{SparseCrossing:validExclusion} if the atoms of $N$ satisfy $R^+$, which is the case as atoms that do not satisfy $R^+$ are discarded, then $\vert \{ c: D \models [T_{\psi}] \not< [c]\} \vert = \vert U^c(\psi) \vert$.  On the other hand, $\vert \{ c: D \models [r_R] \not< [c]\} \vert \approx \vert C - C(r_R) \vert$ in most cases.   \\
(ii) As training progresses it becomes possible (by discarding atoms of the dual as in the proof of Theorem \ref{SparseCrossing:smallerModelThatSatisfiesR}) to obtain an atomization of the dual dominated by the pinning terms,  $\vert D \vert \approx  \vert D_N \vert$. When the atomization of the dual becomes dominated by the pinning terms we can neglect the contribution of $\sum_{\xi_{r_R} \in D_R} \vert C - C(r_R) \vert$ and then $\,n \leq \sum_{\xi_{\psi} \in D_N } \vert U^c(\psi) \vert \approx \alpha \vert D \vert $ where $\alpha$ is the average atom size $\alpha = \frac{1}{\vert D_N \vert} \sum_{ \xi_{\psi} \in D_N} \vert U^c(\psi)\vert$.
\end{proof}

\newpage
\subsection{Algorithms}
\label{SparseCrossing:algorithms}

\SetKwProg{Fn}{Function}{}{}
\SetKwRepeat{Do}{do}{while}%
\SetKwFor{Where}{where}{}{}

\begin{algorithm}
	Initialize $M_0 = F_C(\emptyset)$ if $M_0$ is not provided\;
	$N_0 \leftarrow M_0$\;
	create $m$ batches $R_i \subset R$, with or without repeating duples\;
	\ForEach{$i = 1$ to $m$}{
		$M_i \leftarrow M_{i-1}$\;
        		$N_{i} \leftarrow$ $\phi \in N_{i-1}$ such that $\forall r \in R_i^+$  $\phi \not\in dis(r)$\;
		$D \leftarrow$ build the dual for $N_{i}$ and $R_i$\;
		$\Lambda \leftarrow$ enforce trace constraints for $R_i$ in model $M_i$ with dual $D$\;
		$M_i \leftarrow M_{i} \cup \Lambda$\;
		$M_i \leftarrow$ sparse-crossing of $R_i^+$ in model $M_i$ with dual $D$\;
		$N_i \leftarrow N_i \cup M_i$\;
	}
\caption{batch training}
\label{SparseCrossing:batchTraining}
\end{algorithm}

\begin{algorithm}
    $tr(\phi) = \emptyset$\;
    \ForEach{$c \in U^c(\phi)$}{
        $tr(\phi) \leftarrow tr(\phi) \cup L^a([c])$\;
    }
    \Return  $tr(\phi)$\;
\caption{trace of atom $\phi$ for dual $D = (v, \, L^a([\,\,]) )$}
\label{SparseCrossing:traceOfAtom}
\end{algorithm}

\begin{algorithm}
    $tr(t) = \{positions\,\,of\,\,v \}$\;
    \ForEach{$\phi  \in M $}{
          \If{$C(t) \cap U^c(\phi)  \not= \emptyset$}{ 
              $tr(\phi) \leftarrow$ trace of atom $\phi$ for dual $D$\;
              $tr(t) \leftarrow  tr(t) \cap tr(\phi)$\;
          }    
    }
    \Return  $tr(t)$\;
\caption{trace of term $t$ in model $M$ with dual  $D = (v, \, L^a([\,\,]) )$}
\label{SparseCrossing:traceOfTerm}
\end{algorithm}

\begin{algorithm}
    $I = \emptyset$\;
    \ForEach{$(r_L, r_R) \in R^-$}{
        $ind(r_R) \leftarrow  C(r_R)$\;
        $I \leftarrow I \cup \{ind(r_R)\}$  \;
    }
    \ForEach{$ind \in I$}{
   	 \Do{$ind$ changes} {
            \ForEach{$(r_L, r_R) \in R^+$}{
                \If{$r_R \subset ind$}{
                    $ind \leftarrow ind \cup  \{  component\, constants\, of\,  r_L \} $\;
                }
            }
        }
    }

    $v \leftarrow$ transform $I$ into a vector without repeated elements. Concatenate  $N$ at the end\;  	
    $T \leftarrow  \{ r_L : (r_L, r_R) \in R^-  \} \cup  \{ r_R : (r_L, r_R) \in R^-  \} \cup  C$\;  	 
    \ForEach{$t \in T$}{
         $L^a([t]) = \emptyset$\;  	
         \ForEach{$k \in v$}{
                   \If{$k \in I$}{ 
                          \If{$t \subseteq k$}{ 
	                    $L^a([t]) \leftarrow L^a([t]) \cup \{ position\,\, of\,\, k\,\,in\,\,v  \}$\;
                          }
                   }
                  \If{$k \in N$}{ 
                          \If{$t \cap U^c(k) = \emptyset$}{ 
	                    $L^a([t]) \leftarrow L^a([t]) \cup \{ position\,\, of\,\, k\,\,in\,\,v  \}$\;
                          }
                   }
        }
    }

     \ForEach{$(r_L, r_R) \in R^-$}{
        \If{$L^a([r_H]) \subseteq L^a([r_L])$}{ 
             \Return ``$R$ is inconsistent'';
         }
    }
    optional: discard one or many elements of v if $\forall(r_L, r_R) \in R^-$ still holds $L^a([r_H]) \not\subseteq L^a([r_L])$\;
    \Return  $v$ and \{$L^a([c])$ : $c \in C\}$\;
    
\caption{build the dual for the set of atoms $N$ and duples $R$}
\label{SparseCrossing:dualAlgebraAlgorithm}
\end{algorithm}

\begin{algorithm}
    \small
     \ForEach{$\phi \in M$}{
    	$tr(\phi) \leftarrow$ trace of atom $\phi$ for dual $D$\;
     }
    $T \leftarrow  \{ r_L : (r_L, r_R) \in R  \} \cup  \{ r_R : (r_L, r_R) \in R  \} \cup C$\;  	 
    \ForEach{$t \in T$}{
    	$tr(t) \leftarrow$ trace of term $t$ in $M$ with dual $D$\;
     }
      $\Lambda \leftarrow \emptyset$\;
      \Do{$\Lambda$ changes} {
	     \ForEach{$(r_L, r_R) \in R^-$}{
                \If{$tr(t_R) \subseteq tr(r_L) $}{
                   	$dC \leftarrow C(r_L) \backslash C(r_R)$\;
			\Do{$dC \not= \emptyset$} {
                             $c \leftarrow$ randomly extract $c$ from $dC$ with removal\;
				\If{$tr(r_R) \backslash \big(tr(r_L) \cap L^a([c])\big) \not= \emptyset$}{
	                         $\phi \leftarrow$ create atom such that $U^c(\phi)= \{c\}$\;
                                $tr(\phi) \leftarrow L^a([c]))$\;
	                         $\Lambda \leftarrow  \Lambda \cup \{ \phi \}$\;
				   \ForEach{$t \in T$}{
    	                             \If{$c \in C(t)$}{
                                         $tr(t) \leftarrow  tr(t) \cap L^a([c]))$\;
                                    }
                                }
                                $dC = \emptyset$\;
				}
			}
                }
            }
            \ForEach{$(r_L, r_R) \in R^+$}{
                \If{$tr(t_R) \not\subseteq tr(r_L)$}{
                      $dI = tr(t_R) \backslash tr(r_L)$;
                      $dC \leftarrow C(r_R) \backslash C(r_L)$\;
			\Do{$dI \not= \emptyset$} {
                             $c \leftarrow$ randomly extract $c$ from $dC$ with removal\;
				\If{$dI \cap tr(c) \not= dI$}{
                                $dI \leftarrow dI \cap tr(c)$\;
	                         $\phi \leftarrow$ create atom such that $U^c(\phi)= \{c\}$\;
                                $tr(\phi,  D) \leftarrow L^a([c])$\;
	                         $\Lambda \leftarrow  \Lambda \cup \{ \phi \}$\;
				   \ForEach{$t \in T$}{
    	                             \If{$c \in C(t)$}{
                                         $tr(t) \leftarrow  tr(t) \cap L^a([c]))$\;
                                    }
                                }
				}
			}
                }
            }
    }
    \Return  $\Lambda$ and $\{ tr(\phi) :  \phi \in M \cup \Lambda  \}$ and $\{ tr(c) :  c \in C \}$\;
\caption{enforce trace constraints for $R$ in model $M$ and dual $D = (v, \, L^a([\,\,])$}
\label{SparseCrossing:traceConstraintsAlgorithm}
\end{algorithm}

\begin{algorithm}
$L_M^a(t)  \leftarrow \emptyset$\;
\ForEach{$\phi \in M$}{
  \If {$C(t) \in U^c(\phi)$} {
      $L_M^a(t)  \leftarrow L_M^a(t)  \cup \{ \phi \}$\;
  }
}
\Return$L_M^a(t)$\;
\caption{lower atomic segment of term $t$ in model $M$}
\label{SparseCrossing:lowerAtomicSegmentOfTerm}
\end{algorithm}

\begin{algorithm}
$V \leftarrow \{positions\,\,of\,\,v \}$\;
$S \leftarrow M$\;
\ForEach{$(r_L, r_R) \in R^+$}{
    $dis \leftarrow  L_S^a(r_L) \backslash L_S^a(r_R)$\;
     \If {$dis \not= \emptyset$} {
            \ForEach{$i \in V$} {
		    $tD(i) \leftarrow \emptyset$\;
	    }	
	    \ForEach{$\varphi \in L_S^a(r_R)$}{
		   $dT \leftarrow V \backslash tr(\varphi)$\;
		   \ForEach{$i \in dT$} {	
		      $tD(i) \leftarrow tD(i) \cup \{ \varphi \}$\;
		   }	
	    }
           $S \leftarrow S\backslash dis$\;
	    \ForEach{$\phi \in dis$} {
	        $dT \leftarrow V \backslash tr(\phi)$\;
	       \Do{$dT \not= \emptyset$} {
	            $i \leftarrow$ randomly extract index from $dT$ with removal\;       
	            $\varphi \leftarrow$ randomly choose an atom from $tD(i)$\;  
	            $dT \leftarrow dT \cap tr(\varphi)$\;
	            $\phi \bigtriangledown \varphi \leftarrow$ create atom such that $U^c(\phi)= U^c(\phi) \cup U^c(\varphi)$\;
                   $tr(\phi \bigtriangledown \varphi) \leftarrow tr(\phi) \cup tr(\varphi)$\;
                   $S \leftarrow S \cup \{ \phi \bigtriangledown \varphi \}$\;
	       }
       }
    }
   \If {$\vert S \vert$ grows beyond some relative or absolute limit} {
       $S \leftarrow$ subset of $S$ that preserves the traces of the constants\;
   }	
}
\Return S\;
\caption{Sparse-crossing of $R^+$ in $M$ with trace $tr(\,)$ and dual $D = (v, \, L^a([\,\,])$}
\label{SparseCrossing:sparseCrossing}
\end{algorithm}

\begin{algorithm}
$V \leftarrow \{positions\,\,of\,\,v \}$\;
\ForEach{$i \in V$} {
    $tD(i) \leftarrow \emptyset$\;
}	
 \ForEach{$\varphi \in M$}{
     $dT \leftarrow V \backslash tr(\varphi)$\;
     \ForEach{$i \in dT$} {	
       $tD(i) \leftarrow tD(i) \cup \{ \varphi \}$\; 
    }	
}
$N \leftarrow  \emptyset$\;
\ForEach{$c \in C$}{
     $dT \leftarrow V \backslash tr(c)$\;
     \Do{$dT \not= \emptyset$} {
        $i \leftarrow$ randomly extract index from $dT$\;      
        \If {$tD(i) \cap L_M^a(c) \cap N = \emptyset$} {
            $\varphi \leftarrow$ randomly choose an atom from $tD(i) \cap L_M^a(c)$\;
            $N \leftarrow N \cup \{ \varphi \}$\;
         }
        \Else {
             $\varphi \leftarrow$ randomly choose an atom from $tD(i) \cap L_M^a(c) \cap N$\;
        }
        $dT \leftarrow dT \cap tr(\varphi)$\;  
    }
}
\Return N\;
\caption{find a subset of $M$ that preserves the traces of the constants with trace function $tr(\,)$ and dual $D = (v, \, L^a([\,\,])$}
\label{SparseCrossing:simplifyFromConstants}
\end{algorithm}

 \end{document}